\pdfoutput=1
\documentclass[11pt]{article}

\usepackage[letterpaper,margin=1in]{geometry}
\usepackage[numbers,sort&compress]{natbib}
\usepackage[utf8]{inputenc}
\usepackage[T1]{fontenc}
\usepackage{hyperref}
\usepackage{url}
\usepackage{booktabs}
\usepackage{amsfonts}
\usepackage{amssymb}
\usepackage{nicefrac}
\usepackage{microtype}
\usepackage{xcolor}
\hypersetup{colorlinks=true,linkcolor=blue,citecolor=blue,urlcolor=blue}

\usepackage{amsmath,bbm,bm}
\usepackage{amsfonts}
\usepackage{amsthm}
\usepackage{mathtools}
\usepackage{algorithm}
\usepackage{algpascal}
\usepackage{algorithmicx}
\usepackage{algpseudocode}
\usepackage{tabularx}
\algdef{SE}[SUBALG]{Indent}{EndIndent}{}{\algorithmicend\ }%
\algtext*{Indent}
\algtext*{EndIndent}

\newtheorem{thm}{Theorem}%
\newtheorem{lem}{Lemma}

\newtheorem{prop}{Proposition}
\newtheorem{cor}{Corollary}

\newtheorem{aspt}{Assumption}

\newtheorem{rmk}{Remark}

\newtheorem{defn}{Definition}

\usepackage{xcolor}
\newcount\comments  %
\newcommand{\genComment}[2]{\ifnum\comments=1{\textcolor{#1}{\textsf{\footnotesize #2}}}\fi}

\def\eqref#1{equation~\ref{#1}}

\def\1{\bm{1}}

\newcommand{\tr}{\textnormal{tr}}

\DeclareMathAlphabet{\mathsfit}{\encodingdefault}{\sfdefault}{m}{sl}
\SetMathAlphabet{\mathsfit}{bold}{\encodingdefault}{\sfdefault}{bx}{n}

\def\gA{{\mathcal{A}}}

\def\gC{{\mathcal{C}}}
\def\gD{{\mathcal{D}}}

\def\gF{{\mathcal{F}}}

\def\gH{{\mathcal{H}}}

\def\gL{{\mathcal{L}}}
\def\gM{{\mathcal{M}}}
\def\gN{{\mathcal{N}}}
\def\gO{{\mathcal{O}}}

\def\gS{{\mathcal{S}}}
\def\gT{{\mathcal{T}}}

\newcommand{\E}{\mathbb{E}}

\newcommand{\R}{\mathbb{R}}

\newcommand{\Cov}{\mathrm{Cov}}

\DeclareMathOperator*{\argmax}{arg\,max}

\newcommand{\PP}{\operatorname{Pr}}

\title{From POMDP Theory to Deep RL with Particle Filters}

\author{%
  Kevin Tan \\
  Wharton Statistics and Data Science\\
  University of Pennsylvania\\
  \texttt{kevtan@wharton.upenn.edu} \\
  \and
  Ziping Xu \\
  School of Data and Information Sciences\\
  University of North Carolina at Chapel Hill\\
  \texttt{zipingxu@unc.edu}
}

\date{}

\begin{document}

\maketitle

\begin{abstract}
    Model-based reinforcement learning in partially observable environments requires inferring hidden states, learning latent dynamics, and planning under the resulting belief. Practical algorithms often solve one aspect of the above pipeline while relying on approximations for the others that are often poorly understood, while existing POMDP theory typically analyzes computationally inefficient theoretical algorithms far from modern deep RL practice. To fill this gap, we propose a practical end-to-end algorithm that approximately solves the hidden state filtering problem via sequential Monte Carlo (SMC), updates a latent dynamics model online, and performs approximate online planning for $d$ steps before bootstrapping leaves with a state-value critic (a QMDP approximation). We analyze each approximation and prove an end-to-end regret bound decomposing the error into model estimation error, particle filter error, Monte Carlo rollout error, and a structural bias term induced by the QMDP approximation controllable by the planning depth that we further analyze in various special cases. The resulting guarantee makes explicit how statistical sample size, particle budget, rollout budget, and planning depth affect performance. We instantiate our algorithm in an MPC-based deep RL framework, demonstrating its practicality on a variety of environments. On a POMDP that satisfies our assumptions, each term of the bound decays at its predicted rate. Our results provide a practical and theoretically grounded route for bridging theory and practice within online decision-making in POMDPs. 
\end{abstract}

\section{Introduction}

Partially observable Markov decision processes (POMDPs) provide a standard framework for sequential decision making under latent state uncertainty \citep{kaelbling1998planning}. They are a natural model for problems in robotics, healthcare, recommendation, and other settings \citep{hauskrecht2000planning,pineau2003point} where the learner must act using only indirect and noisy observations of the underlying system. 
In principle, optimal control in a POMDP can be reduced to planning in the belief-state MDP, where the belief is the posterior distribution over latent states given the observation history \cite{cassandra1994acting,kaelbling1998planning}. This makes model-based RL especially natural in POMDPs: to act under a belief, the learner must already learn a model rich enough to perform filtering, and this same model can be used for planning. 

However, the belief state reduction does not remove the main computational and learning difficulties. An online model-based decision maker must simultaneously (i) learn the unknown model, (ii) maintain a posterior over latent states online, and (iii) plan over this continuous belief space. These three tasks--\textit{learning, inference, and planning}--interact in a way that makes exact online control in POMDPs intractable in both computation and sample complexity in general \cite{papadimitriou1987complexity} without special structure, motivating a large literature on \textit{approximate} methods for each component individually. 

Existing methods typically make one component of the above POMDP pipeline tractable while leaving the others either computationally expensive or theoretically uncontrolled. For instance, on the planning side, Monte Carlo tree search has led to scalable online POMDP solvers such as Partially Observable Monte Carlo Planning (POMCP), sparse tree search, and progressive widening methods \cite{silver2010monte,lim2021sparse,sunberg2018online}. However, these are primarily planning algorithms under a known generative model, and do not natively provide end-to-end learning guarantees for unknown latent dynamics.
A more aggressive approximation is QMDP, which replaces the belief-MDP continuation value by the fully observed MDP value averaged under the current belief \cite{littman1995learning}. QMDP is computationally efficient but can introduce substantial bias. Existing theory gives limited guidance on when this bias is small or how it can be mitigated. 
Regarding inference, particle filters, a.k.a. sequential Monte Carlo (SMC), provide a flexible approximation to Bayesian filtering in nonlinear non-Gaussian systems \cite{crisan2002survey,whiteley2013stability}, but SMC theory for POMDPs \cite{lim2021sparse,lim2023optimalityparticlepomdp} has not been fully integrated with approximate online planning and model learning. On learning, recent theoretical work has made significant progress on sample-efficient model estimation for POMDPs. \cite{liu2022partially,liu2023optimistic} develop Optimistic MLE as a generic model-based algorithm with polynomial sample complexity for a broad class of POMDPs. These results establish statistical feasibility under exact or abstract planning oracles. They do not specify how to perform approximate filtering or planning with the learned model in a computationally practical way.

A separate practical line of work uses recurrent policies or differentiable planning modules. Recurrent model-free deep RL methods, such as GRU or LSTM-based PPO and Q-learning variants, can perform well empirically under partial observability, but their hidden states are difficult to interpret as beliefs and are hard to analyze theoretically. QMDP-net takes a more structured approach by embedding approximate Bayesian filtering and QMDP planning into a differentiable architecture \cite{karkus2017qmdpnet}. This is closer in spirit to our goal: combine model learning, belief state inference, and planning in a way that resembles modern deep learning practice, while maintaining provable guarantees. However, the approximation error of learned variational or differentiable filtering is difficult to control. Our view is that using SMC as the filtering primitive yields a more principled bridge: it is compatible with learned latent dynamics and has well-developed stability and approximation theory \cite{crisan2002survey,whiteley2013stability}. 

In summary, existing work advances planning, inference, or learning separately, but no current method provides a unified pipeline that combines all three with both computational practicality and end-to-end theoretical guarantees. This leaves open the central question
\begin{center}
    \textit{Can we design a practical model-based deep RL pipeline for POMDPs
that still admits end-to-end theoretical guarantees?}
\end{center}

\subsection{Our Contributions}

This paper develops a practical and analyzable route from POMDP theory to model-based deep RL. We answer the central question by developing a model-based learning-inference-planning pipeline for POMDPs called \textbf{PF-MCTQ} that is compatible with deep RL practice. Our contributions are:
\begin{enumerate}
    \item \textbf{A practical inference-learning-planning pipeline.} We propose an online algorithm that learns a latent dynamics model, tracks the current belief with particles, and plans under this belief. Each is modular and can be instantiated with standard deep RL primitives.

    \item \textbf{A controlled alternative to full belief-tree planning.} Instead of treating the particle belief as the planning state, we plan particle-wise and average action-values of the fully observable MDP. This gives a practical middle ground between expensive belief-tree search and cheap but potentially biased QMDP. Our analysis shows that the bias decays geometrically after $d$ steps of belief-MDP lookahead. For cheaper state-level planners, we show that QMDP is exact for the belief-linear component of the value, with error only from an uncertainty-sensitive residual, clarifying when the approximation is small.

    \item \textbf{An end-to-end regret guarantee.}
    We prove a regret bound that decomposes performance loss into model estimation error, particle filtering error, Monte Carlo rollout error, and QMDP structural bias. The bound makes explicit how sample size, particle budget, rollout budget, and planning depth jointly control performance.

    \item \textbf{A bridge to deep RL practice.}
    Compared with theoretical POMDP learning algorithms that rely on exact POMDP solvers or posterior sampling over parameters, our approach uses MPC-style online planning, approximate filtering, and critic bootstrapping. These are standard primitives in model-based deep RL, while remaining explicit enough to support theoretical analysis. We instantiate our algorithm in an MPC-based deep RL implementation and evaluate it on various POMDP environments, demonstrating its practicality. On a POMDP satisfying our assumptions, each term of the bound decays at its predicted rate (Section~\ref{sec:validation}).
\end{enumerate}

Overall, the paper argues that particle filtering can serve as a practical bridge between POMDP theory and deep RL. It gives a belief representation that is explicit enough to analyze, flexible enough for nonlinear latent dynamics, and compatible with MPC-style planning and learned critics. This yields a route to online POMDP learning that is closer to modern model-based RL than exact belief-space planning, while providing a transparent theoretical account that lets users trace each source of approximation error and make informed design choices.

Section~\ref{sec:problem} introduces the online learning problem and the regret criterion in the belief-state MDP. Section~\ref{sec:algorithm} presents the PF-MCTQ algorithm, describing the MCTQ planning subroutine and the full online pipeline. Section~\ref{sec:analysis} develops the error analysis: we decompose the regret into four traceable terms, analyze the structural QMDP bias and its variance interpretation, bound each error source, and combine them into the end-to-end regret guarantee. We conclude with case studies on LQG, contextual MDPs, and smooth-value POMDPs that illustrate when the QMDP approximation is benign.

\vspace{-1ex}
\section{Problem Formulation}
\label{sec:problem}

\begin{figure}[t]
    \centering
    \includegraphics[width=\linewidth]{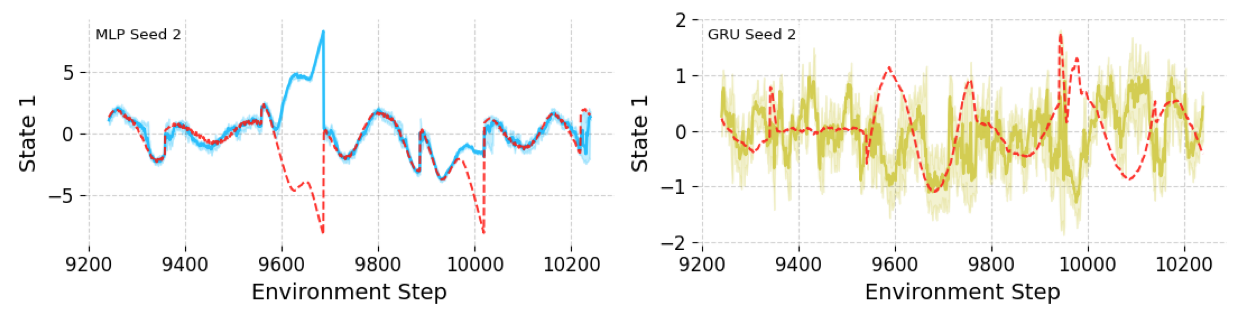}
    \vspace{-4ex}
    \caption{Particle beliefs for the unobserved velocity of a double integrator LQR over the last 1000 steps. Mean, 10th and 90th quantiles visualized. Beliefs track the state when the system is stable.}
    \label{fig:lqr-posterior-small}
    \vspace{-2ex}
\end{figure}

We study online reinforcement learning in a partially observable environment with unknown dynamics. Since the latent state is hidden, the natural state variable for control is the belief, the posterior distribution of the latent state given the history. We first introduce the POMDP model and its belief-state dynamics.
Figure~\ref{fig:lqr-posterior-small} shows a particle belief for a partially observed double integrator (Section~\ref{sec:experiments}).

\paragraph{Parameterized POMDP model.}
For each parameter $\theta \in \Theta$, let $\mathcal M_{\theta} = (\gS,\gA,\gO, P_{\theta}, Z_{\theta}, r,\gamma,b_0)$ be an infinite-horizon discounted POMDP, where $\gS$ is the latent state space, $\gA$ is the action space, and $\gO$ is the observation space. The controlled state transition kernel is $P_{\theta}(\cdot \mid s,a) \in \Delta(\gS)$, the observation kernel is $Z_{\theta}(\cdot \mid s') \in \Delta(\gO)$, the reward function is $r:\gS\times\gA\to\mathbb [0, R_{\max}]$, the discount factor is $\gamma\in(0,1)$, and $b_0\in\Delta(\gS)$ is the initial belief. The true environment is specified by an unknown true parameter $\theta^\star$, and we write $\gM^{\star} := \gM_{\theta^\star}$. Here, for a measurable space $\mathcal X$, $\Delta(\mathcal X)$ denotes the set of probability measures on $\mathcal{X}$, and beliefs $b\in\Delta(\gS)$ are written as measures $b(\mathrm ds)$.

\paragraph{Generative process and belief states.}
At each time $t=0,1,2,\ldots$, the latent state, observation, and reward evolve according to
\begin{equation}\label{eq:generative}
S_{t+1}\sim P_{\theta^\star}(\cdot \mid S_t,A_t),\qquad
O_{t+1}\sim Z_{\theta^\star}(\cdot \mid S_{t+1}),\qquad
R_t = r(S_t,A_t).
\end{equation}
The learner does not observe $S_t$. Instead, decisions are based on the entire history up to time $t$ denoted by $H_t := (O_{1:t},A_{0:t-1},R_{0:t-1})$. Here $O_t$ is available before choosing $A_t$. A possibly randomized policy is a sequence $\pi=\{\pi_t\}_{t\ge0}$ such that $\pi_t(\cdot \mid H_t)\in\Delta(\gA)$.

\textbf{Belief states and Bayes filtering.} Fix a candidate parameter $\theta$. The corresponding belief state at time $t$ is the posterior distribution $b_t^\theta(\cdot) := \mathbb P_{\theta}(S_t\in\cdot \mid H_t)$. When $\theta=\theta^\star$, we write $b_t := b_t^{\theta^\star}$ for the true belief. Given a belief $b\in\Delta(\gS)$ and action $a\in\gA$, define the one-step predictive state distribution and the corresponding predictive observation law respectively by
\begin{equation}\label{eq:predictive_belief}
\tilde b_{\theta}(\cdot \mid b,a)
:= \int_{\gS} P_{\theta}(\cdot \mid s,a)\, b(\mathrm ds),  \qquad p_{\theta}(\cdot \mid b,a)
:= \int_{\gS} Z_{\theta}(\cdot \mid s')\, \tilde b_{\theta}(\mathrm ds' \mid b,a).
\end{equation}
After observing $o\in\gO$, the posterior is updated through the Bayes filter (which we later approximate)
\begin{equation}\label{eq:bayes_update}
\tau_{\theta}(b,a,o)(\mathrm ds) := \frac{Z_{\theta}(o \mid s)\, \tilde b_{\theta}(\mathrm ds \mid b,a)} {\int_{\gS} Z_{\theta}(o \mid s')\, \tilde b_{\theta}(\mathrm ds' \mid b,a)}.
\end{equation}

\paragraph{Belief-MDP value functions.}
The reward associated with belief $b$ and action $a$ is $r(b,a) := \E_{s \sim b}[r(s,a)]$. We assume throughout that $\gA$ is finite and that $|r(s,a)|\le R_{\max}$ for all $(s,a)\in\gS\times\gA$.

For each $\theta$, the POMDP induces a fully observed MDP on $\Delta(\gS)$. For any bounded measurable function $V:\Delta(\gS)\to\mathbb R$, define the belief-MDP Bellman operator
\begin{equation}\label{eq:bellman_operator}
(\gT_{\theta}V)(b)
:=
\max_{a\in\gA}
\left\{
r(b,a)+\gamma \int_{\gO} V\big(\tau_{\theta}(b,a,o)\big)\, p_{\theta}(o \mid b,a)\, \mathrm do
\right\}.
\end{equation}
Let $V_{\theta}^{\star,B}$ denote the unique fixed point of $\gT_{\theta}$. The optimal belief-state action-value function is
\begin{equation}\label{eq:Qstar_belief}
Q_{\theta}^{\star,B}(b,a)
:=
r(b,a)+\gamma \int_{\gO} V_{\theta}^{\star,B}\big(\tau_{\theta}(b,a,o)\big)\, p_{\theta}(o \mid b,a)\, \mathrm do,
\text{ and }
V_{\theta}^{\star,B}(b)=\max_{a\in\gA}Q_{\theta}^{\star,B}(b,a).
\end{equation}
More generally, for a policy $\pi$ acting on histories, let $V_{\theta}^{\pi,B}(b)$ be its discounted value in the belief-MDP from initial belief $b$, and let $Q^{\star, S}_{\theta}$ and $V^{\star, S}_{\theta}$ denote the value functions of the fully observable MDP.

\vspace{-1ex}
\section{The PF-MCTQ Algorithm}
\label{sec:algorithm}

We present PF-MCTQ, a model-based online learning algorithm that integrates model learning, belief inference, and planning into a modular pipeline. We divide time into doubling epochs $E_k=[2^{k-1},2^k)$, $k=1,\ldots,\lceil\log_2 T\rceil$, using the standard doubling trick~\citep{besson2018doubling}. At the start of each epoch, the algorithm (i) updates the MLE $\bar\theta_k$ using all past data and freezes it for the epoch, (ii) maintains a particle approximation $b_t^J$ of the current belief via sequential predict-reweight-resample, and (iii) plans under $\bar{\theta}_k$ and $b_t^J$ using finite-depth Monte-Carlo belief-tree search with QMDP leaf evaluation. The doubling-epoch structure keeps the model from updating too often.

\paragraph{Learning oracle.} Let $\gL_t(\theta)$ be the log-likelihood of the observed history $H_t$ under model $\theta$. At the start of epoch $E_k$, \textsc{ModelUpdate} returns a $\beta$-approximate maximizer on all data up to $T_k$:
\begin{equation}\label{eq:beta_mle}
\bar\theta_k\in\bigl\{\theta\in\Theta:\gL_{T_k}(\theta)\ge \max_{\theta'\in\Theta}\gL_{T_k}(\theta')-\beta\bigr\},
\end{equation}
for a tolerance $\beta\ge 0$. The model $\bar\theta_k$ is frozen throughout epoch $E_k$. When $\beta=0$ this is exact MLE, but we accommodate a $\beta$-suboptimal solver. In practice, many recent gradient-based methods for the MLE in partially observed Markov process (POMP) models apply here \citep{naesseth2018variationalsequentialmontecarlo, tan2024acceleratedinferencepartiallyobserved, chen2025self, ionides15, corenflos21}.

\paragraph{Filtering update.} After playing $a_t$ and observing $o_{t+1}$, particles are propagated and reweighted under the current epoch's model $\bar\theta_{k(t)}$:
\begin{equation}\label{eq:pf_predict_weight}
\hat s^{(t+1)}_j \sim P_{\bar\theta_{k(t)}}(\cdot \mid s^{(t)}_j,a_t),
\qquad
w^{(t+1)}_j \propto Z_{\bar\theta_{k(t)}}(o_{t+1}\mid \hat s^{(t+1)}_j),
\end{equation}
followed by multinomial resampling
$s^{(t+1)}_j \sim \mathrm{Categorical}(w^{(t+1)}_1,\ldots,w^{(t+1)}_J)$ for $j=1,\ldots,J$.
This is the bootstrap sampling-importance-resampling (SIR) filter; any SMC scheme satisfying the $L^p$ error bound of Appendix~\ref{app:pf_proof} can be substituted.

\begin{algorithm}[t]
\caption{\textsc{Pf-Mctq}: Particle filtering with Monte Carlo tree QMDP planning. \textsc{Mctq} (Algorithm~\ref{alg:mctq}) takes the particle belief $b_t^J$ and returns an estimate $\widehat Q_t(b_t^J,a)$ of $Q_{\bar\theta_k,d}(b_t^J,a)$ for every root action $a\in\gA$.}
\label{alg:pf_mcts}
\begin{algorithmic}[1]
\Require particles $J$; depth $d$; rollouts $N$; horizon $T$; tolerance $\beta$; model class $\Theta$; exploration parameter $\eta>0$.
\State Sample $s^{(1)}_1,\ldots,s^{(1)}_J \sim b_0$.
\State Set epochs $E_k=[2^{k-1},\,2^k)$ for $k=1,\ldots,K:=\lceil\log_2 T\rceil$, with $\varepsilon_k:=\eta/\sqrt{|E_k|}$.
\For{epoch $k=1,2,\ldots,K$}
  \State $\bar\theta_k \gets \textsc{ModelUpdate}(H_{T_k})$ via~\eqref{eq:beta_mle}; compute $\widehat Q^{\mathrm{MDP}}_{\bar\theta_k}\approx Q_{\bar\theta_k}^{\star,S}$.
  \For{$t\in E_k$}
    \State $\{\widehat Q_t(b_t^J,a)\}_{a\in\gA}\gets \textsc{Mctq}\bigl(\bar\theta_k,\,b_t^J,\,d,\,N,\,\widehat Q^{\mathrm{MDP}}_{\bar\theta_k}\bigr)$ \Comment{Algorithm~\ref{alg:mctq}}
    \State Play $a_t = \arg\max_a \widehat Q_t(b_t^J,a)$ w.p.\ $1-\varepsilon_k$; else $a_t \sim \text{Unif}(\gA)$.
    \State Observe $(o_{t+1},r_t)$.
    \State Propagate, reweight, and resample particles via~\eqref{eq:pf_predict_weight} under $\bar\theta_k$.
  \EndFor
\EndFor
\end{algorithmic}
\vspace{-1ex}
\end{algorithm}

\paragraph{Approximate planning.} For the current belief, MCTQ runs Monte Carlo tree search on the belief tree for the first $d$ steps and applies the QMDP approximation at the leaves. This gives a middle ground between expensive belief-tree search and the cheap but potentially biased QMDP approximation.

For a belief $b\in\Delta(\gS)$, define the QMDP leaf value and the depth-$d$ belief tree value as 
\begin{equation}
\bar V_{\theta}^{\star,S}(b)
:=
\max_{a\in\gA}\int_{\gS} Q_{\theta}^{\star,S}(s,a)\, b(\mathrm ds), \quad \text{ and } \quad \bar V_{\theta,d}^{\star,S}:=\gT_\theta^d\, \bar V_{\theta}^{\star,S}, \qquad d\ge 0,
\label{eq:U0_def}
\end{equation}
where $Q_{\theta}^{\star,S}$ is the optimal action-value function of the fully observable MDP induced by $\theta$, and $\gT_\theta$ is the Bellman operator of the belief-MDP from~\eqref{eq:bellman_operator}. The corresponding action-value function is
\begin{equation}
Q_{\theta,d}(b,a) := r(b,a)+\gamma \int_{\gO} \bar V_{\theta,d-1}^{\star,S}\bigl(\tau_\theta(b,a,o)\bigr)\, p_\theta(o\mid b,a)\, \mathrm do, \qquad d\ge 1.
\label{eq:Qd_def}
\end{equation}
When $d=0$, $\bar V_{\theta,0}^{\star,S}(b)=\max_a\int Q_\theta^{\star,S}(s,a)\,b(\mathrm ds)$ and MCTQ reduces exactly to QMDP. As $d\to\infty$, $\bar V_{\theta,d}^{\star,S}\to V_\theta^{\star,B}$ (see Proposition~\ref{prop:structural_error} in Appendix~\ref{app:structural_proofs} for the proof).

\paragraph{Particle belief and MCTS approximation.}
In practice, the current belief is represented by a particle approximation $b_t^J=\frac{1}{J}\sum_{j=1}^J \delta_{s_t^{(j)}}$.
Starting from $b^J_t$, MCTQ builds a search tree of depth $d$ with $N$ Monte Carlo simulations per root action, inspired by the POMCP framework~\cite{silver2010monte}.
Each simulation samples a particle from $b_t^J$, descends to depth $d$ using uniform random action selection at internal nodes, expands nodes as it passes through them, and records its state in the particle set $B(v)$ of every node $v$ it visits.
Each depth-$d$ leaf is evaluated by the QMDP value of its pooled particles, $\max_{a'}|B(v)|^{-1}\sum_{s\in B(v)}Q_\theta^{\star,S}(s,a')$, and a bottom-up dynamic program then computes $\widehat Q_t(b_t^J,a)$ for each root action.
We use uniform selection for the analysis and do not claim it is the best rule in practice. The analysis only needs the internal rule to ignore the current tree's statistics and to give every action probability at least $\pi_{\min}>0$, so a frozen learned proposal also qualifies, but UCB1 does not (Appendix~\ref{app:mctq}).

\paragraph{Observation discretization.}
With continuous $\gO$, every simulated observation is distinct, so tree nodes below the root are never revisited and a depth-$d$ tree is not well defined.
MCTQ therefore indexes children by a binned observation $\phi(o)$, where $\phi:\gO\to\bar\gO$ has $M:=|\bar\gO|$ bins of diameter $\mathrm{diam}(\phi)$.
For finite $\gO$ we take $\phi=\mathrm{id}$ and $M=|\gO|$, and for $d=0$ no binning is needed.
The resulting bias and its trade-off against $N$ are made explicit in Theorem~\ref{thm:final_regret_bound}.

\vspace{-1ex}
\section{Error Analysis and Regret Guarantee}
\label{sec:analysis}

We measure performance by the expected cumulative regret defined as:
\begin{equation}\label{eq:regret_def_sec4}
\mathrm{Reg}_T := \E\!\left[\sum_{t=1}^T \text{regret}_t \right] :=\E\!\left[\sum_{t=1}^T\bigl(V_{\theta^\star}^{\star,B}(b_t)-V_{\theta^\star}^{\pi_t,B}(b_t)\bigr)\right],
\end{equation}
where $\pi_t$ is the policy deployed from round $t$, and the expectation is over the trajectory and algorithmic randomness. A key contribution is an end-to-end regret bound for PF-MCTQ, enabled by the following novel regret decomposition. Let $k(t)$ be the epoch containing round $t$, with frozen model $\bar\theta_{k(t)}$:
\begin{align}
\text{regret}_t
=
\underbrace{V_{\theta^\star}^{\star,B}(b_t)-V_{\bar\theta_{k(t)}}^{\star,B}(b_t)}_{\text{(A) optimal-value model error}}
+
\underbrace{V_{\bar\theta_{k(t)}}^{\star,B}(b_t)-V_{\bar\theta_{k(t)}}^{\pi_t,B}(b_t)}_{\text{(B) planning + state-id error}}
+
\underbrace{V_{\bar\theta_{k(t)}}^{\pi_t,B}(b_t)-V_{\theta^\star}^{\pi_t,B}(b_t)}_{\text{(C) fixed-policy model error}}, \label{eq:main_decomp_round_t}
\end{align}
Terms (A) and (C) capture model estimation error: the gap between $\bar\theta_{k(t)}$ and $\theta^\star$. Term (B) captures planning approximation under the learned model, which is further decomposed into three sources:
\begin{equation}
 \text{(B)} \;\le\;
\underbrace{\bigl|V_{\bar\theta_{k(t)}}^{\star,B}(b_t)-V_{\bar\theta_{k(t)}}^{\star,B}(b_t^J)\bigr|}_{\text{(B1) belief approx.}}
+
\underbrace{V_{\bar\theta_{k(t)}}^{\star,B}(b_t^J)-V_{\bar\theta_{k(t)}}^{\pi_t,B}(b_t^J)}_{\text{(B2) planning suboptimality}}
+
\underbrace{\bigl|V_{\bar\theta_{k(t)}}^{\pi_t,B}(b_t^J)-V_{\bar \theta_{k(t)}}^{\pi_t,B}(b_t)\bigr|}_{\text{(B1$'$) belief approx.}}. \label{eq:B_decomp}
\end{equation}
Terms (B1) and (B1$'$) evaluate values at the particle belief $b_t^J$ instead of the posterior $b_t$, each bounded by $O(1/J^{1/(2\kappa_{\mathrm{pf}})})$ via particle-filter concentration.
The planning gap (B2) splits into Monte Carlo error from MCTS rollouts ($O(1/\sqrt{N})$) and structural error from QMDP bootstrapping at depth $d$.

We distinguish the belief-MDP optimum $V_\theta^{\star,B}$, the deterministic \emph{planning target} $\bar V_{\theta,d}^{\star,S}=\gT_\theta^d\bar V_\theta^{\star,S}$ of~\eqref{eq:U0_def}, and the randomized MCTQ estimate $\widehat Q_t(b_t^J,a)$ from pooled particle sets.
The target differs from the optimum by the structural error, which decays geometrically in $d$ and involves no particles. The estimate differs from the target by the Monte Carlo error, which grows with $d$, so for fixed $N$ the bound is U-shaped in $d$.
Each component enters only through its guarantee (a $\beta$-approximate MLE, an SMC filter with an $L^p$ bound, a planner with a Monte Carlo bound), so the decomposition extends to, e.g., POMCP-style planners or variational filters.

\subsection{Assumptions and main result}
\label{sec:final_regret}

Modern SMC analyses \cite{delmoral2004feynmankac, chopin20, karjalainen2026forgettingparticlefilters} impose rather unavoidable but stringent conditions on the transition kernel and observation density. We inherit these in the form of the following Doeblin and bounded observation density assumption (Assumption \ref{aspt:particle_error}) on the latent dynamics. This choice aligns with our aim to provide an end-to-end regret bound that decomposes error across the learning–inference–planning pipeline, rather than to develop optimal exploration methods for POMDP model estimation, as Assumption \ref{aspt:particle_error} ensures that simple exploration strategies like $\varepsilon$-greedy have sufficient coverage. At the same time, the Doeblin condition resembles commonplace strong mixing assumptions within the RL literature \citep{gaur2024closinggapachievingglobal, wu2022finitetimeanalysistimescale, xu2021improvingsamplecomplexitybounds, kausik2023mixtures} and the $\alpha$-revealing condition commonplace in the POMDP literature plays a similar role to the bounded observation density condition \citep{liu2022partially,liu2023optimistic}.

\begin{aspt}[Uniform mixing and bounded weights]\label{aspt:particle_error} \text{ }
\begin{enumerate}
\item \emph{Uniform Doeblin condition.} There exist $\varepsilon_{\mathrm{mix}}>0$ and a probability measure $\nu\in\Delta(\gS)$ such that $P_\theta(\cdot\mid s,a)\ge \varepsilon_{\mathrm{mix}}\,\nu(\cdot)$ for all $(\theta,s,a)\in\Theta\times\gS\times\gA$.
\item \emph{Bounded observation density.} There exist $0<g_{\min}\le g_{\max}<\infty$ such that $g_{\min}\le Z_\theta(o\mid s)\le g_{\max}$ for all $(\theta,s,o)\in\Theta\times\gS\times\gO$.
\end{enumerate}
\end{aspt}

Let $\gN_\Theta(\varepsilon)$ be the $\varepsilon$-covering number of $\Theta$, the minimum number of $\ell_1$-balls of radius $\varepsilon$ (in trajectory likelihood) needed to cover $\Theta$. Define the effective model complexity $\beta_T:=\beta+\log\bigl[\gN_\Theta(1/T)\,T/\delta\bigr]$, where $\beta$ is the MLE tolerance from~\eqref{eq:beta_mle} and $\delta\in(0,1)$ is a confidence parameter.

\begin{thm}[End-to-end regret of \textsc{Pf-Mctq}]\label{thm:final_regret_bound}
Suppose Assumption~\ref{aspt:particle_error} holds. Algorithm~\ref{alg:pf_mcts}, run with exploration parameter $\eta\asymp T^{1/4}(\beta_T/\log T)^{1/4}$ and observation binning $\phi$ with $M$ bins of diameter $\mathrm{diam}(\phi)=\omega_Z^{-1}\bigl((1-\gamma)g_{\min}/(2\gamma g_{\max}V_{\max}\sqrt N)\bigr)$ ($\phi=\mathrm{id}$ and $M=|\gO|$ when $\gO$ is finite), and, if $d\ge2$, $N$ above the threshold of Remark~\ref{rmk:mc_constants}, and $\kappa_{\mathrm{pf}}:=\max \left(1, {\log \left(2 g_{\max }^2 / g_{\min }^2\right)}/{\log (1 / \gamma)}\right)$ satisfies
\begin{align}
\mathrm{Reg}_T
\;\le\;
\underbrace{\widetilde O\!\left(\frac{C_{\mathrm{est}}\, R_{\max}}{(1-\gamma)^{3/2}}\,T^{3/4}\right)}_{\text{model + exploration + filter}}
+
\underbrace{\frac{C_{\mathrm{pf}}\,R_{\max}}{(1-\gamma)^2}\cdot\frac{T}{J^{1/(2\kappa_{\mathrm{pf}})}}}_{\text{particle approx.}}
 +
\underbrace{\frac{C_{\mathrm{mc}}\,R_{\max}}{(1-\gamma)^2}\cdot\frac{T}{\sqrt{N}}}_{\text{Monte Carlo}}
+
\underbrace{\frac{2\gamma^{d+1}}{(1-\gamma)^2}\sum_{t=1}^T\Xi_{\bar{\theta}_{k(t)}}}_{\text{QMDP structural bias}}, \nonumber
\end{align}
\end{thm}
Here $\widetilde O$ hides factors polylogarithmic in $T$, $\omega_Z$ is the modulus of continuity of the observation likelihood ratio (Appendix~\ref{app:planning_state_proofs}), $M=|\bar\gO|$ is the number of observation bins, and $\Xi_{\bar{\theta}_{k(t)}}$ is the QMDP structural bias of~\eqref{eq:Xi_def}.
The coefficients are
\begin{align*}
    C_{\mathrm{est}} &= \mathrm{Poly}(|\gA|, g_{\max}/g_{\min}, \epsilon_\mathrm{mix}^{-1}, \log(1/\gamma), \log\gN_\Theta(1/T), \beta), \quad C_{\mathrm{mc}} = \textstyle{O(d\gamma^d  |\gA|^{(d+3)/2}M^{d/2})},\\
    C_{\mathrm{pf}} &=\widetilde{O}\bigl(V_{\max}^2\,g_{\max}/(\varepsilon_{\mathrm{mix}}^3\,g_{\min})\cdot(\log J)^{3/2}\sqrt{\log(\gN_\Theta(1/T)\,|\gA|)}\bigr) \;\; \text{ (See ~\eqref{eq:Cpf_explicit})}.
\end{align*}

\textbf{Reading the bound.} Theorem \ref{thm:final_regret_bound} bounds cumulative regret with four major terms. The first term captures model estimation error, exploration cost, and filter mismatch between the true posterior $b_t^{\theta^{\star}}$ and the particle belief $b_t^{\bar\theta_{1:K}}$. The $\widetilde O(T^{3/4})$ rate arises from the coverage transfer needed to bound the optimal-value model error (term~(A) in~\eqref{eq:main_decomp_round_t}) via $\varepsilon$-greedy exploration. Terms 2, 3, and 4 capture the total approximate planning errors that may scale linearly in $T$. They can be diminished by increasing the number of particles $J$, number of Monte Carlo rollouts $N$, and the planning depth $d$.
The particle error decays as $J^{-1/(2\kappa_{\mathrm{pf}})}$ at cost $O(J)$, the Monte Carlo error as $N^{-1/2}$ at cost $O(N)$, and the structural error as $\gamma^{d}$ at a tree size of up to $(|\gA|M)^{d}$ nodes.
The bound therefore implies a \textbf{sample-vs.-computational efficiency trade-off} through the marginal return of each budget, but a concrete allocation requires knowledge of $\kappa_{\mathrm{pf}}$ and $\Xi_\theta$ and is not claimed to be optimal. Section~\ref{sec:validation} gives an empirical budgeting rule on an instance where the assumptions hold.
With continuous observations, the binning keeps the discretization bias at the Monte Carlo order $O(V_{\max}/\sqrt N)$, but $M$ grows as $N^{d_o/2}$ for $\gO\subseteq\R^{d_o}$, so $C_{\mathrm{mc}}$ grows exponentially in $d\cdot d_o$. Balancing bin width against $N$ instead gives the rate $\widetilde O(N^{-1/(d\,d_o+2)})$ (Remark~\ref{rmk:disc_resolution}).
This cost is absent only at $d=0$, the regime used in our deep RL experiments.

\textbf{Technical challenges.} Bounding each source in~\eqref{eq:B_decomp} under online learning raises three difficulties. First, controlling (B1) requires working with a test-function class defined directly over $\gS$, rather than only the parametric family $\{Q^{\star,B}_\theta\}_\theta$. This is because the belief-space value functions may lack the regularity needed for uniform particle-filter error control. In fact, the planning-induced class over multiple steps is effectively uncountable over the possible model parameters, so we control it through truncation. The required truncation level depends on the desired particle approximation accuracy: as the particle error target decreases, the tail must be made smaller, which enlarges the truncated function class. This creates a tradeoff between the complexity of the truncated class and the particle ensemble size, leading to a slower ($1/J^{1/2\kappa_{\mathrm{pf}}}$), essentially nonparametric, rate.

Second, the filtering analysis must handle data generated online by a changing policy. Classical particle-filter stability assumes a fixed generative model, whereas here the filter receives actions from a non-stationary, improving policy. We extend time-uniform $L^p$ bounds~\citep{delmoral2004feynmankac} to this setting by exploiting the Doeblin condition uniformly over actions. Third, Doeblin coverage applies to states but not actions, so we ensure action coverage via $\varepsilon$-greedy exploration. This lets us transfer MLE concentration across policies, controlling terms (A) and (C) and the Monte Carlo error in (B2).

\begin{rmk}[Relation to POMDP hardness]
POMDP sample complexity may in the worst case scale exponentially in the effective horizon $1/(1-\gamma)$.
Our bound is polynomial in $1/(1-\gamma)$ for two reasons.
First, under the Doeblin condition the Bayes filter forgets its initialization at a geometric rate $\rho=1-\varepsilon_{\mathrm{mix}}g_{\min}/g_{\max}<1$ uniformly over actions, so errors add up over time rather than compound.
Second, bounded observation densities play the role of the $\alpha$-revealing condition~\citep{liu2022partially}: they bound importance weights, so data transfer across policies without an $(|\gA|/\varepsilon)^{H}$ blow-up, and particle weights stay non-degenerate.
This does not circumvent the lower bound: $\varepsilon_{\mathrm{mix}}^{-1}$ and $g_{\max}/g_{\min}$ can be exponential in the effective horizon for hard instances, and $C_{\mathrm{est}}$ and $C_{\mathrm{pf}}$ depend polynomially on them.
Appendix~\ref{app:stress} illustrates failures when $g_{\min}=0$ and when $\Xi_\theta$ is large.
\end{rmk}

\begin{rmk}[Bypassing exploration with offline coverage]
The $\widetilde O(T^{3/4})$ rate is driven by the coverage transfer for term~(A), which requires $\varepsilon$-greedy exploration to ensure model accuracy at the optimal policy. If instead the agent has access to $N_{\mathrm{off}}$ offline trajectories from a behavior policy $\pi_b$ with single-policy concentrability $c_{\mathrm{off}}^\star$ (the worst-case ratio of trajectory $\ell_1$ distances between $\pi^\star$ and $\pi_b$), then greedy exploitation ($\varepsilon_t=0$) achieves regret $\widetilde O(\sqrt{T}/(1-\gamma)) + O(c_{\mathrm{off}}^\star T/(\sqrt{N_{\mathrm{off}}}(1-\gamma)))$ plus computational terms. When $N_{\mathrm{off}} \ge (c_{\mathrm{off}}^\star)^2 T$, the total regret is $\widetilde O(\sqrt{T})$ without any additional online exploration. The formal statement and proof are in Appendix~\ref{app:offline_coverage}.
\end{rmk}

\subsection{Understanding the structural bias}
\label{sec:understanding_structural_bias}

It remains to discuss the QMDP approximation error, which is governed by the \emph{Jensen gap}
\begin{equation}
\Xi_\theta
:=
\sup_{b\in\Delta(\gS)}
\left(
\int_{\gS}V_\theta^{\star,S}(s)\,b(\mathrm ds)
-
\max_{a\in\gA}\int_{\gS}Q_\theta^{\star,S}(s,a)\,b(\mathrm ds)
\right),
\label{eq:Xi_def}
\end{equation}
which measures the value lost by choosing a single action for all states (QMDP planning) rather than the state-dependent optimal action. We show (Proposition~\ref{prop:structural_error} in Appendix~\ref{app:structural_proofs}) that the depth-$d$ MCTQ structural error satisfies $\|\bar V_{\theta,d}^{\star,S}-V_{\theta}^{\star,B}\|_\infty\le\frac{\gamma^{d+1}}{1-\gamma}\Xi_\theta$, decaying geometrically with depth $d$.

The quantity $\Xi_\theta$ has an interpretation via the posterior \emph{variance} of value-relevant features. Let $g:\gS\to\R^k$ be a measurable embedding, and write
$m_g(b):=\E_{S\sim b}[g(S)]$ and $\Sigma_g(b):=\Cov_{S\sim b}(g(S))$.

\begin{aspt}[Value-relevant embedding]\label{aspt:value_embedding_main}
There exist twice continuously differentiable maps $F_\theta^a:\R^k\to\R$ for each $a\in\gA$ such that
$Q_\theta^{\star,S}(s,a)=F_\theta^a(g(s))$ for all $(s,a)\in\gS\times\gA$, with curvature bounds, where $F_\theta(x):=\max_a F_\theta^a(x)$ and $V_\theta^{\star,S}(s)=F_\theta(g(s))$, of
\[
\sup_{a\in\gA}\sup_{x\in\mathrm{conv}(g(\gS))}\bigl\|\nabla^2 F_\theta^a(x)\bigr\|_{\mathrm{op}}\le L_\theta
\qquad\text{and}\qquad
\sup_{x\in\mathrm{conv}(g(\gS))}\bigl\|\nabla^2 F_\theta(x)\bigr\|_{\mathrm{op}}\le L_\theta
,\]
where $\|\cdot\|_{\mathrm{op}}$ is the spectral norm.
\end{aspt}
In practice, the condition holds when either (i) a single action is optimal throughout $\mathrm{conv}(g(\gS))$, or (ii) $F_\theta$ is a single smooth function of $s$ (as in linear-quadratic-Gaussian (LQG) systems, where it is a quadratic form and both conditions hold with the same $L_\theta$).

\begin{prop}[Curvature bound on $\Xi_\theta$]\label{prop:variance_qmdp_bias_main}
Under Assumption~\ref{aspt:value_embedding_main}, for every belief $b\in\Delta(\gS)$,
\begin{equation}
\Xi_\theta(b)
:=\widetilde V_\theta^{\star}(b)-\bar V_\theta^{\star,S}(b) 
\;\le\;
L_\theta\,\tr\bigl(\Sigma_g(b)\bigr), 
\label{eq:variance_xi_bound_main}
\end{equation}
where $\widetilde V_\theta^{\star}(b):=\int_{\gS}V_\theta^{\star,S}(s)\,b(\mathrm ds)$ is the linear QMDP surrogate.
In particular, $\Xi_\theta=\sup_b\Xi_\theta(b)$.
\end{prop}

The proof is in Appendix~\ref{app:variance_proofs}.
That QMDP is accurate once the belief concentrates is expected; Proposition~\ref{prop:variance_qmdp_bias_main} is sharper in two ways.
First, only the posterior covariance of the value-relevant features $g$ matters; uncertainty in directions to which the value is insensitive does not enter, so the full belief need not concentrate.
Second, the curvature $L_\theta$ converts this covariance into a value gap, which enters Theorem~\ref{thm:final_regret_bound} as $\gamma^{d+1}\Xi_\theta/(1-\gamma)$ and trades off against planning depth.
In discounted LQG, for instance, the Kalman covariance converges to a strictly positive steady state and never vanishes, yet the QMDP policy is belief-optimal, and $\Xi_\theta=0$ when the control does not enter the dynamics (Appendix~\ref{app:case_study_proofs}).

Appendix~\ref{app:case_study_proofs} instantiates the structural analysis on linear-quadratic Gaussian systems, where $\Xi_\theta$ has a closed form in terms of the Riccati solution and the steady-state Kalman covariance, and on contextual MDPs with identifiable latent context, where $\Xi_{\theta,t}\le 2V_{\max}\rho_t$ for a context-misidentification probability $\rho_t$, so eventual identification becomes the rate condition $\sum_t\rho_t=o(\sqrt T)$.

\vspace{-1ex}
\section{Experiments}
\label{sec:experiments}

\begin{figure}[t]
    \centering
    \includegraphics[width=\linewidth]{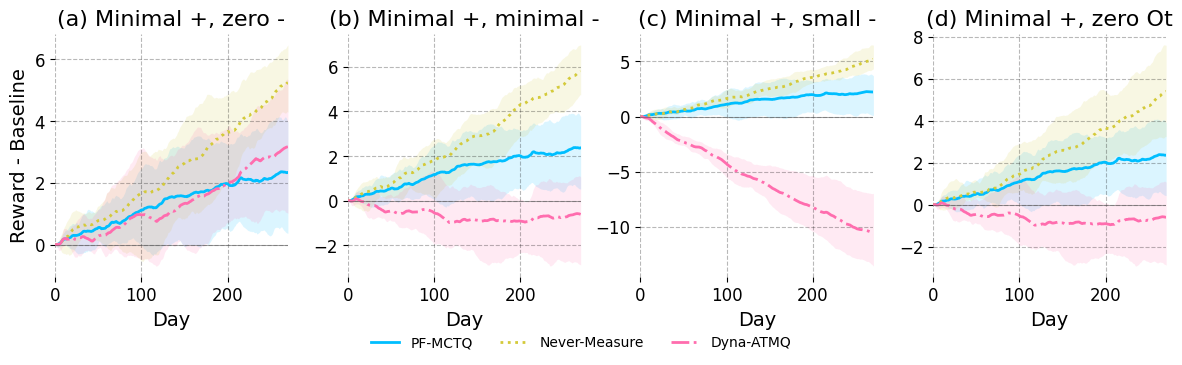}
    \vspace{-4ex}
    \caption{Cumulative reward above the baseline rollout in four HeartSteps scenarios varying the positive intervention effect and the delayed negative measurement effect or proxy $O_t$ informativeness. Mean and 2.5\%/97.5\% quantiles over five seeds.}
    \label{fig:heartsteps}
    \vspace{-1ex}
\end{figure}

Our experiments test whether particle filtering and model-based planning can be used as a practical layer on top of modern deep reinforcement learning under partial observability.
Here, our goal is not to obtain state-of-the-art asymptotic performance.
Instead, we aim to evaluate whether explicit particle beliefs can improve low-sample decision making with modest additional overhead.
Table~\ref{tab:theory_coverage} states which experiments fall under the theory. The deep RL experiments demonstrate that the pipeline is practical but are not evidence for the constants of Theorem~\ref{thm:final_regret_bound}.

\begin{table}[t]
\centering
\caption{Which experiments satisfy the assumptions of Theorem~\ref{thm:final_regret_bound}. \checkmark: holds; $\times$: fails; $\sim$: approx.}
\label{tab:theory_coverage}
\small
\begin{tabular}{lccccl}
\toprule
Environment & Doeblin & $g_{\min}>0$ & Finite $\gA$ & Bounded $r$ & Role \\
\midrule
Mixing Tiger (Sec.~\ref{sec:validation}) & \checkmark & \checkmark & \checkmark & \checkmark & theory-valid \\
Active sensing (App.~\ref{app:stress}) & \checkmark & $\sim$ & \checkmark & \checkmark & structural diagnostic \\
Box search (App.~\ref{app:stress}) & $\times$ & $\times$ & \checkmark & \checkmark & $g_{\min}=0$ by design \\
HeartSteps & $\sim$ & $\sim$ & \checkmark & \checkmark & not claimed \\
Double-integrator LQR & \checkmark & $\times$ & $\times$ & $\times$ & filtering diagnostic \\
MuJoCo & $\times$ & $\times$ & $\times$ & $\times$ & heuristic demonstration \\
\bottomrule
\end{tabular}
\vspace{-2.5ex}
\end{table}

\textbf{Empirical validation of the regret decomposition.}\label{sec:validation}
On a Doeblin-mixing Tiger~\citep{kaelbling1998planning} that satisfies Assumption~\ref{aspt:particle_error} exactly ($\varepsilon_{\mathrm{mix}}=0.05$, $0.15\le Z\le0.85$, $\gamma=0.8$; Appendix~\ref{app:validation}), varying one budget at a time shows each term decaying at its predicted rate (Table~\ref{tab:validation}).
The smallest tested $(J,N,d)$ putting every isolated term below $0.10$, $0.05$, and $0.025$ is $(16,16,6)$, $(64,64,9)$, and $(128,256,12)$.

\begin{table}[t]
\centering
\caption{Fitted vs.\ predicted rates on mixing Tiger; Theorem~\ref{thm:final_regret_bound} gives slower $J^{-1/(2\kappa_{\mathrm{pf}})}$, $T^{3/4}$.}
\label{tab:validation}
\small
\begin{tabular}{llll}
\toprule
Term & Error at endpoints & Fitted & Predicted \\
\midrule
Particle filtering, $J=8\to512$ & $0.144\to0.012$ & $J^{-0.59}$ & $J^{-1/2}$ \\
Monte Carlo, $N=4\to512$ & $0.171\to0.015$ (RMSE) & $N^{-0.50}$ & $N^{-1/2}$ \\
Structural, $d=0\to12$ & $0.314\to0.0216$ & $0.80$ per level & $\gamma=0.80$ \\
Estimation (online MLE), $\sum_{t\le T}\varepsilon_{\mathrm{opt}}(t)$ & & $T^{0.57}$ & $\sqrt{T}$ \\
\bottomrule
\end{tabular}
\vspace{-2ex}
\end{table}

\textbf{Deep RL experiments.} We instantiate PF-MCTQ using the \texttt{pypomp} \cite{pypomp2024github} package for particle-filtering primitives
with MOP \cite{tan2024acceleratedinferencepartiallyobserved} or VSMC \cite{naesseth2018variationalsequentialmontecarlo} gradient estimators. This allows highly optimized JAX particle filters on GPU. We learn a neural dynamics
model online, maintain a particle approximation to the latent state posterior,
and use an actor and a critic trained with recurrent SAC-style updates \citep{ni2021recurrentbaseline} in the same vein as \cite{han2026neuralodesdemodels}.  In the experiments we use belief-tree depth $d=0$: each particle is rolled out at the state level under a learned actor and bootstrapped with a learned critic, rather than planning over a tree of particle beliefs.
With continuous actions and observations a depth-$d$ tree requires binning at a cost of up to $(|\gA|M)^d$ nodes, so this is the main practical gap from the theory.
Rollout planning corresponds to the QMDP case $d=0$, whereas deeper belief planning reduces the structural bias. Appendix~\ref{app:planners} runs the analyzed MCTQ at $d>0$ on finite POMDPs. Algorithm \ref{alg:pf_mcts_practical} gives the practical algorithm. All experiments took about three days on two RTX 3090 GPUs. Code is available at \url{https://github.com/pypomp/pf_control}.

\textbf{Double-integrator LQR.} We first study a partially observed double-integrator LQR.  The state contains
position and velocity, but the observation reveals only position. Thus the
latent velocity is therefore not identifiable from the current observation alone. Different
velocities induce the same observation and can only be distinguished using
history.  This environment isolates the filtering problem.  The posterior
diagnostics in Figure \ref{fig:lqr-posterior-small} show that particle beliefs can track the latent state without access to memory, while failures occur when the system is unstable.
The GRU backbone improves this behavior relative to a memoryless MLP, suggesting that particle filtering and recurrent state representations are
complementary. More details and figures can be found in Appendix \ref{app:lqr}.

\textbf{HeartSteps.} We next evaluate the HeartSteps benchmark of \cite{gao2025activemeasuringreinforcementlearning} as an actively observable MDP with a latent reward
state, where the learner can choose to measure the state at a cost. PF-MCTQ is run in the never-measure regime. It never directly measures
the latent reward and must infer it from passive observations.  The strongest
benchmark is RLSVI-SMC/never-measure, which is specifically tailored to the
linear-Gaussian HeartSteps structure and uses hand-specified linear features.
Consequently, RLSVI-SMC performs best in most scenarios.  PF-MCTQ typically
improves over the baseline rollout but trails RLSVI-SMC because it must learn a deep
model, actor, and critic from data.  Dyna-ATMQ is not directly comparable as a
no-measure method. It actively measures the latent reward roughly half the time,
which gives it an informational advantage, but this advantage can be offset when
measurement has delayed negative effects on future engagement. Details are in Appendix~\ref{app:heartsteps}.

\begin{figure}[t]
    \centering
    \includegraphics[width=\linewidth]{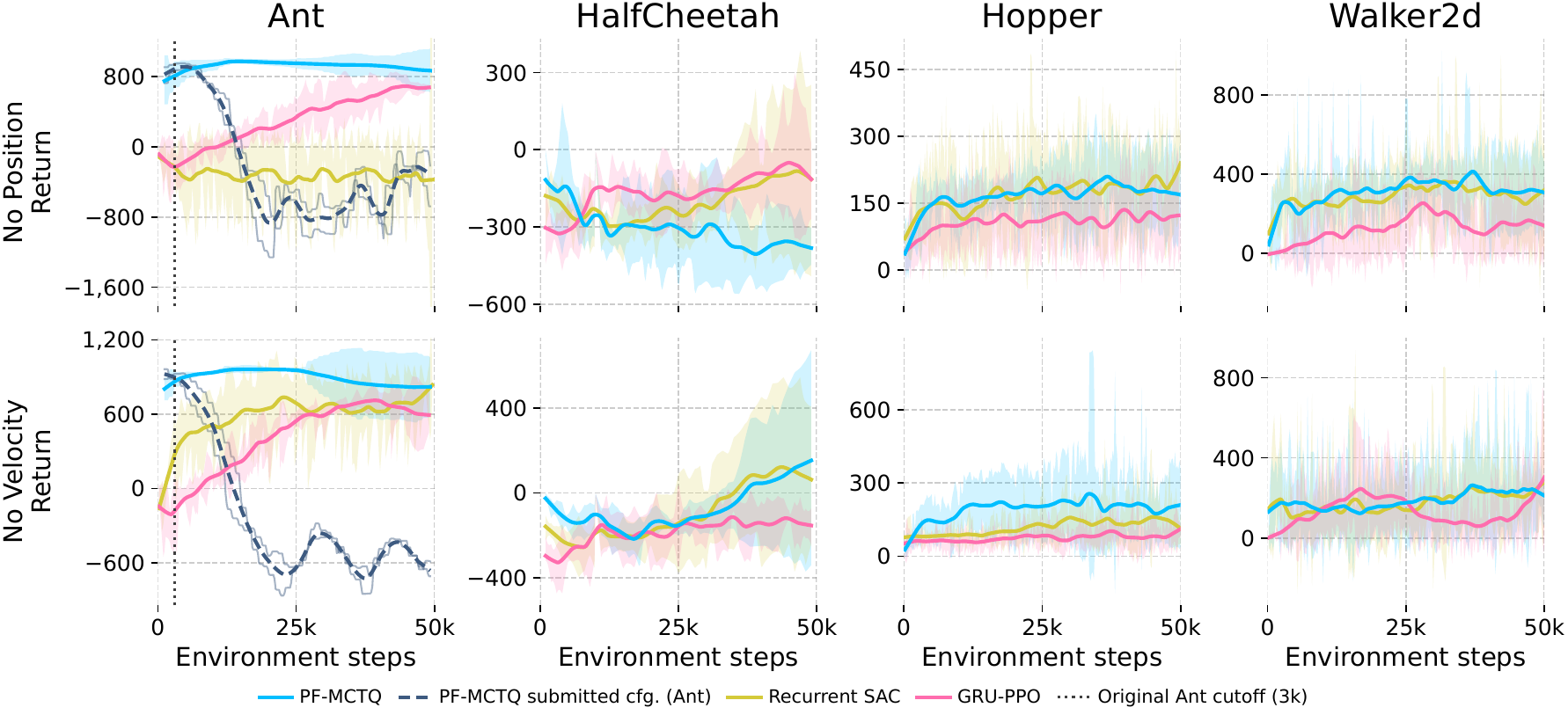}
    \vspace{-4ex}
    \caption{Partial-observation MuJoCo learning curves in the low-sample regime: smoothed means with 95\% Student-$t$ bands over five seeds (Ant: three seeds for corrected PF-MCTQ; the submitted configuration, dashed, shows its two seeds as thin lines). Ant PF-MCTQ shows the 10-episode rolling training return; others, per-episode returns. Dotted: the submitted 3k-step cutoff; none extrapolated.}
    \label{fig:mujoco}
    \vspace{-2.5ex}
\end{figure}

\textbf{MuJoCo partial-observation control.}
Finally, we evaluate Ant, HalfCheetah, Hopper, and Walker2d with position
or velocity removed from the observation.  We compare PF-MCTQ to
Recurrent SAC and GRU-PPO baselines from the recurrent model-free POMDP
benchmarking literature~\citep{ni2021recurrentbaseline}.  These are strong comparators
that learn a recurrent history representation
directly from reward. Figure~\ref{fig:mujoco} shows that PF-MCTQ is competitive in
the ultra-low-sample regime.  On Ant, PF-MCTQ learns a high-return policy within
the first few thousand steps under both ablations.
We no longer stop Ant early. When run to $49.2$k steps, the submitted configuration's policy collapses (dashed). After restoring online model updates and ordinary weight decay, a controlled reduction of the policy update-to-data ratio (one update per ten steps) is decisive and prevents this (Appendix~\ref{app:ant}). Over the prespecified 40k--49.2k window, corrected PF-MCTQ averages $893\pm177$ and $822\pm268$ (position/velocity hidden; three seeds, 95\% Student-$t$) against $-396$ and $-518$ for the submitted configuration.
On Hopper and
Walker2d, PF-MCTQ is broadly competitive with Recurrent SAC and usually stronger
than GRU-PPO early in training.  HalfCheetah is more mixed: PF-MCTQ performs well
when velocity is hidden but struggles when position is hidden, suggesting that
model bias and rollout error can dominate when the omitted coordinates are
crucial for short-horizon prediction.

On Ant, corrected PF-MCTQ needs about 30\% of recurrent SAC's training time per step plus a one-time 86.3\,s model warm start, while GRU-PPO is faster still (Table~\ref{tab:runtime}).
Overall, these results show that particle filtering and model-based planning can be implemented on top of existing deep RL algorithms and can yield sample-efficiency gains in partially observed control.  The results should not be read as a claim
of asymptotic dominance over recurrent model-free RL.  Rather, they show that the
learning--filtering--planning pipeline can be made operational in
standard deep RL benchmarks.

\vspace{-1ex}
\section{Discussion}
\label{sec:conclusion}

We presented an online model-based POMDP pipeline combining particle filtering, latent-dynamics learning, and approximate planning with an end-to-end regret bound. We decompose regret into model estimation, particle approximation, Monte Carlo planning, and QMDP structural errors, clarifying how statistical and computational budgets affect performance. The framework bridges POMDP theory and model-based deep RL. Experiments with a simplified rollout-based version show that it can be made operational and competitive in low-sample partially observed control.

\textbf{Limitations.} Our analysis relies on the Doeblin mixing condition and bounded observation densities, which ensure passive state coverage and obviate the need for strategic exploration.
These assumptions fail in many POMDPs.
In item search with exact observations, the observation density is not bounded below, and a particle set of size $J$ misses the true location among $K$ with probability $(1-1/K)^J$. In Light--Dark-type problems, information must be actively gathered, so $\Xi_\theta$ is large and shallow planning never gathers it (see Appendix~\ref{app:stress}).
Second, the bounded reward assumption excludes settings such as LQG.
Third, the deep RL experiments use $d=0$ with a learned critic and actor, which is the QMDP case of our algorithm, and our guarantees apply to non-adaptive action selection inside the tree but not to UCT-style selection.
Natural extensions are strategic exploration to relax the dynamics assumptions, unbounded rewards, and guarantees for adaptive tree search.

\section*{Acknowledgments}
Ziping Xu was supported by the National Science Foundation under award DMS-2625746.

\bibliographystyle{plain}
\bibliography{main.bib}

\begin{thebibliography}{10}

\bibitem{pypomp2024github}
Aaron Abkemeier, Jun Chen, Kevin Tan, Jesse Wheeler, and Edward Ionides.
\newblock pypomp: Partially observed {M}arkov processes in {P}ython, version
  0.4.4.10.
\newblock \url{https://github.com/pypomp/pypomp}, 2026.

\bibitem{anjarlekar2026scalable}
Ameya Anjarlekar, S.~Rasoul Etesami, and R.~Srikant.
\newblock Scalable policy-based {RL} algorithms for {POMDP}s.
\newblock In {\em The Thirty-ninth Annual Conference on Neural Information
  Processing Systems}, 2025.

\bibitem{besson2018doubling}
Lilian Besson and Emilie Kaufmann.
\newblock What doubling tricks can and can't do for multi-armed bandits.
\newblock {\em arXiv preprint arXiv:1803.06971}, 2018.

\bibitem{cassandra1994acting}
Anthony~R Cassandra, Leslie~Pack Kaelbling, and Michael~L Littman.
\newblock Acting optimally in partially observable stochastic domains.
\newblock In {\em Proceedings of the Twelfth National Conference on Artificial
  Intelligence (AAAI)}, pages 1023--1028, 1994.

\bibitem{chen2025self}
Yuan Chen, Mathieu Gerber, Christophe Andrieu, and Randal Douc.
\newblock Self-organizing state-space models with artificial dynamics.
\newblock {\em Journal of the Royal Statistical Society Series B: Statistical
  Methodology}, page qkaf069, 2025.

\bibitem{chopin20}
Nicolas Chopin and Omiros Papaspiliopoulos.
\newblock {\em An Introduction to Sequential {Monte Carlo}}.
\newblock Springer Series in Statistics. Springer, 2020.

\bibitem{corenflos21}
Adrien Corenflos, James Thornton, George Deligiannidis, and Arnaud Doucet.
\newblock Differentiable particle filtering via entropy-regularized optimal
  transport.
\newblock In Marina Meila and Tong Zhang, editors, {\em Proceedings of the 38th
  International Conference on Machine Learning}, volume 139 of {\em Proceedings
  of Machine Learning Research}, pages 2100--2111. PMLR, 18--24 Jul 2021.

\bibitem{crisan2002survey}
Dan Crisan and Arnaud Doucet.
\newblock A survey of convergence results on particle filtering methods for
  practitioners.
\newblock {\em IEEE Transactions on signal processing}, 50(3):736--746, 2002.

\bibitem{delmoral2004feynmankac}
Pierre Del~Moral.
\newblock {\em Feynman-Kac Formulae: Genealogical and Interacting Particle
  Systems with Applications}.
\newblock Springer, 2004.

\bibitem{doucmoulinesolsson2014}
Randal Douc, Eric Moulines, and Jimmy Olsson.
\newblock Long-term stability of sequential {M}onte {C}arlo methods under
  verifiable conditions.
\newblock {\em Annals of Applied Probability}, 24(5):1767--1802, 2014.

\bibitem{du2021particle}
Simon~S Du, Wei Hu, Zhiyuan Li, Ruoqi Shen, Zhao Song, and Jiajun Wu.
\newblock When is particle filtering efficient for planning in partially
  observed linear dynamical systems?
\newblock In {\em Uncertainty in artificial intelligence}, pages 728--737.
  PMLR, 2021.

\bibitem{gao2025activemeasuringreinforcementlearning}
Daiqi Gao, Ziping Xu, Aseel Rawashdeh, Predrag Klasnja, and Susan Murphy.
\newblock Active measuring in reinforcement learning with delayed negative
  effects.
\newblock In {\em The 29th International Conference on Artificial Intelligence
  and Statistics}, 2025.

\bibitem{gaur2024closinggapachievingglobal}
Mudit Gaur, Amrit~Singh Bedi, Di~Wang, and Vaneet Aggarwal.
\newblock Closing the gap: achieving global convergence (last iterate) of
  actor-critic under markovian sampling with neural network parametrization.
\newblock In {\em Proceedings of the 41st International Conference on Machine
  Learning}, pages 15153--15179, 2024.

\bibitem{golowich2022planningobservablepomdpsquasipolynomial}
Noah Golowich, Ankur Moitra, and Dhruv Rohatgi.
\newblock Planning in observable pomdps in quasipolynomial time.
\newblock arXiv preprint arXiv:2201.04735, 2022.

\bibitem{golowich2023planning}
Noah Golowich, Ankur Moitra, and Dhruv Rohatgi.
\newblock Planning and learning in partially observable systems via filter
  stability.
\newblock In {\em Proceedings of the 55th Annual ACM Symposium on Theory of
  Computing}, pages 349--362, 2023.

\bibitem{han2026neuralodesdemodels}
Chao Han, Stefanos Ioannou, Luca Manneschi, T.~J. Hayward, Michael Mangan,
  Aditya Gilra, and Eleni Vasilaki.
\newblock Neural ode and sde models for adaptation and planning in model-based
  reinforcement learning.
\newblock arXiv preprint arXiv:2603.23245, 2026.

\bibitem{hauskrecht2000planning}
Milos Hauskrecht and Hamish Fraser.
\newblock Planning treatment of ischemic heart disease with partially
  observable markov decision processes.
\newblock {\em Artificial Intelligence in Medicine}, 18(3):221--244, 2000.

\bibitem{ionides15}
Edward~L. Ionides, Dao Nguyen, Yves Atchad\'{e}, Stilian Stoev, and Aaron~A.
  King.
\newblock Inference for dynamic and latent variable models via iterated,
  perturbed {B}ayes maps.
\newblock {\em Proceedings of the National Academy of Sciences},
  112(3):719--724, 2015.

\bibitem{kaelbling1998planning}
Leslie~Pack Kaelbling, Michael~L Littman, and Anthony~R Cassandra.
\newblock Planning and acting in partially observable stochastic domains.
\newblock {\em Artificial intelligence}, 101(1-2):99--134, 1998.

\bibitem{kara2022convergencefinitememoryqlearning}
Ali~Devran Kara and Serdar Y{\"u}ksel.
\newblock Convergence of finite memory q learning for pomdps and near
  optimality of learned policies under filter stability.
\newblock {\em Mathematics of Operations Research}, 48(4):2066--2093, 2023.

\bibitem{karjalainen2026forgettingparticlefilters}
Joona Karjalainen, Anthony Lee, Sumeetpal~S. Singh, and Matti Vihola.
\newblock On the forgetting of particle filters.
\newblock arXiv preprint arXiv:2309.08517, 2026.

\bibitem{karkus2017qmdpnet}
Peter Karkus, David Hsu, and Wee~Sun Lee.
\newblock {QMDP-Net}: Deep learning for planning under partial observability.
\newblock In {\em Advances in Neural Information Processing Systems},
  volume~30, 2017.

\bibitem{kausik2023mixtures}
Chinmaya Kausik, Kevin Tan, and Ambuj Tewari.
\newblock Learning mixtures of {M}arkov chains and {MDP}s.
\newblock In Andreas Krause, Emma Brunskill, Kyunghyun Cho, Barbara Engelhardt,
  Sivan Sabato, and Jonathan Scarlett, editors, {\em Proceedings of the 40th
  International Conference on Machine Learning}, volume 202 of {\em Proceedings
  of Machine Learning Research}, pages 15970--16017. PMLR, 23--29 Jul 2023.

\bibitem{kausik2023learning}
Chinmaya Kausik, Kevin Tan, and Ambuj Tewari.
\newblock Learning mixtures of markov chains and mdps.
\newblock In {\em International Conference on Machine Learning}, pages
  15970--16017. PMLR, 2023.

\bibitem{lim2023optimalityparticlepomdp}
Michael~H Lim, Tyler~J Becker, Mykel~J Kochenderfer, Claire~J Tomlin, and
  Zachary~N Sunberg.
\newblock Optimality guarantees for particle belief approximation of pomdps.
\newblock {\em Journal of Artificial Intelligence Research}, 77:1591--1636,
  2023.

\bibitem{lim2021sparse}
Michael~H Lim, Claire~J Tomlin, and Zachary~N Sunberg.
\newblock Sparse tree search optimality guarantees in pomdps with continuous
  observation spaces.
\newblock In {\em Proceedings of the Twenty-Ninth International Conference on
  International Joint Conferences on Artificial Intelligence}, pages
  4135--4142, 2021.

\bibitem{littman1995learning}
Michael~L Littman, Anthony~R Cassandra, and Leslie~Pack Kaelbling.
\newblock Learning policies for partially observable environments: Scaling up.
\newblock In {\em Machine learning proceedings 1995}, pages 362--370. Elsevier,
  1995.

\bibitem{liu2022partially}
Qinghua Liu, Alan Chung, Csaba Szepesv{\'a}ri, and Chi Jin.
\newblock When is partially observable reinforcement learning not scary?
\newblock In {\em Conference on Learning Theory}, pages 5175--5220. PMLR, 2022.

\bibitem{liu2023optimistic}
Qinghua Liu, Praneeth Netrapalli, Csaba Szepesvari, and Chi Jin.
\newblock Optimistic {MLE}: A generic model-based algorithm for partially
  observable sequential decision making.
\newblock In {\em Proceedings of the 55th Annual ACM Symposium on Theory of
  Computing}, pages 363--376, 2023.

\bibitem{naesseth2018variationalsequentialmontecarlo}
Christian Naesseth, Scott Linderman, Rajesh Ranganath, and David Blei.
\newblock Variational sequential monte carlo.
\newblock In {\em International conference on artificial intelligence and
  statistics}, pages 968--977. PMLR, 2018.

\bibitem{ni2021recurrentbaseline}
Tianwei Ni, Benjamin Eysenbach, and Ruslan Salakhutdinov.
\newblock Recurrent model-free {RL} is a strong baseline for many pomdps.
\newblock {\em CoRR}, abs/2110.05038, 2021.

\bibitem{papadimitriou1987complexity}
Christos~H Papadimitriou and John~N Tsitsiklis.
\newblock The complexity of markov decision processes.
\newblock {\em Mathematics of operations research}, 12(3):441--450, 1987.

\bibitem{pineau2003point}
Joelle Pineau, Geoff Gordon, Sebastian Thrun, et~al.
\newblock Point-based value iteration: An anytime algorithm for {POMDPs}.
\newblock In {\em Ijcai}, volume~3, pages 1025--1032, 2003.

\bibitem{silver2010monte}
David Silver and Joel Veness.
\newblock Monte-carlo planning in large pomdps.
\newblock {\em Advances in neural information processing systems}, 23, 2010.

\bibitem{song2023hybrid}
Yuda Song, Yifei Zhou, Ayush Sekhari, J~Andrew Bagnell, Akshay Krishnamurthy,
  and Wen Sun.
\newblock Hybrid rl: Using both offline and online data can make rl efficient.
\newblock In {\em International Conference on Learning Representations}, 2023.

\bibitem{sunberg2018online}
Zachary Sunberg and Mykel Kochenderfer.
\newblock Online algorithms for pomdps with continuous state, action, and
  observation spaces.
\newblock In {\em Proceedings of the International Conference on Automated
  Planning and Scheduling}, volume~28, pages 259--263, 2018.

\bibitem{tan2024acceleratedinferencepartiallyobserved}
Kevin Tan, Giles Hooker, and Edward~L. Ionides.
\newblock Accelerated inference for partially observed markov processes using
  automatic differentiation.
\newblock arXiv preprint arXiv:2407.03085, 2024.

\bibitem{whiteley2013stability}
Nick Whiteley.
\newblock Stability properties of some particle filters.
\newblock {\em The Annals of Applied Probability}, pages 2500--2537, 2013.

\bibitem{wu2022finitetimeanalysistimescale}
Yue~Frank Wu, Weitong Zhang, Pan Xu, and Quanquan Gu.
\newblock A finite-time analysis of two time-scale actor-critic methods.
\newblock {\em Advances in Neural Information Processing Systems},
  33:17617--17628, 2020.

\bibitem{xu2021improvingsamplecomplexitybounds}
Tengyu Xu, Zhe Wang, and Yingbin Liang.
\newblock Improving sample complexity bounds for (natural) actor-critic
  algorithms.
\newblock {\em Advances in Neural Information Processing Systems},
  33:4358--4369, 2020.

\end{thebibliography}

\appendix

\section{Further Related Work}

Our work is related to research on approximate planning in POMDPs, particle-based belief representations, and inference in latent-state models.

\paragraph{Planning in POMDPs.}
POMDP planning can be formulated as dynamic programming over the belief state \cite{kaelbling1998planning}, but exact planning is difficult because the belief space is continuous and must be updated online. Practical methods therefore rely on approximations. Sampling-based planners such as POMCP combine Monte-Carlo tree search with particle-based belief updates and scale to large problems without explicit belief-stateenumeration \cite{silver2010monte,lim2021sparse}, while progressive-widening methods extend this approach to continuous spaces \cite{sunberg2018online}. Another common approximation is QMDP, which replaces belief-space continuation values with fully observed MDP values averaged under the current belief \cite{cassandra1994acting,littman1995learning}. Our method is connected to this line of work through combining finite-depth MCTS with a QMDP-style bootstrap, but differs in that we study this approximation in an online learning setting and explicitly quantify the resulting structural error.
POMCP, POWSS, and sparse PFT are online planners in that they replan from the current belief at every step. However, they are analyzed with the model given to the planner, so their analyses are not online learning results. The same holds for our planner when it is analyzed at a frozen $\theta$.
Their guarantees concern a single planning call at an exactly known $\theta$, with no estimation error and no horizon $T$, so they do not directly yield a regret bound when the planner is called with a re-estimated $\bar\theta_k$.
Because the planner enters our decomposition only through $\epsilon_{\mathrm{mc}}(t,d)$, one could substitute such a planner by re-deriving that term (and term (B1), which a particle-belief MDP absorbs) under $\bar\theta_k$.
Planning on the particle-belief MDP also requires $O(J)$ generative calls per node, compared with $O(1)$ for a state-level simulation.

\paragraph{Particle-based inference under partial observability.}
A related line of work uses particle filtering to approximate the belief state. Classical particle filtering theory provides convergence and stability guarantees for approximate Bayes filtering \cite{crisan2002survey,whiteley2013stability}. \cite{lim2023optimalityparticlepomdp} studied the particle-belief MDP directly and established optimality guarantees relative to the true belief-state MDP. More recent work has studied planning with particle methods in structured partially observed systems \cite{du2021particle} as well as joint state and parameter inference in state-space models \cite{chen2025self}. Our paper connects these directions in a different way: rather than planning over propagated particle beliefs, we use particles to represent the current belief, perform state-level lookahead from those particles, and combine this with online model learning. This leads to an end-to-end regret analysis that separates estimation error, particle approximation error, Monte-Carlo planning error, and QMDP-style approximation bias.

\paragraph{Model-free approaches.}
Model-free alternatives instead avoid explicit filtering by giving the learner a short horizon lookback or memory \citep{ni2021recurrentbaseline, kara2022convergencefinitememoryqlearning}. The latter are often effective in practice, especially with deep recurrent policies, but their hidden states are difficult to interpret as beliefs, and remain challenging to analyze. Existing sample complexity guarantees for the former often contain caveats, such as only quasipolynomial rather than polynomial sample complexity \citep{golowich2023planning, golowich2022planningobservablepomdpsquasipolynomial}, or non-standard assumptions \citep{anjarlekar2026scalable}. We will focus on the model-based approach.

\section{MCTQ Planning Subroutine}
\label{app:mctq}

We describe the full MCTQ (Monte Carlo Tree search with QMDP leaves) planning subroutine referenced in Algorithm~\ref{alg:pf_mcts}.
MCTQ estimates the depth-$d$ belief-tree action-value $Q_{\theta,d}(b^J,a)$ defined in~\eqref{eq:Qd_def} by Monte Carlo tree search on the particle belief.
It is inspired by the POMCP framework~\cite{silver2010monte}, with three departures: internal actions are drawn uniformly at random rather than by UCB1, every simulation descends to the full depth $d$, and values are computed by a bottom-up dynamic programming (DP) backup after all simulations complete.

\paragraph{Overview.}
Given the current particle belief $b^J=\frac{1}{J}\sum_{j=1}^J\delta_{s_j}$, model $\theta$, search depth $d$, simulation budget $N$ per root action, and a precomputed fully observable Q-function $Q_\theta^{\star,S}$ (e.g., from value iteration), MCTQ estimates $Q_{\theta,d}(b^J,a)$ for each root action $a\in\gA$ separately (round-robin), so that every root action receives exactly $N$ simulations.
For each root action $a$, each simulation samples a particle $s_j$ from $b^J$, takes the fixed root action $a$ to obtain a child state $s'\sim P_\theta(\cdot\mid s_j,a)$ and observation $o\sim Z_\theta(\cdot\mid s')$, and then descends the search tree $T_a$ from the child node $(a,o)$ down to depth $d$, selecting actions uniformly at random at each internal node.
Nodes are created the first time a simulation passes through them, and every node on the path records the simulated state.
After all $N$ simulations complete, a bottom-up DP backup computes value estimates: at each depth-$d$ leaf, the QMDP value of its accumulated particles; at each internal node, $\hat V(v)=\max_{a'}\hat Q(v,a')$, where $\hat Q(v,a')$ is the empirical mean of $r+\gamma\hat V(\text{child})$ over the simulations that selected $a'$ at $v$.
The root estimate $\widehat Q(b^J,a)$ is the same empirical mean at the root.

Like POMCP~\cite{silver2010monte}, MCTQ indexes its search trees by action-observation histories $h=(a_1,o_1,\ldots,a_k,o_k)$ rather than by explicit beliefs, and represents the belief at each node implicitly by the collection of states simulated through that node.
The round-robin over root actions is needed because the regret analysis (Lemma~\ref{lem:delta_bound_main}) requires uniformly good estimates $\widehat Q(b^J,a)$ for \emph{all} $a\in\gA$.
The analysis (Proposition~\ref{prop:mc_concentration}) uses two properties of the internal action-selection rule $\pi$, i.e., the rule that picks actions at nodes below the root within one planning call:
\begin{enumerate}
\item[(S1)] \emph{Non-adaptivity.} $\pi$ may depend on the node's history, on $\theta$, on the root particle set $b^J$, and on fresh randomness, but not on any statistic accumulated by the current tree.
\item[(S2)] \emph{Full support.} $\pi(a'\mid\cdot)\ge\pi_{\min}>0$ for every $a'\in\gA$.
\end{enumerate}
Uniform selection satisfies both with $\pi_{\min}=1/|\gA|$.
So does a learned proposal that is frozen during the planning call and mixed with the uniform distribution.
UCB1 violates (S1), because its choice at a node depends on the visit counts and value estimates accumulated by the current tree.
We therefore use uniform selection for the analysis and do not claim it is the best rule in practice.
Together with the absence of an early return (every simulation descends to depth $d$), (S1) makes the path of each simulation a function of its own random seed alone (Lemma~\ref{lem:path_decoupling}).
The simulation seeds are i.i.d.\ given $b^J$, but the simulation returns are not, because the leaf evaluator pools particles across simulations and the DP backup couples them.
The analysis therefore treats $\widehat Q(b^J,a)$ as a bounded-difference function of the i.i.d.\ seeds rather than as an average of independent returns.

\paragraph{Tree structure and belief particles.}
Each search tree $T_a$ (one per root action $a$) has the particle belief $b^J$ at its root (depth $0$), the root action fixed to $a$, internal nodes at depths $1,\ldots,d-1$, and leaves at depth $d$.
Each node $v$ is identified with the action-observation history that leads to it from the root, so a given history always refers to the same node.
A child of node $v$ under action $a'$ and observation $o$ is denoted $\textsc{Child}(v,a',o)$; it is created on demand when first visited.

Each node $v$ maintains a \emph{particle set} $B(v)$: the collection of all states that have been simulated through $v$ across all simulations.
The empirical distribution $\hat b_v:=\frac{1}{|B(v)|}\sum_{s\in B(v)}\delta_s$ is the particle approximation of the belief at $v$, and $N(v):=|B(v)|$ is the number of simulations that have visited $v$.
In addition, each internal node stores:
\begin{itemize}
\item $N(v,a')$: the number of simulations that selected action $a'$ at $v$;
\item $\bar r(v,a')$: the running mean of the immediate rewards $r(s,a')$ over those simulations.
\end{itemize}
These statistics are accumulated during the simulation phase and used in the post-simulation DP backup.

\paragraph{QMDP leaf evaluation.}
At depth-$d$ leaf nodes (remaining depth $\ell=0$), the QMDP value of the accumulated belief is
\begin{equation}\label{eq:qmdp_leaf_eval}
\bar V_\theta^{\star,S}(B(v)):=\max_{a'\in\gA}\frac{1}{|B(v)|}\sum_{s\in B(v)}Q_\theta^{\star,S}(s,a').
\end{equation}
This takes the max \emph{outside} the average over states, choosing a single action for all particles, which coincides with the QMDP approximation $\bar V_\theta^{\star,S}$ from~\eqref{eq:U0_def}. This is distinct from the per-state evaluation $\frac{1}{|B(v)|}\sum_s\max_{a'}Q_\theta^{\star,S}(s,a')$, which would compute the linear surrogate $\widetilde V_\theta^{\star}$ instead. Using~\eqref{eq:qmdp_leaf_eval} ensures the algorithm matches the theoretical target $\bar V_{\theta,d}^{\star,S}=\gT_\theta^d\bar V_\theta^{\star,S}$ analyzed in Section~\ref{sec:analysis}.

\paragraph{Simulation procedure.}
For each root action $a$, each of the $N$ simulations proceeds as follows.
A particle $s_j$ is sampled uniformly from $b^J$, and the root transition is simulated: $s'\sim P_\theta(\cdot\mid s_j,a)$, $o\sim Z_\theta(\cdot\mid s')$, with immediate reward $r(s_j,a)$.
Starting from the child node $\textsc{Child}(\mathrm{root},a,o)$ with remaining depth $\ell=d-1$, the \textsc{Simulate} subroutine (Algorithm~\ref{alg:simulate}) descends the tree recursively:
\begin{enumerate}
\item \emph{Record.} Add the current state $s$ to $B(v)$.
\item \emph{QMDP leaf ($\ell=0$).} Stop. The leaf is evaluated by~\eqref{eq:qmdp_leaf_eval} on its final particle set once all simulations are complete.
\item \emph{Expansion (new internal node).} If $v$ is not yet expanded, initialize $N(v,a')\gets0$ and $\bar r(v,a')\gets0$ for all $a'\in\gA$, and continue. There is no early return, so a simulation that creates a node continues to descend.
\item \emph{Uniform selection.} Select
\begin{equation}\label{eq:uniform_action}
a'\sim\mathrm{Unif}(\gA).
\end{equation}
\item \emph{Simulate and recurse.} Draw $s''\sim P_\theta(\cdot\mid s,a')$ and $o'\sim Z_\theta(\cdot\mid s'')$, update $N(v,a')$ and $\bar r(v,a')$ with the reward $r(s,a')$, and recurse into $\textsc{Child}(v,a',o')$ with remaining depth $\ell-1$.
\end{enumerate}
Every simulation therefore visits exactly one node at each depth $1,\ldots,d$.
After all $N$ simulations for root action $a$, the DP backup (Phase~2 of Algorithm~\ref{alg:mctq}) computes, from the leaves upward,
\begin{equation}\label{eq:dp_backup}
\begin{aligned}
\hat V(v)&=\bar V_\theta^{\star,S}(B(v)) &&\text{(leaf $v$)},\\
\hat Q(v,a')&=\bar r(v,a')+\gamma\sum_{w\in\mathrm{Ch}(v,a')}\frac{N(w)}{N(v,a')}\,\hat V(w) &&\text{(internal $v$, $N(v,a')>0$)},\\
\hat V(v)&=\max_{a':\,N(v,a')>0}\hat Q(v,a') &&\text{(internal $v$)},
\end{aligned}
\end{equation}
where $\mathrm{Ch}(v,a')$ is the set of children of $v$ under action $a'$.
Equivalently, $\hat Q(v,a')$ is the average of $r(s,a')+\gamma\hat V(\text{child})$ over the simulations that selected $a'$ at $v$.
The root estimate is
\[
\widehat Q(b^J,a)=\frac{1}{N}\sum_{n=1}^N\bigl[r(s_{j_n},a)+\gamma\,\hat V(v_{n,1})\bigr],
\]
where $s_{j_n}$ is the particle sampled by simulation $n$ and $v_{n,1}$ is the depth-$1$ node it enters.
When $d=0$ no tree is built and $\widehat Q(b^J,a)=\frac1J\sum_{j=1}^J Q_\theta^{\star,S}(s_j,a)$.

The full procedure is given in Algorithms~\ref{alg:mctq} and~\ref{alg:simulate}.

\begin{algorithm}[t]
\caption{\textsc{Mctq}: Monte Carlo tree QMDP planning (uniform exploration, full-depth simulations, DP backup)}
\label{alg:mctq}
\begin{algorithmic}[1]
\Require model $\theta$; particle belief $b^J=\frac{1}{J}\sum_{j=1}^J\delta_{s_j}$; depth $d$; simulations per action $N$; MDP Q-function $Q_\theta^{\star,S}$
\If{$d=0$} \Comment{No tree: exact QMDP}
    \State\Return $\bigl\{\frac1J\sum_{j=1}^J Q_\theta^{\star,S}(s_j,a)\bigr\}_{a\in\gA}$
\EndIf
\For{each action $a\in\gA$} \Comment{Round-robin: $N$ simulations per root action}
    \State Initialize empty search tree $T_a$
    \For{$n=1,\ldots,N$} \Comment{Phase 1: tree construction}
        \State Sample $j_n\sim\mathrm{Unif}([J])$
        \State Sample $s'\sim P_\theta(\cdot\mid s_{j_n},a)$, $o\sim Z_\theta(\cdot\mid s')$
        \State $v_{n,1}\gets\textsc{Child}(\mathrm{root}(T_a),a,o)$ \Comment{Child of root for $(a,o)$}
        \State $\textsc{Simulate}(s',\,d-1,\,\theta,\,v_{n,1})$ \Comment{Alg.~\ref{alg:simulate}}
    \EndFor
    \State \Comment{Phase 2: DP backup~\eqref{eq:dp_backup}}
    \State $\hat V(v)\gets\max_{a'\in\gA}\frac{1}{|B(v)|}\sum_{s\in B(v)}Q_\theta^{\star,S}(s,a')$ at each leaf $v$ of $T_a$ \Comment{Pooled QMDP leaf}
    \For{each internal node $v$ of $T_a$, from depth $d-1$ up to depth $1$}
        \State $\hat Q(v,a')\gets\bar r(v,a')+\gamma\sum_{w\in\mathrm{Ch}(v,a')}\frac{N(w)}{N(v,a')}\hat V(w)$ for each $a'$ with $N(v,a')>0$
        \State $\hat V(v)\gets\max_{a':\,N(v,a')>0}\hat Q(v,a')$
    \EndFor
    \State $\widehat Q(b^J,a)\gets\frac{1}{N}\sum_{n=1}^N\bigl[r(s_{j_n},a)+\gamma\,\hat V(v_{n,1})\bigr]$
\EndFor
\State\Return $\{\widehat Q(b^J,a)\}_{a\in\gA}$
\end{algorithmic}
\end{algorithm}

\begin{algorithm}[t]
\caption{\textsc{Simulate}: Recursive uniform-exploration descent to full depth}
\label{alg:simulate}
\begin{algorithmic}[1]
\Require state $s$; remaining depth $\ell$; model $\theta$; tree node $v$
\State $B(v)\gets B(v)\cup\{s\}$ \Comment{Record state; $N(v)=|B(v)|$}
\If{$\ell=0$} \Comment{Leaf: evaluated in Phase 2 of Alg.~\ref{alg:mctq}}
    \State\Return
\EndIf
\If{$v$ is not yet expanded} \Comment{Expansion: initialize and continue (no early return)}
    \State Expand $v$: $N(v,a')\gets 0$, $\bar r(v,a')\gets 0$ for all $a'\in\gA$
\EndIf
\State Sample $a'\sim\mathrm{Unif}(\gA)$ \Comment{Uniform, reads no tree statistic}
\State Sample $s'\sim P_\theta(\cdot\mid s,a')$, $o\sim Z_\theta(\cdot\mid s')$
\State $N(v,a')\gets N(v,a')+1$; $\bar r(v,a')\gets\bar r(v,a')+\frac{r(s,a')-\bar r(v,a')}{N(v,a')}$ \Comment{Incremental mean}
\State $\textsc{Simulate}(s',\,\ell-1,\,\theta,\,\textsc{Child}(v,a',o))$
\end{algorithmic}
\end{algorithm}

\paragraph{Convergence.}
As $N\to\infty$, the following two properties hold.
First, by Lemma~\ref{lem:path_decoupling} the events ``simulation $n$ reaches $v$ and selects $a'$'' are i.i.d.\ across $n$ with positive probability for every reachable node $v$ and every $a'$, so every such node and action is visited infinitely often.
Second, the states entering a node $v$ are i.i.d.\ draws from the (binned) conditional belief $b_v$ at $v$, so $\hat b_v\to b_v$ by the law of large numbers.
At each depth-$d$ leaf, $\bar V_\theta^{\star,S}(B(v))\to\bar V_\theta^{\star,S}(b_v)$, and the DP backup then gives, level by level, $\hat Q(v,a')\to Q_{\theta,\ell}^{\phi}(b_v,a')$ and $\hat V(v)\to \max_{a'}Q_{\theta,\ell}^{\phi}(b_v,a')=\bar V_{\theta,\ell}^{\star,S}(b_v)$, where $\ell$ is the remaining depth at $v$.
Since the root action is fixed and states are sampled from $b^J$, the backed-up estimate $\widehat Q(b^J,a)\to Q_{\theta,d}^{\phi}(b^J,a)$ for each $a\in\gA$.
Proposition~\ref{prop:mc_concentration} quantifies the rate.

When $d=0$, no tree search is performed and the planner returns $\max_a\frac{1}{J}\sum_j Q_\theta^{\star,S}(s_j,a)$, which is the QMDP value $\bar V_\theta^{\star,S}(b^J)$.
When $d\to\infty$, the tree-search value converges to $Q_\theta^{\star,B}(b^J,a)$.

\paragraph{Computational cost.}
For each root action, each of the $N$ simulations traverses a path of exactly $d$ nodes in the tree, requiring $O(d)$ generative model calls and $O(1)$ work for uniform action sampling at each internal node.
The post-simulation DP backup adds $O(|\gA|)$ work per node but is performed once, not per simulation.
The QMDP leaf evaluation~\eqref{eq:qmdp_leaf_eval} can be maintained incrementally by storing running sums $\sum_{s\in B(v)}Q_\theta^{\star,S}(s,a')$ for each action $a'$, reducing the per-simulation leaf cost to $O(|\gA|)$.
The total cost per planning step is $|\gA|Nd$ generative-model calls across all root actions. The additional $|\gA|$ factor in the backup involves only arithmetic operations and no model calls.

\paragraph{Relation to the error analysis.}
The Monte Carlo error $\epsilon_{\mathrm{mc}}(t,d)=\sup_a|\widehat Q_t(b_t^J,a)-Q_{\theta,d}(b_t^J,a)|$ arises from the finite simulation budget $N$ per root action.
The simulation seeds are i.i.d.\ given $b_t^J$, and under (S1) with no early return each simulation's path is a function of its own seed (Lemma~\ref{lem:path_decoupling}).
The estimate $\widehat Q_t(b_t^J,a)$ is then a bounded-difference function of the seeds, so the Efron--Stein and McDiarmid inequalities control its fluctuation at rate $O(\Lambda_d/\sqrt N)$, while the DP-backup bias vanishes at rate $O(\Phi_d/\sqrt{N})$ (Proposition~\ref{prop:mc_concentration}).
The simulation returns themselves are not independent, and the argument does not require them to be.
The proof is written for uniform selection. For a rule satisfying (S1)--(S2), the same argument applies with $|\gA|$ replaced by $1/\pi_{\min}$ wherever it bounds the fraction of a node's visits that select a given action.

\section{Figures and Additional Details for Experiments}
\label{app:experiments}

\begin{algorithm}[H]
    \caption{Practical particle filtering with state-level Monte-Carlo planning}\label{alg:pf_mcts_practical}
    \begin{algorithmic}[1]
    \Require number of state particles $J$; rollout length $H_{\mathrm{roll}}$; rollouts per action $N$; horizon $T$;
    initial parameter $\bar\theta_1$; model update oracle $\textsc{ModelUpdate}(\cdot)$.
    \State \textbf{Initialize:} model dataset $\gD_m$, policy dataset $\gD_p$. 
    \State Sample $s^{(1)}_1,\ldots,s^{(1)}_J \sim b_0$ and set $b_1^J = \frac{1}{J}\sum_{j=1}^J \delta_{s^{(1)}_j}$.
    \For{$t=1,2,\ldots,T$}
        \If{model update}
        \State $\gD_m = \gD_m \cup \gD_p$.
      \State Update $\bar\theta_t$ via updating the model on $\gD_m$ with gradient descent.
      \State $\gD_p = \emptyset$.
      \Else 
      \State $\bar\theta_t \gets \bar\theta_{t-1}$.
      \EndIf
      \For{each particle $j=1,\ldots,J$ and each action $a\in\gA$}
        \State $\widehat Q_t(s^{(t)}_j,a) \gets \textsc{StateMCTS}\big(s^{(t)}_j,a,\bar\theta_t,H_{\mathrm{roll}},N,\widehat Q^{\mathrm{MDP}}_{\bar\theta_t}, \widehat{\pi}_{\bar{\theta}_t}^{\mathrm{MDP}}\big)$, where actions are sampled by $\widehat{\pi}_{\bar{\theta}_t}^{\mathrm{MDP}}$ and not the greedy policy w.r.t. $\widehat Q_t$. Collect $\widehat{a}_j^{(t)}$, the best action sampled for particle $j$.
      \EndFor
      \State $\widehat Q_t(b_t^J,a) \gets \frac{1}{J}\sum_{j=1}^J \widehat Q_t(s^{(t)}_j,a)$ for all $a\in\gA$.
      \State Play $a_t \in \argmax_{j\in [J]} \widehat Q_t(b_t^J, \widehat{a}_j^{(t)})$. Alternatively, sample $\tilde{s}^{(t)} \sim b_t^J$ and play $a_t \sim  \widehat{\pi}_{\bar{\theta}_t}^{\mathrm{MDP}}(\cdot | \tilde{s}^{(t)})$.
      \State Observe $(o_{t+1}, r_t)$ and append to $\gD_p$.
      \For{$j=1,\ldots,J$}
        \State Sample prediction $\hat s^{(t+1)}_j \sim P_{\bar\theta_t}(\cdot\mid s^{(t)}_j,a_t)$.
        \State Set weight $w^{(t+1)}_j \propto Z_{\bar\theta_t}(o_{t+1}\mid \hat s^{(t+1)}_j)$.
      \EndFor
      \State Resample $s^{(t+1)}_1,\ldots,s^{(t+1)}_J \sim \mathrm{Categorical}(w^{(t+1)}_1,\ldots,w^{(t+1)}_J)$.
      \State Set $b_{t+1}^J = \frac{1}{J}\sum_{j=1}^J \delta_{s^{(t+1)}_j}$.
      \State Update MDP value estimate $\widehat Q^{\mathrm{MDP}}_{\bar\theta_t}(s,a)$ and policy $\widehat{\pi}_{\bar{\theta}_t}^{\mathrm{MDP}}$ for $(\gS,\gA,P_{\bar\theta_t},r,\gamma)$ on $\gD_p$.
    \EndFor
    \end{algorithmic}
    \end{algorithm}
    
\subsection{Double Integrator LQR}
\label{app:lqr}

We use a partially observed double-integrator LQR system. The latent state is
\[
x_t =
\begin{bmatrix}
p_t \\
v_t
\end{bmatrix},
\]
where $p_t$ is position and $v_t$ is velocity. The scalar control
$u_t \in [-3,3]$ is acceleration. With time step $\Delta t = 0.05$, the
dynamics are
\[
x_{t+1}
=
A x_t + B u_t + w_t,
\qquad
A =
\begin{bmatrix}
1 & \Delta t \\
0 & 1
\end{bmatrix},
\qquad
B =
\begin{bmatrix}
\frac{1}{2}\Delta t^2 \\
\Delta t
\end{bmatrix},
\]
with process noise
\[
w_t \sim \mathcal{N}(0, \sigma_w^2 I_2),
\qquad
\sigma_w = 0.01.
\]

The agent observes only the position coordinate:
\[
o_t = C x_t + \epsilon_t,
\qquad
C =
\begin{bmatrix}
1 & 0
\end{bmatrix},
\qquad
\epsilon_t \sim \mathcal{N}(0,\sigma_o^2),
\qquad
\sigma_o = 0.05.
\]
Thus the velocity $v_t$ is latent. It is not recoverable from the current
observation alone because the observation map is not injective:
\[
\operatorname{rank}(C)=1 < 2.
\]
Equivalently, for any fixed position $p$ and any two distinct velocities
$v \neq v'$,
\[
C
\begin{bmatrix}
p \\
v
\end{bmatrix}
=
p
=
C
\begin{bmatrix}
p \\
v'
\end{bmatrix}.
\]
Therefore, even in the noiseless observation case, a single current
observation cannot distinguish states that differ only in velocity. With
observation noise, the ambiguity is strictly larger. Recovering information
about $v_t$ requires temporal history, e.g. past observations and actions,
rather than $o_t$ alone.

The reward is the negative quadratic LQR cost,
\[
r_t
=
-
\left(
x_t^\top Q x_t
+
u_t^\top R u_t
\right),
\qquad
Q = I_2,
\qquad
R = 0.1 I_1.
\]
The discounted objective is
\[
J(\pi)
=
\mathbb{E}_{\pi}
\left[
\sum_{t=0}^{T-1}
\gamma^t r_t
\right],
\qquad
\gamma = 0.9,
\qquad
T = 200.
\]

In our experiments, the initial state is sampled as
\[
x_0 \sim \mathcal{N}(0, I_2),
\]
episodes terminate after $T=200$ steps or when
\[
\|x_t\|_2 > 20.
\]
We use $32$ particles for belief tracking, rollout length $1$, planner
discount $0.9$, and a maximum of $16$ candidate actions.

\begin{figure}[H]
    \centering
    \includegraphics[width=\linewidth]{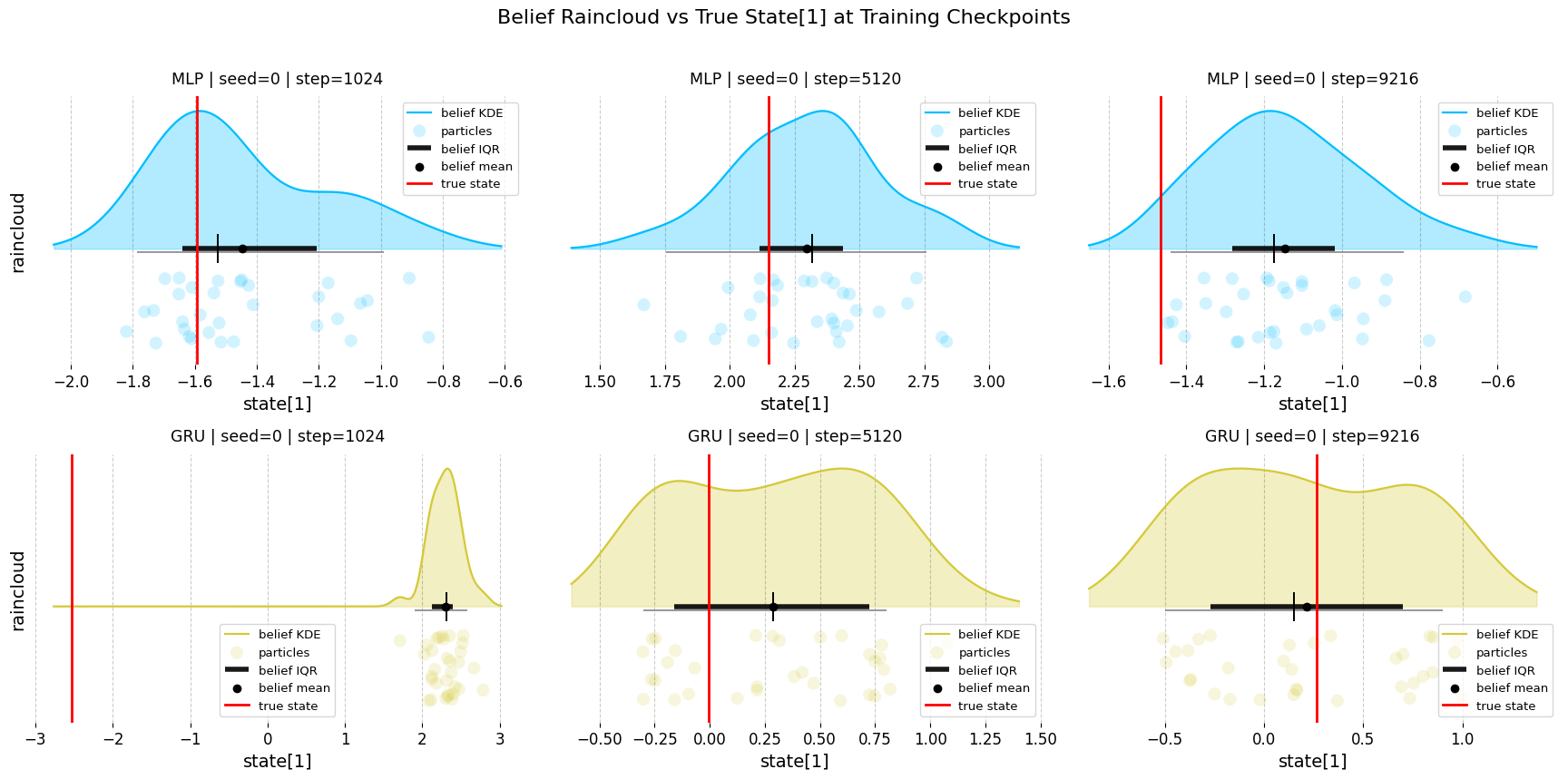}
    \caption{Snapshot of belief state at select training checkpoints. State recovery improves over time.}
    \label{fig:lqr-snapshot}
\end{figure}

\newpage

\begin{figure}[H]
    \centering
    \includegraphics[width=\linewidth]{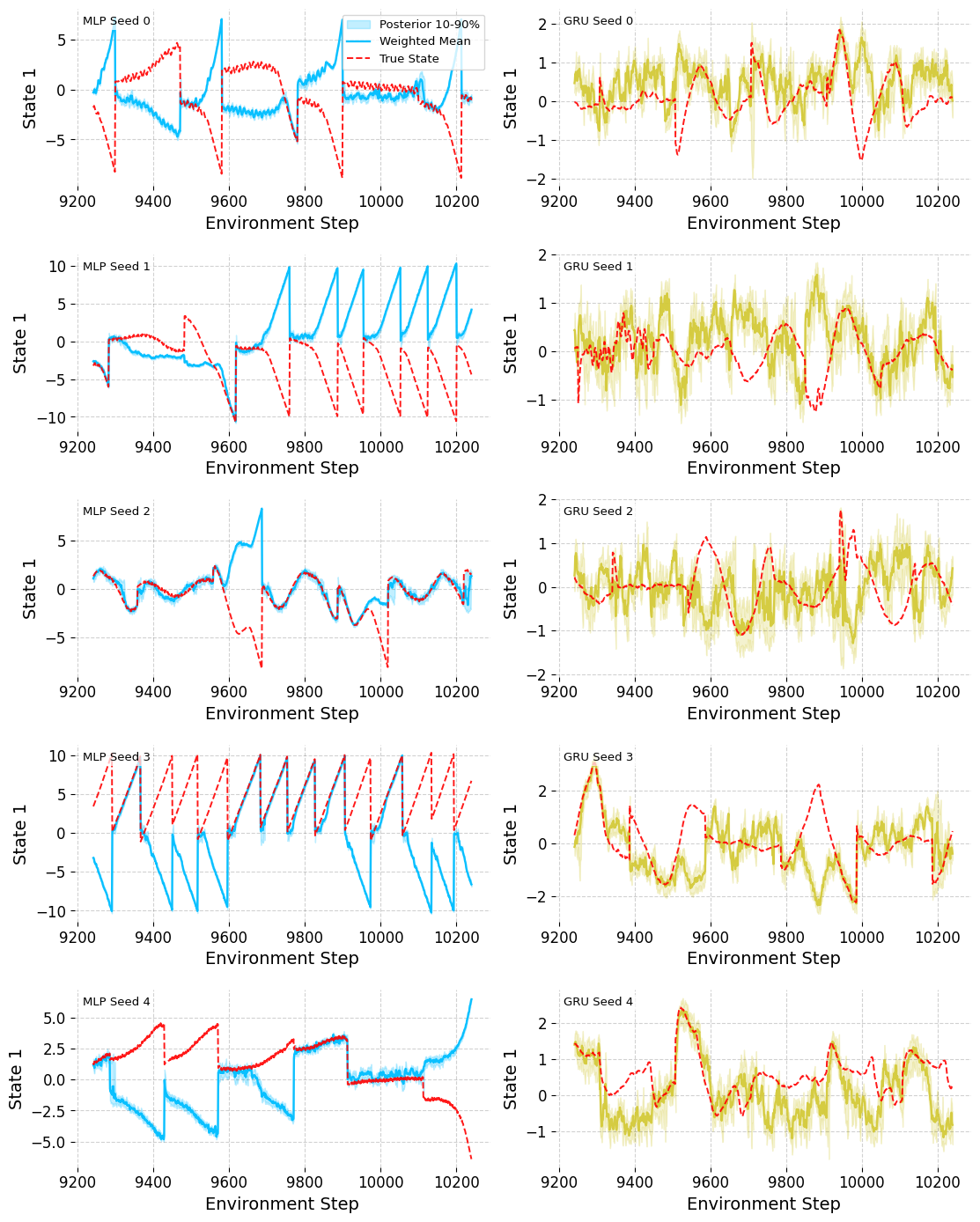}
    \caption{Posterior panel over five seeds for the double integrator LQR, over the last 1000 steps. The 10th quantile and 90th quantile are visualized. In this setting, only the magnitude, but not the direction, latent state is identifiable from the current observations alone. With no access to memory, the particle belief state still exhibits the ability to tracking the belief state, as seen in seeds 2 and 4, but fails when the state rapidly diverges. This suggests that the particle belief can be used in tandem with memory for latent state recovery.. }
    \label{fig:lqr-posterior}
\end{figure}

\newpage 

\begin{figure}[H]
    \centering
    \includegraphics[width=\linewidth]{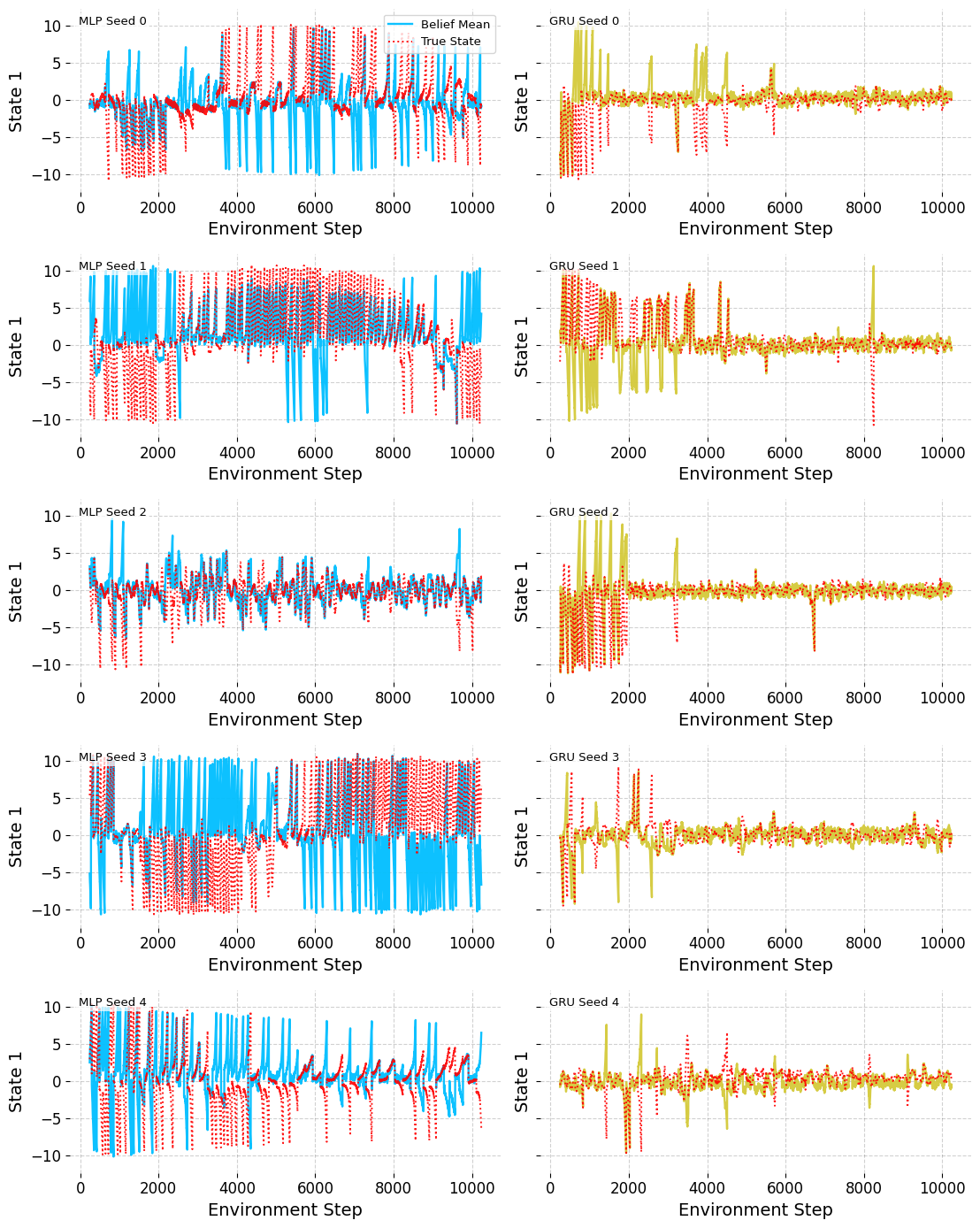}
    \caption{Belief mean panel over five seeds for LQR. The MLP backbone diverges, while the GRU backbone converges.}
    \label{fig:lqr-belief-mean}
\end{figure}

\newpage

\subsection{Heartsteps}
\label{app:heartsteps}

\paragraph{HeartSteps as an AOMDP.}
We follow the formulation of \cite{gao2025activemeasuringreinforcementlearning} in formulating HeartSteps as an actively observable Markov decision process
(AOMDP). At each day $t$, the agent first chooses a measurement action
$I_{t,1} \in \{0,1\}$, where $I_{t,1}=1$ corresponds to sending a survey
that reveals the latent reward/commitment state. The agent then chooses a
control action $A_{t,2} \in \{0,1\}$, where $A_{t,2}=1$ corresponds to
sending a walking-suggestion intervention.

Let $R_t$ denote the user's latent commitment to being active on day $t$,
$E_t$ denote engagement or burden, $C_{t,2}$ denote context before the
intervention decision, $M_{t,2}$ denote the proximal physical activity
outcome, and $O_t$ denote a passive proxy emission of the latent reward.
Following the AOMDP indexing, the latent state is
\[
U_{t+1} := R_t,
\]
and the observed state after the control action is
\[
Z_{t+1} := (M_{t,2}, E_t).
\]
The reward is the latent reward itself,
\[
r(Z_{t+1}, U_{t+1}) = R_t.
\]

The measurement-stage and control-stage states are
\[
S^I_{t,1} = (E_{t-1}, R_{t-1}),
\qquad
S^A_{t,2} = (E_{t-1}, R_{t-1}, C_{t,2}, I_{t,1}).
\]
Before choosing $I_{t,1}$, the agent observes $E_{t-1}$ and the passive
emission $O_{t-1}$, but not $R_{t-1}$. If $I_{t,1}=1$, the agent observes
$R_{t-1}$ before choosing $A_{t,2}$; if $I_{t,1}=0$, $R_{t-1}$ remains
latent.

The HeartSteps transition model is
\[
C_{t,2}
=
\theta^C_0 + \varepsilon^C_{t,2},
\]
\[
M_{t,2}
=
\theta^M_0
+
\theta^M_1 E_{t-1}
+
\theta^M_2 R_{t-1}
+
\theta^M_3 C_{t,2}
+
A_{t,2}
\left(
\theta^M_4
+
\theta^M_5 E_{t-1}
+
\theta^M_6 R_{t-1}
+
\theta^M_7 C_{t,2}
\right)
+
\varepsilon^M_{t,2},
\]
\[
E_t
=
\theta^E_0
+
\theta^E_1 E_{t-1}
+
\theta^E_A A_{t,2}
+
\theta^E_{AE} A_{t,2} E_{t-1}
+
\theta^E_I I_{t,1}
+
\theta^E_{IE} I_{t,1} E_{t-1}
+
\varepsilon^E_t,
\]
\[
R_t
=
\theta^R_0
+
\theta^R_1 M_{t,2}
+
\theta^R_2 E_t
+
\theta^R_3 R_{t-1}
+
\varepsilon^R_t,
\]
and the passive emission model is
\[
O_t
=
\theta^O_0
+
\theta^O_1 R_t
+
\varepsilon^O_t.
\]
The noise terms are mean-zero, with variances
$\sigma_C^2,\sigma_M^2,\sigma_E^2,\sigma_R^2,\sigma_O^2$.

The latent reward is not recoverable from the current passive observation
alone. Indeed, when $I_{t,1}=0$, the posterior over $R_{t-1}$ satisfies
\[
p(R_{t-1}=r \mid E_{t-1}, O_{t-1}, H_{t,1})
\propto
p(O_{t-1} \mid R_{t-1}=r)
p(R_{t-1}=r \mid E_{t-1}, H_{t,1}),
\]
where
\[
O_{t-1}
=
\theta^O_0
+
\theta^O_1 R_{t-1}
+
\varepsilon^O_{t-1}.
\]
Thus, unless the emission noise is degenerate and the prior is already a
point mass, the posterior is not a Dirac measure. In particular, for
$\sigma_O^2>0$, multiple values of $R_{t-1}$ have positive likelihood for
the same observed $O_{t-1}$. Therefore the latent reward cannot be
identified from the current passive observation alone. By contrast, when
$I_{t,1}=1$, the survey reveals $R_{t-1}$ and collapses the posterior:
\[
p(R_{t-1} \mid H_{t,2})
=
\delta_{R_{t-1}}.
\]

The AOMDP objective is
\[
\pi^\star
\in
\arg\max_\pi
\mathbb{E}_\pi
\left[
\sum_{t=1}^{H}
\gamma^{t-1} R_t
\right].
\]

\paragraph{Implementation settings.}
In our HeartSteps experiments, the encoded simulator state is
\[
x_t =
(R_t, E_t, t/H, C_t, M_t, O_t) \in \mathbb{R}^6.
\]
The observed coordinates are
\[
(E_t, t/H, C_t, M_t, O_t),
\]
and the only latent coordinate is $R_t$. We use horizon
\[
H = 270,
\qquad
\gamma = 0.9,
\]
binary control actions $A_{t,2}\in\{0,1\}$, and binary measurement actions
$I_{t,1}\in\{0,1\}$ in the AOMDP formulation. The benchmark contains
$42$ user-specific environments.

For PF-MCTQ, our default settings use
\[
J = 16
\]
particles, rollout length $1$, planner discount $0.9$, and at most $16$
candidate actions. The model and policy are trained online over
$27$ epochs with $10$ environment steps per epoch, for a total of
$270$ environment steps.

\paragraph{Benchmarks.}
RLSVI-SMC is the never-measure benchmark in our experiments.  It fixes the
measurement action to zero,
\[
    I_{t,1} = 0 \qquad \text{for all } t,
\]
so the latent reward is not directly observed after initialization.  Instead,
the algorithm maintains an SMC approximation to the posterior belief over the
latent reward,
\[
    \widehat b_t^R = \sum_{j=1}^J w_t^{(j)} \delta_{\widehat R_t^{(j)}} ,
\]
and chooses the control action using randomized least-squares value iteration
with a linear value-function approximation,
\[
    Q_A(b_{t,2}, a)
    \approx \phi_A(b_{t,2}, a)^\top \beta_A .
\]
Thus RLSVI-SMC is a strong and intentionally tailored benchmark for this
linear-function-approximation setting: the belief update uses the specified
linear-Gaussian working model, and the control rule uses hand-specified linear
features of the belief state, context, action, and action interactions.

PF-MCTQ is also evaluated under the never-measure information pattern,
\[
    I_{t,1}=0,
\]
so it has the same lack of direct latent-reward measurements as RLSVI-SMC.
However, PF-MCTQ is less specialized to the linear benchmark structure: it must
learn a dynamics model, actor, and critic, parameterized by neural networks,
from the observed data.  Consequently, PF-MCTQ incurs an additional estimation
cost relative to RLSVI-SMC, which is already matched to the linear basis and
Bayesian regression structure of the benchmark.

Dyna-ATMQ differs from both no-measure methods.  It is an active-measurement
baseline that maintains a tabular belief over a discretized state space, learns
a Dirichlet transition model and Q-table, and separately decides whether to
measure by comparing the estimated value of measurement to a measurement cost.
Its control action is selected from the belief-weighted Q-values,
\[
    A_{t,2}
    \in \arg\max_{a}
    \sum_s b_t(s) Q(s,a),
\]
while its measurement decision can reveal the latent reward and collapse the
posterior uncertainty.  In our experiments Dyna-ATMQ measures the environment
around half the time on average, giving it an informational advantage over both
RLSVI-SMC/never-measure and PF-MCTQ, which do not measure the environment.

\begin{figure}[H]
    \centering
    \includegraphics[width=\linewidth]{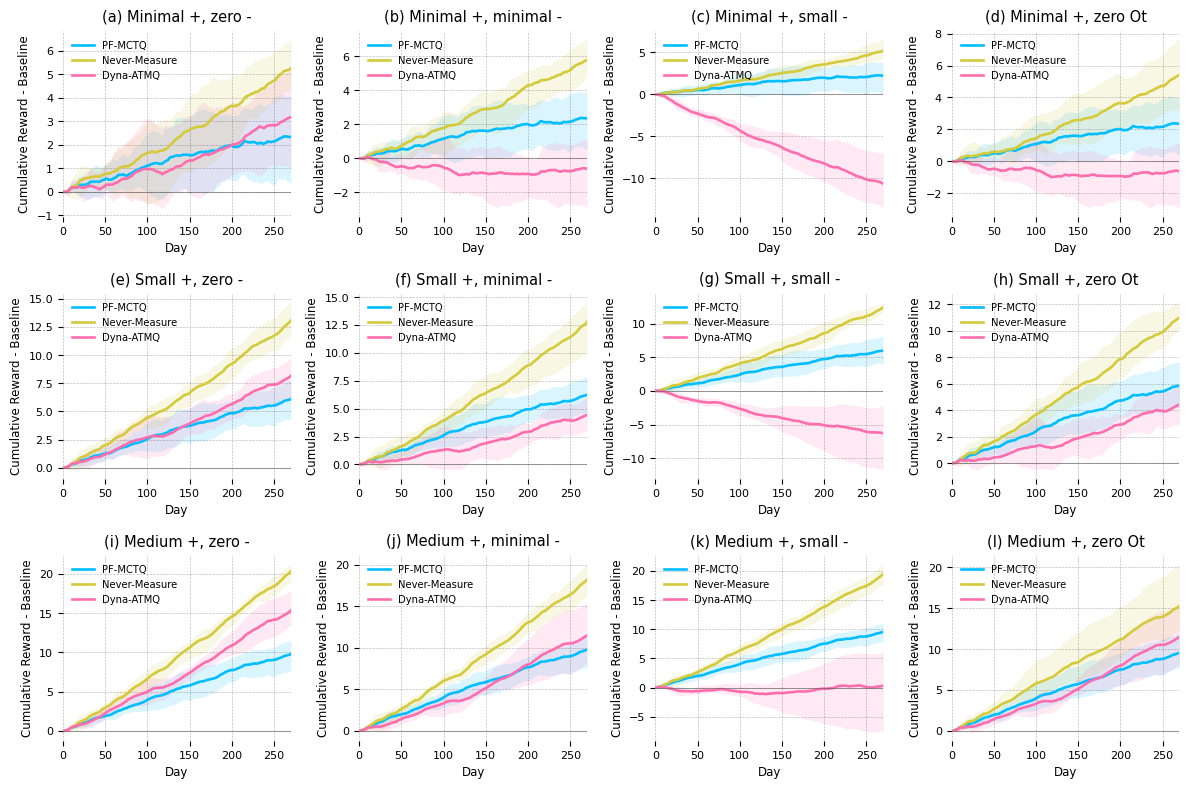}
    \caption{Cumulative reward above the baseline rollout across the twelve HeartSteps scenarios.
Panels vary the positive intervention effect (rows: minimal, small, medium) and the delayed negative measurement effect or proxy informativeness (columns: zero negative effect, minimal negative effect, small negative effect, and uninformative passive proxy $O_t$). Solid lines show the mean across runs and shaded bands show 2.5\% and 97.5\% quantiles.
The never-measure RLSVI-SMC benchmark performs best in most scenarios because it is specifically matched to the linear-Gaussian HeartSteps structure and is given access to hand-specified linear features. PF-MCTQ also never measures, but must learn its model, actor, and critic from data using neural approximators, so it typically improves over baseline while trailing the tailored linear benchmark. Dyna-ATMQ often underperforms when measurement has delayed negative effects, especially when the positive intervention effect is weak, because it measures frequently and therefore gains latent-state information at the cost of worsening future engagement. Its relative performance improves when the positive effect is larger or when passive proxy information is unavailable and so direct measurements become more valuable.
}
    \label{fig:heartsteps-all}
\end{figure}

\begin{figure}[H]
    \centering
    \includegraphics[width=\linewidth]{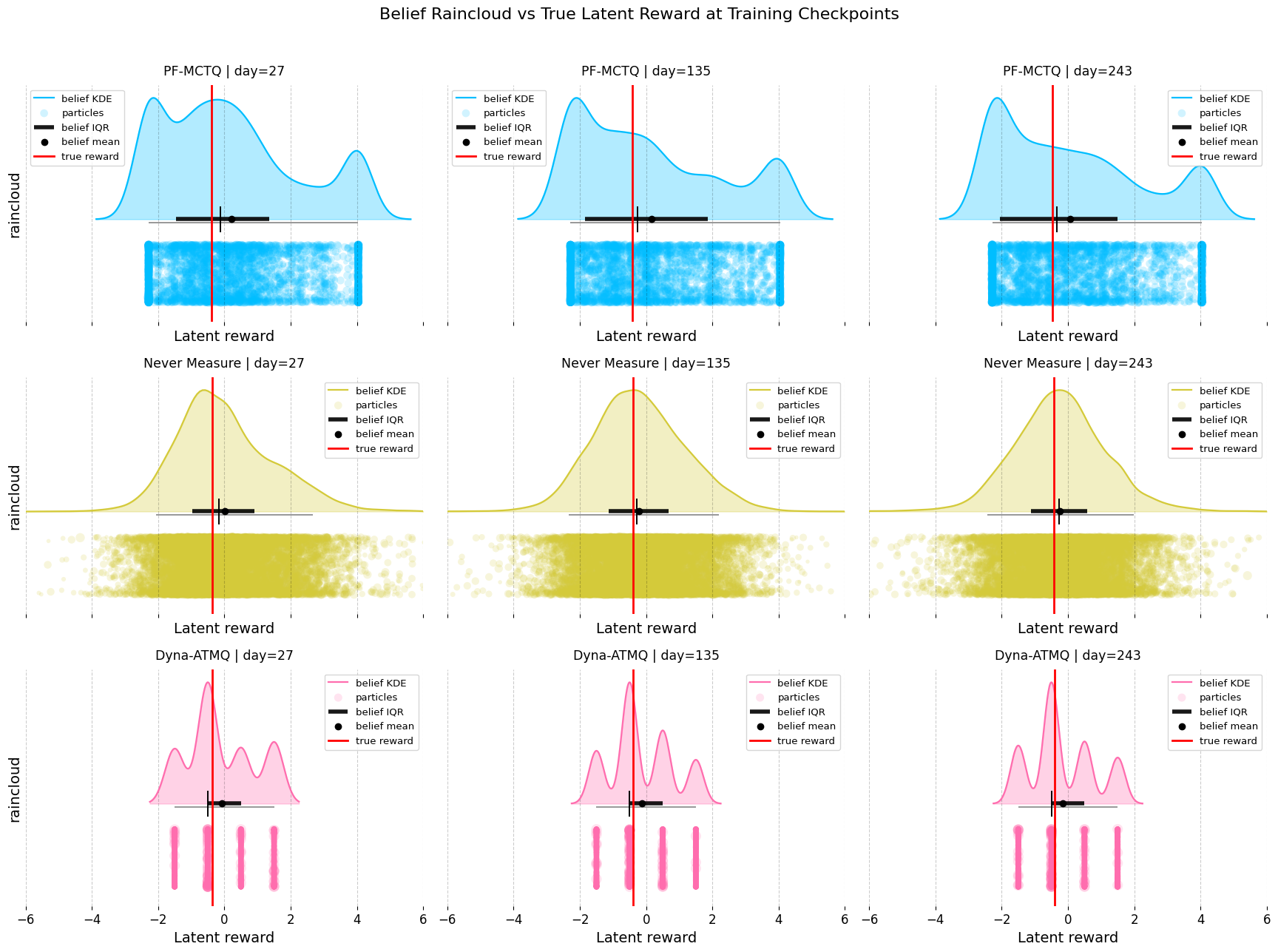}
    \caption{Belief over the latent reward at three training checkpoints in HeartSteps. Rows show PF-MCTQ, the never-measure RLSVI-SMC benchmark, and Dyna-ATMQ; columns show days 27, 135, and 243. Each panel displays posterior particles, a kernel density estimate, the posterior interquartile range, and the posterior mean. The red vertical line marks the true latent reward. We see that Dyna-ATMQ having access to measurements of the environment gives it an informational advantage over the other two algorithms. On the other hand, the correct linear specification that RLSVI-SMC enjoys allows it to create a posterior that is close to Gaussian. PF-MCTQ, which has to estimate a dynamics model, actor, and critic, parameterized by neural networks,
from the observed data, consequently has a much more expressive and multimodal posterior.
}
    \label{fig:heartsteps-snapshot}
\end{figure}

\begin{figure}[H]
    \centering
    \includegraphics[width=\linewidth]{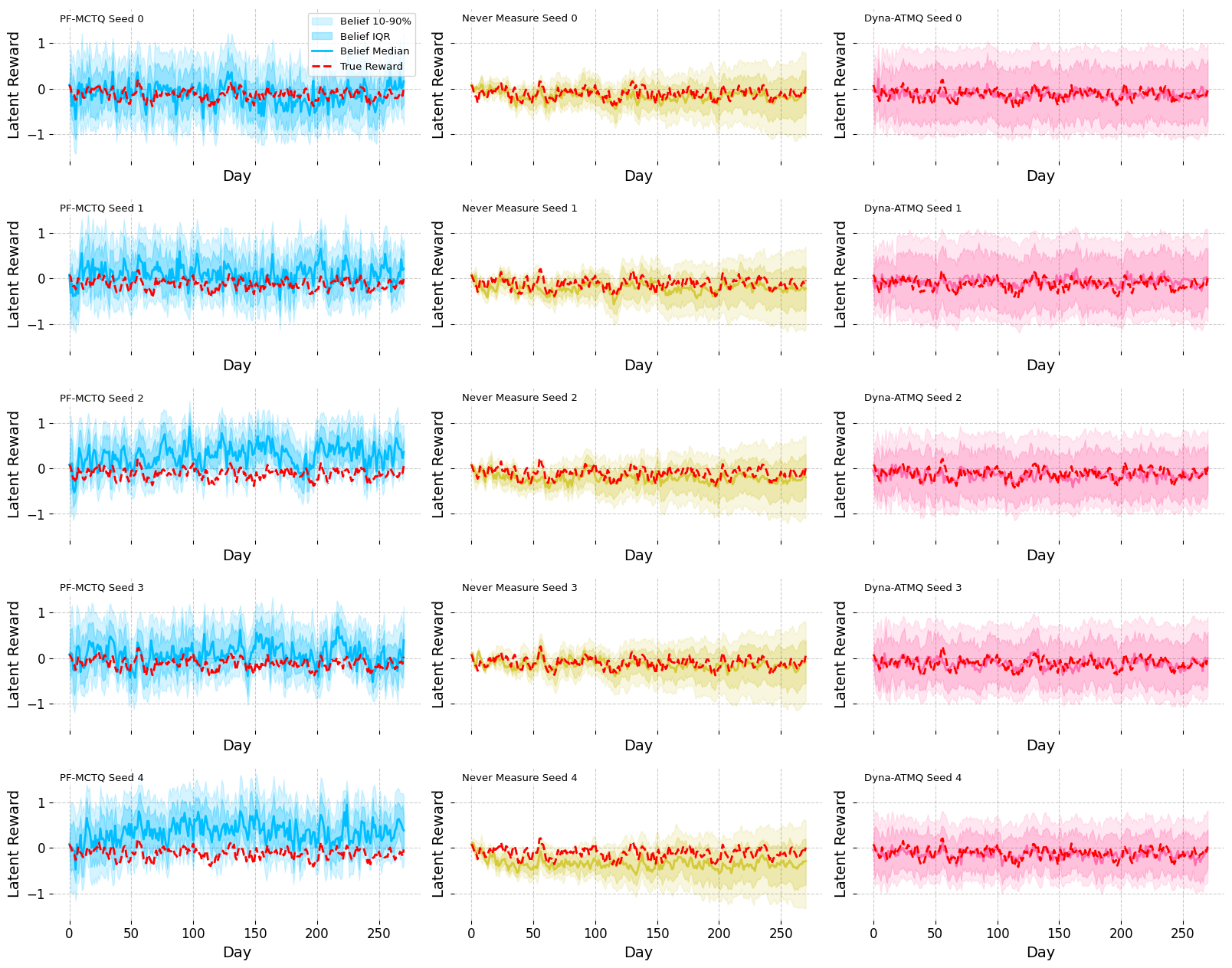}
    \caption{
Latent-reward belief trajectories for HeartSteps scenario (a), shown separately by random seed.
Columns compare PF-MCTQ, the never-measure RLSVI-SMC benchmark, and Dyna-ATMQ; rows show five independent seeds.
Each panel plots the posterior median, interquartile range, and 10--90\% interval for the latent reward, with the dashed red curve showing the true latent reward.
All three methods track the latent reward at the level needed for control, but they express uncertainty differently.
Never-measure produces comparatively narrow posteriors because it is tailored to the linear-Gaussian HeartSteps model and uses a conjugate SMC/RLSVI update.
PF-MCTQ has wider and more variable posteriors because it must learn the transition model and value functions from data while never directly measuring the latent reward.
Dyna-ATMQ also shows broad posterior support, but its belief remains centered near the true trajectory because it actively measures the latent reward on many days, partially collapsing uncertainty despite the coarser tabular approximation.
}
    \label{fig:heartsteps-posterior}
\end{figure}

\subsection{Hyperparameters}

\begin{table}[H]
\centering
\begin{tabular}{ll}
\toprule
Hyperparameter & Value \\
\midrule
Environment & LQR double integrator \\
Observed indices & $[0]$ \\
Warmup steps & $256$ \\
Initial model / policy updates & $100$ / $100$ \\
Epochs / steps per epoch & $40$ / $256$ \\
Model updates per epoch & $10$ \\
Particles & $32$ \\
Model subsequence length & $64$ \\
Network hidden dimensions & $[]$ \\
Recurrent lookback & $5$ \\
Rollout length / discount & $1$ / $0.9$ \\
SAC actor / critic learning rate & $3\times10^{-4}$ / $3\times10^{-4}$ \\
SAC batch size / entropy coefficient & $32$ / $0.01$ \\
Model learning rate / batch size & $3\times10^{-5}$ / $64$ \\
Particle filter mode / $\alpha$ & MOP / $0.8$ \\
ESS threshold / rejuvenation std. & $0.5$ / $0.1$ \\
Environment noise: transition / observation & $0.01$ / $0.05$ \\
Action bound / horizon & $3.0$ / $200$ \\
\bottomrule
\end{tabular}
\vspace{1ex}
\caption{Key hyperparameters for LQR.}
\end{table}

\begin{table}[H]
\centering
\begin{tabular}{ll}
\toprule
Hyperparameter & Value \\
\midrule
Environment & HeartSteps \\
User id / horizon & $0$ / $270$ \\
Noise mode / observation noise & sequential / $0.0$ \\
Warmup steps & $1024$ \\
Initial model / policy updates & $10$ / $30$ \\
Epochs / steps per epoch & $27$ / $10$ \\
Model updates per epoch & $10$ \\
Particles & $16$ \\
GRU backbone / recurrent BPTT & false / false \\
Hidden size / preprocessor / postprocessor & $32$ / $32$ / $32$ \\
Rollout length / discount & $1$ / $0.9$ \\
SAC actor / critic learning rate & $10^{-3}$ / $10^{-3}$ \\
SAC batch size / entropy coefficient & $64$ / $0.003$ \\
Automatic entropy tuning & false \\
Model learning rate / batch size & $3\times10^{-4}$ / $256$ \\
Particle filter mode / $\alpha$ & MOP / $0.8$ \\
ESS threshold / rejuvenation std. & $0.5$ / $0.001$ \\
\bottomrule
\end{tabular}
\vspace{1ex}
\caption{Key hyperparameters for HeartSteps.}
\end{table}

\begin{table}[H]
\centering
\begin{tabular}{ll}
\toprule
Hyperparameter & Value \\
\midrule
Environment & Ant \\
Observation variants & no position, no velocity \\
Observation noise / discount & $0.1$ / $0.99$ \\
Warmup steps & $1000$ \\
Initial model / policy updates & $2000$ / $10$ \\
Epochs / steps per epoch & $164$ / $300$ (submitted: $10$ / $300$) \\
Model updates per epoch & up to $10$, early stopping (submitted: $0$) \\
Policy update every & $10$ env.\ steps (submitted: $1$) \\
AdamW weight decay (actor, critic, model) & $10^{-4}$ (submitted: $0.1$) \\
Particles & $16$ \\
Model subsequence length & $32$ \\
GRU backbone / BPTT size & true / $32$ \\
Hidden dimensions & $[256,256]$ \\
Rollout length / discount & $10$ / $0.99$ \\
SAC actor / critic learning rate & $3\times10^{-4}$ / $3\times10^{-4}$ \\
SAC batch size / entropy coefficient & $256$ / $0.2$ \\
Model learning rate / batch size & $3\times10^{-4}$ / $256$ \\
Particle filter mode / $\alpha$ & MOP / $0.8$ \\
ESS threshold / rejuvenation std. & $0.5$ / $0.1$ \\
\bottomrule
\end{tabular}
\vspace{1ex}
\caption{Key hyperparameters for MuJoCo Ant, using the corrected configuration of Figure~\ref{fig:mujoco}. Values from the submitted version are given in parentheses where they differ.}
\label{tab:hyper_ant}
\end{table}

\begin{table}[t]
\centering
\begin{tabular}{ll}
\toprule
Hyperparameter & Value \\
\midrule
Environment & HalfCheetah \\
Observation variants & no position, no velocity \\
Observation noise / discount & $0.1$ / $0.99$ \\
Warmup steps & $1000$ \\
Initial model / policy updates & $1000$ / $10$ \\
Epochs / steps per epoch & $50$ / $1000$ \\
Model updates per epoch & $100$ \\
Early-stop patience / min updates & $20/20$ no-position; $5/5$ no-velocity \\
Particles & $16$ \\
Model subsequence length & $32$ \\
GRU backbone / BPTT size & true / $32$ \\
Hidden dimensions & $[256,256]$ \\
Rollout length / discount & $10$ / $0.99$ \\
SAC actor / critic learning rate & $3\times10^{-4}$ / $3\times10^{-4}$ \\
SAC batch size / entropy coefficient & $256$ / $0.2$ \\
Model learning rate / batch size & $3\times10^{-4}$ / $256$ \\
Particle filter mode / $\alpha$ & MOP / $0.8$ \\
\bottomrule
\end{tabular}
\vspace{1ex}
\caption{Key hyperparameters for MuJoCo HalfCheetah.}
\end{table}

\begin{table}[H]
\centering
\begin{tabular}{ll}
\toprule
Hyperparameter & Value \\
\midrule
Environment & Hopper \\
Observation variants & no position, no velocity \\
Observation noise / discount & $0.1$ / $0.99$ \\
Warmup steps & $1000$ \\
Initial model / policy updates & $1000$ / $10$ \\
Epochs / steps per epoch & $50$ / $1000$ \\
Model updates per epoch & $100$ \\
Particles & $16$ \\
Model subsequence length & $32$ \\
GRU backbone / BPTT size & true / $32$ \\
Hidden dimensions & $[256,256]$ \\
Rollout length / discount & $10$ / $0.99$ \\
SAC actor / critic learning rate & $3\times10^{-4}$ / $3\times10^{-4}$ \\
SAC batch size / entropy coefficient & $256$ / $0.2$ \\
Model learning rate / batch size & $3\times10^{-4}$ / $256$ \\
Particle filter mode / $\alpha$ & MOP / $0.8$ \\
ESS threshold / rejuvenation std. & $0.5$ / $0.1$ \\
\bottomrule
\end{tabular}

\caption{Key hyperparameters for MuJoCo Hopper.}
\end{table}

\begin{table}[H]
\centering
\begin{tabular}{ll}
\toprule
Hyperparameter & Value \\
\midrule
Environment & Walker2d \\
Observation variants & no position, no velocity \\
Observation noise / discount & $0.1$ / $0.99$ \\
Warmup steps & $1000$ \\
Initial model / policy updates & $1000$ / $10$ \\
Epochs / steps per epoch & $50$ / $1000$ \\
Model updates per epoch & $100$ \\
Early-stop patience / min updates & $20/20$ no-position; $5/5$ no-velocity \\
Particles & $16$ \\
Model subsequence length & $32$ \\
GRU backbone / BPTT size & true / $32$ \\
Hidden dimensions & $[256,256]$ \\
Rollout length / discount & $10$ / $0.99$ \\
SAC actor / critic learning rate & $3\times10^{-4}$ / $3\times10^{-4}$ \\
SAC batch size / entropy coefficient & $256$ / $0.2$ \\
Model learning rate / batch size & $3\times10^{-4}$ / $256$ \\
Particle filter mode / $\alpha$ & MOP / $0.8$ \\
\bottomrule
\end{tabular}
\vspace{1ex}
\caption{Key hyperparameters for MuJoCo Walker2d.}
\end{table}

\section{Detailed Error Analysis and Proofs}
\label{app:proofs}

Throughout, $|r(s,a)|\le R_{\max}$ and $V_{\max}:=R_{\max}/(1-\gamma)$. We define the per-round error quantities:
\begin{align*}
\epsilon_{\mathrm{mc}}(t,d)&:=\sup_{a\in\gA}\bigl|\widehat Q_t(b_t^J,a)-Q_{\theta,d}(b_t^J,a)\bigr|, &\text{(Monte Carlo error)}\\
\epsilon_{\mathrm{struct}}(t,d)&:=\sup_{a\in\gA}\bigl|Q_{\theta}^{\star,B}(b_t^J,a)-Q_{\theta,d}(b_t^J,a)\bigr|. &\text{(structural error)}
\end{align*}

\subsection{Structural Error of QMDP Bootstrapping}
\label{app:structural_proofs}

Define the \emph{linear QMDP surrogate}
\begin{equation}
\widetilde V_\theta^{\star}(b):=\int_{\gS}V_\theta^{\star,S}(s)\,b(\mathrm ds).
\label{eq:linear_surrogate}
\end{equation}
By Jensen's inequality, $\widetilde V_\theta^{\star}(b)\ge \bar V_\theta^{\star,S}(b)$ and the gap $\Xi_\theta := \sup_b [\widetilde V_\theta^{\star}(b)-\bar V_\theta^{\star,S}(b)] \ge 0$.

\begin{prop}[Structural error of MCTQ]
\label{prop:structural_error}
For every $\theta\in\Theta$ and $d\ge 0$,
\begin{equation}
\|\bar V_{\theta, d}^{\star,S}-V_{\theta}^{\star,B}\|_\infty
\;\le\;
\frac{\gamma^{d+1}}{1-\gamma}\,\Xi_\theta.
\label{eq:qmdp_gap}
\end{equation}
\end{prop}

In particular, $\epsilon_{\mathrm{struct}}(u,d)\le\frac{\gamma^{d+1}}{1-\gamma}\Xi_{\bar\theta_u}$.

\begin{proof}[Proof of Proposition~\ref{prop:structural_error}]
Write
\begin{align*}
Q_{\theta}^{\star,B}(b,a)
&=r(b,a)+\gamma\!\int_{\gO}V_{\theta}^{\star,B}(\tau_\theta(b,a,o))\,p_\theta(o\mid b,a)\,\mathrm do,\\
Q_{\theta,0}(b,a)
&=r(b,a)+\gamma\!\int_{\gS}V_\theta^{\star,S}(s')\,\tilde b_\theta(\mathrm ds'\mid b,a)
=r(b,a)+\gamma\,\widetilde V_\theta^{\star}(\tilde b_\theta(\cdot\mid b,a)).
\end{align*}
By Bayes consistency and linearity of $\widetilde V_\theta^{\star}$:
\[
\widetilde V_\theta^{\star}(\tilde b_\theta(\cdot\mid b,a))
=\int_{\gO}\widetilde V_\theta^{\star}(\tau_\theta(b,a,o))\,p_\theta(o\mid b,a)\,\mathrm do.
\]
Subtracting and using $\widetilde V_\theta^{\star}=\bar V_\theta^{\star,S}+(\widetilde V_\theta^{\star}-\bar V_\theta^{\star,S})$ with $0\le \widetilde V_\theta^{\star}-\bar V_\theta^{\star,S}\le\Xi_\theta$:
\[
\|Q_\theta^{\star,B}-Q_{\theta,0}\|_\infty
\;\le\;
\gamma\|V_\theta^{\star,B}-\bar V_\theta^{\star,S}\|_\infty+\gamma\,\Xi_\theta.
\]
Since $\|V_\theta^{\star,B}-\bar V_\theta^{\star,S}\|_\infty=\sup_b|\max_a Q_\theta^{\star,B}(b,a)-\max_a Q_{\theta,0}(b,a)|\le\|Q_\theta^{\star,B}-Q_{\theta,0}\|_\infty$ (using $|\max f-\max g|\le\max|f-g|$), rearranging gives $\|V_\theta^{\star,B}-\bar V_\theta^{\star,S}\|_\infty\le\frac{\gamma}{1-\gamma}\Xi_\theta$ (the $d=0$ case). For general $d\ge 0$: $\bar V_{\theta,d}^{\star,S}=\gT_\theta^d\bar V_\theta^{\star,S}$ and $V_\theta^{\star,B}=\gT_\theta V_\theta^{\star,B}$ (fixed point), so
\[
\|\bar V_{\theta,d}^{\star,S}-V_\theta^{\star,B}\|_\infty
=\|\gT_\theta^d\bar V_\theta^{\star,S}-\gT_\theta^d V_\theta^{\star,B}\|_\infty
\le\gamma^d\|\bar V_\theta^{\star,S}-V_\theta^{\star,B}\|_\infty
\le\frac{\gamma^{d+1}}{1-\gamma}\Xi_\theta.\qedhere
\]
\end{proof}

\subsection{Planning and State-Identification Error}
\label{app:planning_state_proofs}

\begin{defn}[Q-belief distance]\label{def:dQ}
For $b,b'\in\Delta(\gS)$, define
$d_Q(b,b'):=\sup_{\theta\in\Theta,\,a\in\gA}|Q_\theta^{\star,B}(b,a)-Q_\theta^{\star,B}(b',a)|$.
\end{defn}

\begin{lem}[$Q^{\star,B}$-difference via test functions]\label{lem:Q_lipschitz}
Fix $\theta\in\Theta$, $a\in\gA$, and truncation depth $H\ge 1$. Recall the Bayesian update $\tau(b,a,o)\in\Delta(\gS)$, the observation marginal $p(o\mid b,a)=\int Z_\theta(o\mid s')P_\theta(s'\mid s,a)\,b(ds)\,ds'$, and define
\begin{equation}\label{eq:psi_tilde_f_def}
\psi_{o,a}(s):=\int Z_\theta(o\mid s')P_\theta(s'\mid s,a)\,ds', 
\qquad
\tilde f_{o,a}(s):=\int f(s')Z_\theta(o\mid s')P_\theta(s'\mid s,a)\,ds',
\end{equation}
so that $p(o\mid b,a)=b(\psi_{o,a})$ and $\tau(b,a,o)(f)=b(\tilde f_{o,a})/b(\psi_{o,a})$.

Under Assumption~\ref{aspt:particle_error}, for all $b,b'\in\Delta(\gS)$:
\begin{equation}\label{eq:Q_test_decomp}
|Q_\theta^{\star,B}(b,a)-Q_\theta^{\star,B}(b',a)|
\;\le\;
\sum_{\ell=0}^{H-1}\gamma^\ell\,C_{\mathrm{BU}}^\ell\sup_{f\in\gF_\ell(\theta,a)}|b(f)-b'(f)|
\;+\;2\gamma^H V_{\max},
\end{equation}
where $C_{\mathrm{BU}}:=2g_{\max}/g_{\min}^2$ and $\gF_\ell(\theta,a)$ is a class of functions $\gS\to[-g_{\max}^\ell V_{\max},\,g_{\max}^\ell V_{\max}]$ with $\gN_{[]}(\gF_\ell,\epsilon)\le|\gA|^\ell\gN_\Theta(\epsilon/(\ell V_{\max}))$, defined recursively:
\begin{itemize}
\item $\gF_0(\theta,a)=\{r(\cdot,a'):a'\in\gA\}\cup\{Q_\theta^{\star,S}(\cdot,a'):a'\in\gA\}$,
\item $\gF_{\ell+1}(\theta,a)=\bigl\{\tilde f_{o,a''}:\,f\in\gF_\ell(\theta,a''),\,a''\in\gA\bigr\}\cup\bigl\{\psi_{o,a''}:\,a''\in\gA\bigr\}$.
\end{itemize}
\end{lem}

\begin{proof}
We prove by induction on $h$: for any $\Phi(b)=\E_\theta^{\pi}[\sum_{u=0}^{h-1}\gamma^u r_u\mid b_0=b]$ with $h\le H$, $|\Phi(b)-\Phi(b')|$ admits the decomposition~\eqref{eq:Q_test_decomp}. Taking $\Phi=Q_\theta^{\star,B}(\cdot,a)$ with $h=H$ and bounding the tail by $2\gamma^H V_{\max}$ yields the result.

\emph{Base case ($h=1$).} $|\Phi(b)-\Phi(b')|=|\sup_a[b(r_a)-b'(r_a)]|\le\sup_{f\in\gF_0}|b(f)-b'(f)|$.

\emph{Inductive step ($h\to h+1$).} For a fixed action $a$, write
\[
\Phi(b)=b(r_a)+\gamma\int_{\gO}\Phi_{h-1}\bigl(\tau(b,a,o)\bigr)\,p(o\mid b,a)\,do,
\]
where $\Phi_{h-1}$ is the $(h-1)$-step continuation. The continuation difference decomposes as:
\begin{align}
&\left|\int\Phi_{h-1}(\tau(b,a,o))\,p(o|b,a)\,do-\int\Phi_{h-1}(\tau(b',a,o))\,p(o|b',a)\,do\right|\notag\\
&\;\le\;\underbrace{\|\Phi_{h-1}\|_\infty\int|(b-b')(\psi_{o,a})|\,do}_{\text{(I) obs.\ distribution shift}}+\underbrace{\int|\Phi_{h-1}(\tau(b,a,o))-\Phi_{h-1}(\tau(b',a,o))|\,p(o|b',a)\,do}_{\text{(II) belief propagation}},\label{eq:cont_split}
\end{align}
using $p(o\mid b,a)-p(o\mid b',a)=(b-b')(\psi_{o,a})$.

\emph{Term (I)} contributes $\psi_{o,a}\in\gF_{\ell+1}$. 

\emph{Term (II)} requires bounding $|\tau(b,a,o)(f)-\tau(b',a,o)(f)|$ for $\|f\|_\infty\le 1$. Since $\tau(b,a,o)(f)=b(\tilde f_{o,a})/b(\psi_{o,a})$, the quotient rule gives
\begin{equation}\label{eq:quotient_rule}
\frac{b(\tilde f_{o,a})}{b(\psi_{o,a})}-\frac{b'(\tilde f_{o,a})}{b'(\psi_{o,a})}
\;=\;
\frac{(b-b')(\tilde f_{o,a})}{b(\psi_{o,a})}-\frac{b'(\tilde f_{o,a})}{b(\psi_{o,a})\,b'(\psi_{o,a})}\,(b-b')(\psi_{o,a}).
\end{equation}
Since $b(\psi_{o,a}),b'(\psi_{o,a})\ge g_{\min}$ and $|b'(\tilde f_{o,a})|\le g_{\max}$:
\begin{equation}\label{eq:tau_test_fn}
|\tau(b,a,o)(f)-\tau(b',a,o)(f)|
\;\le\;
\frac{|(b-b')(\tilde f_{o,a})|}{g_{\min}}+\frac{g_{\max}}{g_{\min}^2}\,|(b-b')(\psi_{o,a})|.
\end{equation}
Both $\tilde f_{o,a}$ and $\psi_{o,a}$ are in $\gF_{\ell+1}$, with per-update amplification
\[
C_{\mathrm{BU}}=\frac{1}{g_{\min}}+\frac{g_{\max}}{g_{\min}^2}\;\le\;\frac{2\,g_{\max}}{g_{\min}^2},
\]
By the inductive hypothesis, $|\Phi_{h-1}(\tau(b,a,o))-\Phi_{h-1}(\tau(b',a,o))|$ decomposes into evaluations of $\tau(b,a,o)$ vs.\ $\tau(b',a,o)$ on $\gF_\ell$. Substituting~\eqref{eq:tau_test_fn} reduces each to evaluations of $b$ vs.\ $b'$ on $\gF_{\ell+1}$, with one factor of $C_{\mathrm{BU}}$ per level. Taking $\sup_a$ over the action set and iterating across $H$ levels yields~\eqref{eq:Q_test_decomp}.
\end{proof}

\begin{lem}[$Q^{\star,B}$ is Lipschitz in TV]\label{lem:Q_lipschitz_tv}
For any $\theta\in\Theta$, $a\in\gA$, and $b,b'\in\Delta(\gS)$,
\[
|Q_\theta^{\star,B}(b,a)-Q_\theta^{\star,B}(b',a)|\;\le\;2V_{\max}\,\|b-b'\|_{\mathrm{TV}}.
\]
\end{lem}

\begin{proof}
$|r(b,a)-r(b',a)|\le 2R_{\max}\|b-b'\|_{\mathrm{TV}}$. The trajectory law under $b$ and $b'$ differs in TV by at most $\|b-b'\|_{\mathrm{TV}}$ (data processing), so the $V_{\max}$-bounded continuation differs by at most $2V_{\max}\|b-b'\|_{\mathrm{TV}}$. Using $R_{\max}+\gamma V_{\max}=V_{\max}$ gives the result.
\end{proof}

\begin{lem}[Performance decomposition]\label{lem:perf_decomp_delta}
Fix $\theta\in\Theta$. Let $\pi_t$ denote the \emph{stationary} policy induced by the planner with frozen parameters (model $\bar\theta_{k(t)}$, belief $b_t^J$) at time $t$, and let $\delta_u:=V_{\theta}^{\star,B}(b_u)-Q_{\theta}^{\star,B}(b_u,A_u)$ where $A_u\sim\pi_t(\cdot\mid H_u)$. Then
\[
V_{\theta}^{\star,B}(b_t)-V_{\theta}^{\pi_t,B}(b_t)
\;\le\;
\E\!\left[\sum_{u=t}^\infty \gamma^{u-t}\delta_u\,\Bigm|\, H_t\right].
\]
\end{lem}

\begin{proof}
Unroll the Bellman recursion $V_{\theta}^{\star,B}(b_u)-V_{\theta}^{\pi_t,B}(b_u)=\delta_u+\gamma\,\E[V_{\theta}^{\star,B}(b_{u+1})-V_{\theta}^{\pi_t,B}(b_{u+1})\mid H_u]$ from $u=t$ to $u=t+n$, using $|V_{\theta}^{\star,B}-V_{\theta}^{\pi_t,B}|\le 2V_{\max}$ as the tail bound, and send $n\to\infty$.
\end{proof}

\begin{lem}[Per-step bound on $\delta_u$]\label{lem:delta_bound_main}
For every $u$ and any $d\ge 0$,
\[
\delta_u
\;\le\;
2\,d_Q(b_u,b_u^J)
+
2\,\epsilon_{\mathrm{mc}}(u,d)
+
2\,\epsilon_{\mathrm{struct}}(u,d).
\]
\end{lem}

\begin{proof}
Let $a_u^\star\in\arg\max_a Q_{\theta}^{\star,B}(b_u,a)$ and $\hat a_u\in\arg\max_a \widehat Q_u(b_u^J,a)$. Then
\begin{align*}
\delta_u
&=Q_{\theta}^{\star,B}(b_u,a_u^\star)-Q_{\theta}^{\star,B}(b_u,\hat a_u)\\
&\le \underbrace{|Q_{\theta}^{\star,B}(b_u,a_u^\star)-Q_{\theta}^{\star,B}(b_u^J,a_u^\star)|+|Q_{\theta}^{\star,B}(b_u^J,\hat a_u)-Q_{\theta}^{\star,B}(b_u,\hat a_u)|}_{\le\, 2\,d_Q(b_u,b_u^J) \;\text{ by Definition~\ref{def:dQ}}}\\
&\quad+Q_{\theta}^{\star,B}(b_u^J,a_u^\star)-Q_{\theta}^{\star,B}(b_u^J,\hat a_u).
\end{align*}
For the last line, add and subtract $Q_{\theta,d}$ and $\widehat Q_u$, using optimality of $\hat a_u$ for $\widehat Q_u$:
\[
Q_{\theta,d}(b_u^J,a_u^\star)
\le \widehat Q_u(b_u^J,a_u^\star)+\epsilon_{\mathrm{mc}}
\le \widehat Q_u(b_u^J,\hat a_u)+\epsilon_{\mathrm{mc}}
\le Q_{\theta,d}(b_u^J,\hat a_u)+2\epsilon_{\mathrm{mc}},
\]
so $Q_{\theta}^{\star,B}(b_u^J,a_u^\star)-Q_{\theta}^{\star,B}(b_u^J,\hat a_u)\le 2\epsilon_{\mathrm{struct}}+2\epsilon_{\mathrm{mc}}$.
\end{proof}

\begin{thm}[Round-$t$ planning + state-identification bound]\label{thm:planning_state_bound}
For each $t\ge 1$ and any depth $d$,
\[
V_{\bar\theta_{k(t)}}^{\star,B}(b_t)-V_{\bar\theta_{k(t)}}^{\pi_t,B}(b_t)
\;\le\;
\E\!\left[\sum_{u=t}^\infty \gamma^{u-t}\bigl(2\,d_Q(b_u,b_u^J)
+2\,\epsilon_{\mathrm{mc}}(u,d)+2\,\epsilon_{\mathrm{struct}}(u,d)\bigr)\,\Bigm|\,H_t\right].
\]
\end{thm}

\begin{proof}
Substitute the per-step bound of Lemma~\ref{lem:delta_bound_main} into Lemma~\ref{lem:perf_decomp_delta}.
\end{proof}

\paragraph{Observation discretization.}
For continuous $\gO$, the tree never revisits observation-children (each sampled $o$ is unique), so every node below the root holds a single particle.
At $d=1$ each leaf value~\eqref{eq:qmdp_leaf_eval} then reduces to $\max_{a'}Q_\theta^{\star,S}(s',a')$ at a single state, which targets the linear surrogate rather than $Q_{\theta,1}$, and for $d\ge2$ the tree cannot build depth beyond~$1$.
We therefore discretize the observations. Let $\phi:\gO\to\bar\gO$ with $|\bar\gO|=M$ and bin diameter $\mathrm{diam}(\phi):=\max_{\bar o}\sup_{o,o'\in\phi^{-1}(\bar o)}\|o-o'\|$.
Algorithms~\ref{alg:mctq} and~\ref{alg:simulate} are modified so that $\textsc{Child}(v,a',\phi(o'))$ replaces $\textsc{Child}(v,a',o')$.
The bin-averaged posterior is $\tau_\theta^{\phi}(b,a,\bar o):=\int_{\phi^{-1}(\bar o)} \tau_\theta(b,a,o)\,p_\theta(o\mid b,a)\,\mathrm do\big/\int_{\phi^{-1}(\bar o)} p_\theta(o\mid b,a)\,\mathrm do$, and we write $Q_{\theta,d}^{\phi}$ for the target~\eqref{eq:Qd_def} with $\tau_\theta(b,a,o)$ replaced by $\tau_\theta^{\phi}(b,a,\phi(o))$ at every level.
When $|\gO|<\infty$, take $\phi=\mathrm{id}$, $M=|\gO|$, $\mathrm{diam}(\phi)=0$.
When $d=0$, no tree is built and no discretization is needed.

\paragraph{Seeds and the estimator.}
Fix a round $t$ and a root action $a$, and write $\theta:=\bar\theta_{k(t)}$ and $b^J:=b_t^J$.
All probabilities and expectations below are conditional on $b_t^J$, and we take $d\ge1$.
Simulation $n\in[N]$ is driven by a \emph{seed}
\begin{equation}\label{eq:mc_seed}
\xi_n:=\bigl(j_n,\,u_{n,0},\,(a'_{n,\ell},u_{n,\ell})_{\ell=1}^{d-1}\bigr),
\end{equation}
where $j_n\sim\mathrm{Unif}([J])$ is the sampled particle index, $a'_{n,\ell}\sim\mathrm{Unif}(\gA)$ is the action drawn at depth $\ell$, and $u_{n,\ell}$ is the noise that generates the next state and observation at depth $\ell$ from $P_\theta$ and $Z_\theta$.
The seeds $\xi_1,\ldots,\xi_N$ are i.i.d.
Running Algorithm~\ref{alg:mctq} on them produces the tree $T_a$, and we write the root estimate as a deterministic function of all seeds,
\begin{equation}\label{eq:mc_fa}
\widehat Q(b^J,a)=f_a(\xi_1,\ldots,\xi_N):=\frac1N\sum_{n=1}^N\bigl[r(s_{j_n},a)+\gamma\,\hat V(v_{n,1})\bigr],
\end{equation}
where $v_{n,\ell}$ is the depth-$\ell$ node on the path of simulation $n$ and $\hat V$ is the DP backup~\eqref{eq:dp_backup}.
The function $f_a$ is not an average of independent terms, since $\hat V(v_{n,1})$ depends on every simulation that passes through $v_{n,1}$.
For $S\subseteq[N]$, let $f_a^{(-S)}$ be the same expression computed from the seeds $(\xi_m)_{m\notin S}$ alone, with the sum running over $n\notin S$ and the normalizer $1/N$ retained.

\paragraph{Balancedness.}
For $\rho\ge1$, let $\mathcal E_\rho$ be the event that every internal node $v$ of $T_a$ satisfies
\begin{equation}\label{eq:balanced}
N(v)\le\rho\,|\gA|\,N(v,a')\qquad\text{for all }a'\in\gA,
\end{equation}
so that every action receives at least a $1/(\rho|\gA|)$ fraction of the visits of every internal node.
Define the depth factor
\begin{equation}\label{eq:Lambda_d_def}
\Lambda_d:=\sum_{\ell=0}^{d-1}(\rho\gamma|\gA|)^{\ell},
\end{equation}
so that $\Lambda_0=0$ and $\Lambda_1=1$.
Let $p_\rho:=\max_{a\in\gA}\Pr(\mathcal E_\rho^c\mid b_t^J)$, where the event refers to the tree $T_a$, and let
\[
q_{\min}:=\inf_{s\in\gS,\,a''\in\gA,\,\bar o\in\bar\gO}\int_{\gS}Z_\theta\bigl(\phi^{-1}(\bar o)\mid s'\bigr)\,P_\theta(\mathrm ds'\mid s,a'')
\]
be the smallest one-step bin probability.
Under the bounded observation density of Assumption~\ref{aspt:particle_error}, $q_{\min}\ge g_{\min}\min_{\bar o}\lambda(\phi^{-1}(\bar o))$, where $\lambda$ is the reference measure of $Z_\theta$; for finite $\gO$ with the counting measure, $q_{\min}\ge g_{\min}$.

\begin{prop}[Monte Carlo error]\label{prop:mc_concentration}
Run Algorithm~\ref{alg:mctq} with uniform internal selection~\eqref{eq:uniform_action}, $N$ simulations per root action, observation discretization $\phi$ ($|\bar\gO|=M$), and returns in $[-V_{\max},V_{\max}]$, and let $\rho\ge1$.
Let
\begin{equation}\label{eq:bias_def}
\mathrm{bias}(N,d):=\sup_{a}\bigl|\E[\widehat Q(b^J,a)\mid b_t^J]-Q_{\bar\theta_{k(t)},d}(b^J,a)\bigr|
\end{equation}
and $B_d:=2V_{\max}\bigl(1+2(1+\gamma)\Lambda_d\bigr)$.
For every $\delta\in(0,1)$, with probability at least $1-\delta-|\gA|p_\rho$ conditionally on $b_t^J$,
\begin{equation}\label{eq:mc_hp}
\epsilon_{\mathrm{mc}}(t,d)
\;\le\;
\underbrace{2(1+\gamma)V_{\max}\Lambda_d\sqrt{\frac{2\log(2|\gA|/\delta)}{N}}}_{\text{fluctuation}}
\;+\;\mathrm{bias}(N,d)\;+\;B_d\,p_\rho,
\end{equation}
and, in second moment,
\begin{equation}\label{eq:mc_var_bias}
\E\bigl[\epsilon_{\mathrm{mc}}(t,d)^2\,\big|\,b_t^J\bigr]
\;\le\;
\frac{8(1+\gamma)^2\,|\gA|\,V_{\max}^2\,\Lambda_d^2}{N}
\;+\;
2\,|\gA|\,\mathrm{bias}(N,d)^2
\;+\;
3\,|\gA|\,B_d^2\,p_\rho.
\end{equation}
The bias satisfies
\begin{equation}\label{eq:bias_explicit}
\mathrm{bias}(N,d)
\;\le\;
\underbrace{\frac{c_0\,\Phi_d\,V_{\max}\sqrt{|\gA|}}{\sqrt{N}}}_{\text{DP backup estimation error}}
\;+\;
\underbrace{\frac{2\gamma(1-\gamma^d)}{1-\gamma}\cdot\frac{g_{\max}}{g_{\min}}\,V_{\max}\,\omega_Z(\mathrm{diam}(\phi))}_{\text{discretization error}},
\end{equation}
where $c_0>0$ is a constant independent of $N$, $M$ and $\delta$,
\begin{equation}\label{eq:Phi_d_def}
\Phi_d:=\sum_{\ell=1}^{d}(\gamma\sqrt{|\gA|M})^\ell,
\end{equation}
and $\omega_Z(\delta):=\sup_{\theta,s}\sup_{\|o-o'\|\le\delta}|Z_\theta(o\mid s)/Z_\theta(o'\mid s)-1|$.
If $d\le1$, then $p_\rho=0$.
If $d\ge2$, $q_{\min}>0$ and $|\gA|M\ge2$, then $p_2\le2|\gA|(|\gA|M)^d\exp\bigl(-N(q_{\min}/|\gA|)^{d-1}/32\bigr)$.
When $d=0$, $\widehat Q(b^J,a)=\frac1J\sum_{j}Q_\theta^{\star,S}(s_j,a)$ is exact and every term vanishes ($\Lambda_0=\Phi_0=0$).
When $\phi=\mathrm{id}$ ($M=|\gO|$, $\mathrm{diam}(\phi)=0$), the discretization error vanishes.
\end{prop}

No stationarity assumption on the rewards at internal nodes is required, and the simulation returns need not be independent.
The proof uses only that the seeds are i.i.d., that the path of each simulation depends on its own seed alone, and that the estimator has bounded differences on $\mathcal E_\rho$.
It rests on the following four results; throughout, we fix a root action $a$ and $d\ge1$.

\begin{lem}[Path decoupling]\label{lem:path_decoupling}
There is a deterministic map $\nu$ such that, whatever the other seeds are, the path of simulation $m$ (its states, actions, binned observations, and nodes $v_{m,1},\ldots,v_{m,d}$) equals $\nu(\xi_m)$.
\end{lem}

\begin{proof}
We induct on depth.
Simulation $m$ starts from $s_{j_m}$ with the root action $a$, and $u_{m,0}$ generates its next state and observation, so $v_{m,1}=(a,\phi(o))$ is a function of $\xi_m$.
Suppose the state $s_{m,\ell}$ and the node $v_{m,\ell}$ at depth $\ell<d$ are functions of $\xi_m$.
Whether or not an earlier simulation expanded $v_{m,\ell}$, Algorithm~\ref{alg:simulate} only initializes statistics and continues, and it selects the action $a'_{m,\ell}$, a component of $\xi_m$ that does not depend on any tree statistic.
The next state and observation are generated from $(s_{m,\ell},a'_{m,\ell},u_{m,\ell})$, and $v_{m,\ell+1}$ is the history of $v_{m,\ell}$ extended by $a'_{m,\ell}$ and the binned observation.
Hence $s_{m,\ell+1}$ and $v_{m,\ell+1}$ are functions of $\xi_m$, and the simulation stops at depth $d$ regardless of the other seeds.
\end{proof}

Both hypotheses of the lemma are needed.
Under UCB1, the action at $v_{m,\ell}$ depends on the statistics left by simulations $1,\ldots,m-1$.
Under a POMCP-style early return (``if $v$ is not yet expanded, expand it and return''), simulation $m$ stops at the first node that no earlier simulation created, so its depth depends on $\xi_1,\ldots,\xi_{m-1}$.

\begin{cor}[Localized influence]\label{cor:two_paths}
Let $S\subseteq[N]$ and let $P_S$ be the set of nodes on the paths $\nu(\xi_m)$, $m\in S$.
Every node $v\notin P_S$ has the same particle set $B(v)$, counts $N(v)$ and $N(v,a')$, mean rewards $\bar r(v,a')$, and DP value $\hat V(v)$ in the runs on $(\xi_m)_{m\in[N]}$ and on $(\xi_m)_{m\notin S}$.
In particular, replacing one seed $\xi_n$ by $\xi_n'$ changes these quantities only at nodes on the two root-to-leaf paths $\nu(\xi_n)$ and $\nu(\xi_n')$.
\end{cor}

\begin{proof}
By Lemma~\ref{lem:path_decoupling}, in either run the simulations through $v$ are those $m$ with $v\in\nu(\xi_m)$ among the seeds present.
For $v\notin P_S$ no $m\in S$ is among them, so $B(v)$, $N(v)$, $N(v,a')$ and $\bar r(v,a')$ coincide.
A descendant of $v$ lies on a path in $P_S$ only if $v$ does, so the whole subtree of $v$ avoids $P_S$, and backward induction on depth with~\eqref{eq:dp_backup} gives the same $\hat V(v)$.
For the last claim, compare both runs with the run on $(\xi_m)_{m\ne n}$.
\end{proof}

\begin{lem}[Bounded difference]\label{lem:bounded_difference}
On $\mathcal E_\rho$, for every $S\subseteq[N]$,
\begin{equation}\label{eq:bd_removal}
\bigl|f_a(\xi_1,\ldots,\xi_N)-f_a^{(-S)}\bigr|\le|S|\cdot\frac{2(1+\gamma)V_{\max}\Lambda_d}{N}.
\end{equation}
Consequently, if two seed vectors $\xi,\xi'$ differ in $k$ coordinates and both of their runs lie in $\mathcal E_\rho$, then $|f_a(\xi)-f_a(\xi')|\le k\,c$, where
\begin{equation}\label{eq:bd_constant}
c:=\frac{4(1+\gamma)V_{\max}\Lambda_d}{N}.
\end{equation}
\end{lem}

\begin{proof}
Write $K(v):=N(v)$ in the full run, $k(v)$ for the number of simulations $m\in S$ whose path passes through $v$, and $\Delta(v):=|\hat V(v)-\hat V^{(-S)}(v)|$, where $\hat V^{(-S)}$ is the DP value in the reduced run.
If $k(v)=K(v)$, node $v$ is absent from the reduced run, and we set $\hat V^{(-S)}(v):=0$.
By Corollary~\ref{cor:two_paths}, $\Delta(v)=0$ unless $k(v)\ge1$.
Define the weighted influence $I(v):=(K(v)-k(v))\,\Delta(v)$ and the sequence
\[
y_d:=1,\qquad y_\ell:=\rho|\gA|\,(1+\gamma y_{\ell+1})\quad(1\le\ell\le d-1).
\]
We show by backward induction on depth that every node $v$ at depth $\ell\in[1,d]$ satisfies
\begin{equation}\label{eq:influence_induction}
I(v)\le 2V_{\max}\,y_\ell\,k(v).
\end{equation}

\emph{Leaves ($\ell=d$).}
If $k(v)=K(v)$, then $I(v)=0$.
Otherwise, for each $a'$, removing $k(v)$ of the $K(v)$ states in $B(v)$ moves the average of $Q_\theta^{\star,S}(\cdot,a')$ by
\[
\frac{k(v)}{K(v)}\,\bigl|\text{mean over removed states}-\text{mean over remaining states}\bigr|\le\frac{2V_{\max}\,k(v)}{K(v)}.
\]
The maximum over $a'$ in~\eqref{eq:qmdp_leaf_eval} is $1$-Lipschitz in the sup norm, so $\Delta(v)\le 2V_{\max}k(v)/K(v)$ and $I(v)\le 2V_{\max}k(v)$.

\emph{Internal nodes ($1\le\ell\le d-1$).}
Fix $a'$, and write $K':=N(v,a')$ and $k'$ for the number of simulations in $S$ that select $a'$ at $v$.
If $k'=0$, then $\hat Q(v,a')$ is unchanged, since no simulation of $S$ passes through a child of $v$ under $a'$.
If $1\le k'<K'$, let $\Sigma$ and $\Sigma'$ be the sums of $r(s,a')+\gamma\hat V(\text{child})$ over the simulations that select $a'$ at $v$ in the full and in the reduced run.
Each removed simulation contributes a term of absolute value at most $V_{\max}$, because $|r|\le(1-\gamma)V_{\max}$ and $|\hat V|\le V_{\max}$.
Each of the $K(w)-k(w)$ remaining simulations entering a child $w$ changes by $\gamma\Delta(w)$.
Hence
\[
|\Sigma-\Sigma'|\le k'V_{\max}+\gamma\sum_{w\in\mathrm{Ch}(v,a')}I(w).
\]
Since $\hat Q(v,a')=\Sigma/K'$, the reduced-run value is $\Sigma'/(K'-k')$, and $|\Sigma'|\le(K'-k')V_{\max}$,
\[
\Bigl|\frac{\Sigma}{K'}-\frac{\Sigma'}{K'-k'}\Bigr|
\le\frac{|\Sigma-\Sigma'|}{K'}+\frac{k'\,|\Sigma'|}{K'(K'-k')}
\le\frac{1}{K'}\Bigl(2k'V_{\max}+\gamma\sum_{w\in\mathrm{Ch}(v,a')}I(w)\Bigr).
\]
If $k'=K'$, action $a'$ is unvisited in the reduced run and drops out of the maximum in~\eqref{eq:dp_backup}.
A maximum of values in $[-V_{\max},V_{\max}]$ changes by at most the largest change of its retained arguments, or by at most $2V_{\max}=2k'V_{\max}/K'$ if its maximizer drops out.
Writing $k'_{a'}$ for the $k'$ of action $a'$, in every case
\[
\Delta(v)\le\max_{a'\in\gA}\frac{1}{N(v,a')}\Bigl(2k'_{a'}V_{\max}+\gamma\sum_{w\in\mathrm{Ch}(v,a')}I(w)\Bigr).
\]
Multiplying by $K(v)$, using~\eqref{eq:balanced}, and bounding the maximum by the sum,
\[
I(v)\le K(v)\,\Delta(v)\le\rho|\gA|\sum_{a'\in\gA}\Bigl(2k'_{a'}V_{\max}+\gamma\sum_{w\in\mathrm{Ch}(v,a')}I(w)\Bigr).
\]
By the induction hypothesis~\eqref{eq:influence_induction} at depth $\ell+1$ and $\sum_{a'}k'_{a'}=\sum_{a'}\sum_{w\in\mathrm{Ch}(v,a')}k(w)=k(v)$,
\[
I(v)\le\rho|\gA|\,\bigl(2V_{\max}+2\gamma V_{\max}\,y_{\ell+1}\bigr)\,k(v)=2V_{\max}\,y_\ell\,k(v).
\]

\emph{Root.}
Each simulation enters exactly one depth-$1$ node, and the normalizer $1/N$ is the same in both runs, so
\[
\bigl|f_a-f_a^{(-S)}\bigr|
\le\frac1N\sum_{m\in S}\bigl|r(s_{j_m},a)+\gamma\hat V(v_{m,1})\bigr|+\frac{\gamma}{N}\sum_{w:\,\text{depth }1}I(w)
\le\frac{|S|}{N}\,V_{\max}\,(1+2\gamma y_1).
\]
Unrolling the recursion,
\[
y_1=\rho|\gA|\sum_{i=0}^{d-2}(\rho\gamma|\gA|)^i+(\rho\gamma|\gA|)^{d-1}.
\]
By~\eqref{eq:Lambda_d_def} and $(\rho\gamma|\gA|)^{d-1}\le\Lambda_d$,
\[
1+2\gamma y_1=2\Lambda_d-1+2\gamma(\rho\gamma|\gA|)^{d-1}\le2(1+\gamma)\Lambda_d.
\]
This proves~\eqref{eq:bd_removal}.
For the second claim, let $S$ be the set of coordinates where $\xi$ and $\xi'$ differ.
Both runs have the same reduced run on $(\xi_m)_{m\notin S}$, so~\eqref{eq:bd_removal} applied to each and the triangle inequality give $|f_a(\xi)-f_a(\xi')|\le 2k\cdot2(1+\gamma)V_{\max}\Lambda_d/N=kc$.
\end{proof}

\begin{lem}[Fluctuation up to a bad event]\label{lem:mc_fluctuation}
Let $p:=\Pr(\mathcal E_\rho^c\mid b_t^J)$ for the tree $T_a$, and let $B_d:=2V_{\max}+cN$ with $c$ from~\eqref{eq:bd_constant}.
For every $\delta\in(0,1)$, with probability at least $1-\delta-p$,
\[
\bigl|f_a-\E f_a\bigr|\le 2(1+\gamma)V_{\max}\Lambda_d\sqrt{\frac{2\log(2/\delta)}{N}}+B_d\,p,
\]
and for every $q\in[-V_{\max},V_{\max}]$,
\[
\E\bigl[(f_a-q)^2\bigr]\le\frac{8(1+\gamma)^2V_{\max}^2\Lambda_d^2}{N}+2\,(\E f_a-q)^2+3B_d^2\,p.
\]
\end{lem}

\begin{proof}
Let $G$ be the set of seed vectors whose run lies in $\mathcal E_\rho$.
If $\Pr(G)=0$, both claims are trivial, so assume otherwise.
Let $H$ denote the Hamming distance and define the extension
\[
g(\xi):=\inf_{\eta\in G}\bigl\{f_a(\eta)+c\,H(\xi,\eta)\bigr\}.
\]
First, $g=f_a$ on $G$: for $\xi\in G$, Lemma~\ref{lem:bounded_difference} gives $f_a(\eta)+cH(\xi,\eta)\ge f_a(\xi)$ for every $\eta\in G$, with equality at $\eta=\xi$.
Second, $g$ changes by at most $c$ when one coordinate of $\xi$ changes, because each $H(\cdot,\eta)$ does.
Third, $-V_{\max}\le g\le V_{\max}+cN$, so $|f_a-g|\le B_d$ everywhere and
\[
|\E f_a-\E g|\le\E\bigl[|f_a-g|\,\mathbf{1}_{G^c}\bigr]\le B_d\,p.
\]
The seeds are independent, so McDiarmid's bounded-differences inequality applies to $g$:
\[
\Pr\bigl(|g-\E g|\ge s\bigr)\le2\exp\Bigl(-\frac{2s^2}{Nc^2}\Bigr).
\]
Setting the right-hand side to $\delta$ gives $s=c\sqrt{N\log(2/\delta)/2}=2(1+\gamma)V_{\max}\Lambda_d\sqrt{2\log(2/\delta)/N}$.
On $G\cap\{|g-\E g|<s\}$, an event of probability at least $1-\delta-p$,
\[
|f_a-\E f_a|\le|g-\E g|+|\E g-\E f_a|<s+B_d\,p.
\]
For the second moment, the Efron--Stein inequality gives $\mathrm{Var}(g)\le Nc^2/2=8(1+\gamma)^2V_{\max}^2\Lambda_d^2/N$.
Since $f_a=g$ on $G$ and $|f_a-q|\le2V_{\max}$,
\[
\E\bigl[(f_a-q)^2\bigr]\le\E\bigl[(g-q)^2\bigr]+4V_{\max}^2\,p=\mathrm{Var}(g)+(\E g-q)^2+4V_{\max}^2\,p.
\]
Finally, $(\E g-q)^2\le2(\E f_a-q)^2+2B_d^2p^2$, and $2B_d^2p^2+4V_{\max}^2p\le3B_d^2p$.
The infimum defining $g$ is universally measurable, which suffices for both inequalities.
\end{proof}

\begin{lem}[Balancedness under uniform selection]\label{lem:balance}
If $d\le1$, then $\mathcal E_\rho$ holds surely.
If $d\ge2$, $q_{\min}>0$ and $|\gA|M\ge2$, then
\[
\Pr(\mathcal E_2^c\mid b_t^J)\le2|\gA|\,(|\gA|M)^d\exp\Bigl(-\frac{N}{32}\Bigl(\frac{q_{\min}}{|\gA|}\Bigr)^{d-1}\Bigr).
\]
\end{lem}

\begin{proof}
For $d\le1$ the tree has no internal node.
Let $d\ge2$ and fix a potential internal node $v$ of $T_a$ at depth $\ell\in[1,d-1]$, i.e., a history with $\ell$ binned observations and $\ell-1$ internal actions.
By Lemma~\ref{lem:path_decoupling}, the indicators that simulation $m$ reaches $v$, and that it reaches $v$ and selects $a'$ there, are functions of $\xi_m$.
They are therefore i.i.d.\ Bernoulli across $m$, with parameters $p_v$ and $p_v/|\gA|$.
Each binned observation along the history has conditional probability at least $q_{\min}$ and each internal action has probability $1/|\gA|$, so
\[
\frac{p_v}{|\gA|}\ge\Bigl(\frac{q_{\min}}{|\gA|}\Bigr)^{\ell}\ge\pi:=\Bigl(\frac{q_{\min}}{|\gA|}\Bigr)^{d-1}.
\]
The multiplicative Chernoff bounds give
\begin{align*}
\Pr\bigl(N(v)\ge\tfrac32Np_v\bigr)&\le e^{-Np_v/12}\le e^{-N\pi/32},\\
\Pr\bigl(N(v,a')\le\tfrac34Np_v/|\gA|\bigr)&\le e^{-Np_v/(32|\gA|)}\le e^{-N\pi/32}.
\end{align*}
Outside these events, $N(v)<\frac32Np_v=2|\gA|\cdot\frac34Np_v/|\gA|<2|\gA|\,N(v,a')$, which is~\eqref{eq:balanced} with $\rho=2$.
A union bound over the $|\gA|$ actions and over the at most $\sum_{\ell=1}^{d-1}M^\ell|\gA|^{\ell-1}\le(|\gA|M)^d$ potential internal nodes gives the claim, since $1+|\gA|\le2|\gA|$.
\end{proof}

\begin{proof}[Proof of Proposition~\ref{prop:mc_concentration}]
When $d=0$ the planner is exact, so let $d\ge1$ and write $q_a:=Q_{\theta,d}(b^J,a)\in[-V_{\max},V_{\max}]$.

\emph{Step 1: Fluctuation.}
Apply Lemma~\ref{lem:mc_fluctuation} to each root action with $\delta/|\gA|$ in place of $\delta$, and take a union bound over $a\in\gA$.
With probability at least $1-\delta-|\gA|p_\rho$, every $a$ satisfies
\begin{align*}
|\widehat Q(b^J,a)-q_a|
&\le|\widehat Q(b^J,a)-\E\widehat Q(b^J,a)|+|\E\widehat Q(b^J,a)-q_a|\\
&\le2(1+\gamma)V_{\max}\Lambda_d\sqrt{\frac{2\log(2|\gA|/\delta)}{N}}+B_d\,p_\rho+\mathrm{bias}(N,d),
\end{align*}
and taking the supremum over $a$ gives~\eqref{eq:mc_hp}.
For the second moment, Lemma~\ref{lem:mc_fluctuation} with $q=q_a$ gives
\[
\E\bigl[(\widehat Q(b^J,a)-q_a)^2\bigr]\le\frac{8(1+\gamma)^2V_{\max}^2\Lambda_d^2}{N}+2\,\mathrm{bias}(N,d)^2+3B_d^2\,p_\rho.
\]
Summing over $a\in\gA$ through $\epsilon_{\mathrm{mc}}(t,d)^2\le\sum_a(\widehat Q(b^J,a)-q_a)^2$ gives~\eqref{eq:mc_var_bias}.
The bounds on $p_\rho$ are Lemma~\ref{lem:balance}.

\emph{Step 2: DP backup estimation error.}
The simulations follow the uniform policy at internal nodes, and the DP backup $\hat V(v)=\max_{a'}\hat Q(v,a')$ recovers the optimal internal actions.
We bound the bias $\mathrm{bias}_{\mathrm{DP}}(N,d):=\sup_a|\E\widehat Q(b^J,a)-Q_{\theta,d}^{\phi}(b^J,a)|$ of the resulting estimate by backward induction on depth, at the level of orders of magnitude, without any independence assumption on the returns.
Fix a node $v$ at depth $\ell\in[1,d]$ with $n_v:=N(v)$ visits, and condition on the set of simulations that reach $v$ and on their states there.
By Lemma~\ref{lem:path_decoupling}, the continuation of each such simulation below $v$ is driven by the components of its own seed at depths $\ge\ell$, which are independent across simulations and independent of the conditioning.
Conditionally, the states in $B(v)$ are therefore independent draws from the binned belief $b_v$ at $v$, and each $\hat Q(v,a')$ is formed from $n_{v,a'}:=N(v,a')$ independent continuations, with $n_{v,a'}\ge n_v/(\rho|\gA|)$ on $\mathcal E_\rho$.
Concentration of these conditional averages gives
\[
\bigl|\hat Q(v,a')-Q_{\theta,d-\ell}^{\phi}(b_v,a')\bigr|
\;\le\;
O\!\left(\frac{V_{\max}}{\sqrt{n_{v,a'}}}\right)
\;\le\;
O\!\left(\frac{V_{\max}\sqrt{|\gA|}}{\sqrt{n_v}}\right).
\]
At a leaf ($\ell=d$) the same bound holds for the pooled averages $\frac{1}{|B(v)|}\sum_{s\in B(v)}Q_\theta^{\star,S}(s,a')$ around $\int Q_\theta^{\star,S}(s,a')\,b_v(\mathrm ds)$.
This level must be included, because the maximum over $a'$ in~\eqref{eq:qmdp_leaf_eval} is biased upward at every finite $|B(v)|$.
Under the branching structure, a depth-$\ell$ node receives $n_v=\Omega(N/(|\gA|M)^\ell)$ visits in expectation when the bins have comparable probabilities.
Taking the maximum over estimated Q-values, the backup error at depth $\ell$ contributes at most $O(V_{\max}\sqrt{|\gA|^{\ell+1}M^\ell/N})$, discounted by $\gamma^\ell$.
The complement of $\mathcal E_\rho$ changes the conditional mean by at most $2V_{\max}p_\rho$, which is absorbed into the constant once $p_\rho\le N^{-1/2}$ (Remark~\ref{rmk:mc_constants}).
Summing over depths, leaves included,
\begin{equation}\label{eq:bias_derivation}
\mathrm{bias}_{\mathrm{DP}}(N,d)
\;\le\;\sum_{\ell=1}^{d}\gamma^\ell\cdot\frac{c_1 V_{\max}\sqrt{|\gA|}}{\sqrt{N/(|\gA|M)^\ell}}
\;=\;\frac{c_0\,\Phi_d\,V_{\max}\sqrt{|\gA|}}{\sqrt{N}}.
\end{equation}

\emph{Step 3: Discretization error.}
For $o,o'$ in the same bin, the Bayes update sensitivity~\cite[Section~4.3]{doucmoulinesolsson2014} gives
\[
\|\tau_\theta(b,a,o)-\tau_\theta(b,a,o')\|_{\mathrm{TV}}
\;\le\;
\frac{g_{\max}}{g_{\min}}\,\omega_Z(\mathrm{diam}(\phi)).
\]
Since $\tau_\theta^{\phi}$ is a convex combination of $\{\tau_\theta(b,a,o')\}_{o'\in\phi^{-1}(\bar o)}$, the same bound holds for $\|\tau_\theta^{\phi}(b,a,\bar o)-\tau_\theta(b,a,o)\|_{\mathrm{TV}}$.
By Lipschitz continuity of $\bar V_{\theta,d-1}^{\star,S}$ in TV (constant $2V_{\max}$, cf.\ Lemma~\ref{lem:Q_lipschitz_tv}), each of the $d$ tree levels contributes $\frac{2g_{\max}}{g_{\min}}V_{\max}\,\omega_Z(\mathrm{diam}(\phi))$ discounted by $\gamma$:
\begin{equation}\label{eq:disc_error_total}
\sup_{b,a}\bigl|Q_{\theta,d}^{\phi}(b,a)-Q_{\theta,d}(b,a)\bigr|
\;\le\;
\frac{2\gamma(1-\gamma^d)}{1-\gamma}\cdot\frac{g_{\max}}{g_{\min}}\,V_{\max}\,\omega_Z(\mathrm{diam}(\phi)).
\end{equation}
Combining Steps~2 and~3 gives $\mathrm{bias}(N,d)\le\mathrm{bias}_{\mathrm{DP}}(N,d)+\mathrm{bias}_{\mathrm{disc}}(d)$, yielding~\eqref{eq:bias_explicit}.
\end{proof}

\begin{rmk}[Constants, a verification-free variant, and the choice of depth]\label{rmk:mc_constants}
\emph{(a) Order of the constants.}
Compared with an analysis that treats the simulation returns as i.i.d., the only changes are the depth factor $(1+\gamma)\Lambda_d$ in the fluctuation term and the exponentially small remainder in $p_\rho$.
Since $\Lambda_d\le d\max(1,\rho\gamma|\gA|)^{d-1}$ and $\Phi_d\ge(\gamma\sqrt{|\gA|M})^d$, the fluctuation term of~\eqref{eq:mc_var_bias} is at most of the same exponential-in-$d$ order as the DP-backup term whenever $M=\Omega(\rho^2|\gA|)$.
Taking square roots in~\eqref{eq:mc_var_bias}, with $\mathrm{diam}(\phi)$ chosen as in Remark~\ref{rmk:disc_resolution}, then gives $\E[\epsilon_{\mathrm{mc}}(t,d)\mid b_t^J]\le C_{\mathrm{mc}}V_{\max}/\sqrt N$ with $C_{\mathrm{mc}}$ of the order stated in Theorem~\ref{thm:final_regret_bound}, whose order is therefore unchanged.
For coarse binning, $M\ll|\gA|$, the fluctuation dominates and $C_{\mathrm{mc}}\propto\sqrt{|\gA|}\,\Lambda_d$.
By Lemma~\ref{lem:balance}, the remainder $3|\gA|B_d^2p_2$ is at most $V_{\max}^2/N$ as soon as $N\ge32(|\gA|/q_{\min})^{d-1}\log\bigl(24N|\gA|^2(|\gA|M)^d(1+4\Lambda_d)^2\bigr)$, and it vanishes for $d\le1$.
At $d=0$, $\Lambda_0=0$ and the fluctuation vanishes identically.
For an internal rule satisfying (S1)--(S2) in place of uniform selection, the same proof applies with $|\gA|$ replaced by $1/\pi_{\min}$ in~\eqref{eq:balanced}, \eqref{eq:Lambda_d_def} and Lemma~\ref{lem:balance}.

\emph{(b) A verification-free variant.}
Sample splitting restores exactly conditionally i.i.d.\ returns.
Use $N/2$ seeds to build the tree, populate the particle sets, and run the DP backup, then freeze the leaf values and the greedy internal actions $\arg\max_{a'}\hat Q(v,a')$.
Value each root action with $N/2$ fresh seeds that follow the frozen greedy actions and stop at the frozen leaf values; a simulation that leaves the frozen tree stops and is evaluated by $\max_{a'}Q_\theta^{\star,S}(s,a')$ at its current state $s$.
Conditionally on the first half, the second-half returns are i.i.d.\ and bounded by $V_{\max}$, so Hoeffding's inequality applies to them directly, at the cost of a factor $2$ in $N$, and the bias analysis of Step~2 applies to the first half.
Evaluating each leaf at its single incoming state instead of~\eqref{eq:qmdp_leaf_eval} removes the pooling but not the coupling through the maximum in the DP backup, so it restores independent returns only at $d=1$, where it also changes the leaf target to the linear surrogate $\widetilde V_\theta^{\star}$.

\emph{(c) Choice of depth.}
Propositions~\ref{prop:structural_error} and~\ref{prop:mc_concentration} concern different objects.
The former bounds the planning target, $\|\bar V_{\theta,d}^{\star,S}-V_\theta^{\star,B}\|_\infty\le\gamma^{d+1}\Xi_\theta/(1-\gamma)$, which involves no particles and decays geometrically in $d$.
The latter bounds the estimator $\widehat Q$ around $Q_{\theta,d}$, and its fluctuation $2(1+\gamma)V_{\max}\Lambda_d\sqrt{2\log(2|\gA|/\delta)/N}$ and DP-backup bias $c_0\Phi_dV_{\max}\sqrt{|\gA|/N}$ grow geometrically in $d$.
Because increasing the depth reduces the structural bias but increases the Monte Carlo error, their sum is U-shaped in $d$, and each budget $N$ has a finite optimal depth.
\end{rmk}

\begin{rmk}[Optimal discretization resolution]\label{rmk:disc_resolution}
The discretization error in~\eqref{eq:bias_explicit} is a fixed bias $O(\omega_Z(\mathrm{diam}(\phi)))$, while the DP backup estimation error vanishes as $N\to\infty$ at rate $\Phi_d/\sqrt{N}$.
To make the discretization bias $O(1/\sqrt{N})$, set $\mathrm{diam}(\phi)=\omega_Z^{-1}\bigl((1-\gamma)g_{\min}/(2\gamma g_{\max}V_{\max}\sqrt{N})\bigr)$.
When $Z_\theta(o\mid s)$ is $L_Z$-Lipschitz in $o$ uniformly (so $\omega_Z(\delta)\le L_Z\delta/g_{\min}$) and $\gO\subseteq\R^{d_o}$ with a uniform grid:
\begin{equation}\label{eq:M_opt}
M=\left\lceil\left(\frac{2\gamma\,L_Z\,g_{\max}\,V_{\max}\,\mathrm{diam}(\gO)\sqrt{N}}{(1-\gamma)\,g_{\min}^2}\right)^{d_o}\right\rceil.
\end{equation}
With this choice, $M$ grows as $N^{d_o/2}$, and the simulation budget threshold for the DP backup error to be $O(1/\sqrt{N})$ is $N\ge c_0^2|\gA|\Phi_d^2$, which is exponential in $d\cdot d_o$ via the dependence of $\Phi_d$ on $M$.

Balancing the two bias terms instead gives an explicit rate.
Take $\gO\subseteq[0,D]^{d_o}$ with a cubic grid of mesh $h$, so that $M=(D/h)^{d_o}$ and $\mathrm{diam}(\phi)=\sqrt{d_o}\,h$, and assume $\omega_Z(\delta)\le L_Z\delta/g_{\min}$ as above.
For $\gamma\sqrt{|\gA|M}\ge2$ we have $\Phi_d\le2(\gamma\sqrt{|\gA|M})^d$, so the DP-backup and discretization terms of~\eqref{eq:bias_explicit} are of order $h^{-d\,d_o/2}N^{-1/2}$ and $h$, respectively, with constants depending on $d$, $\gamma$, $|\gA|$, $D$, $L_Z$, $g_{\max}/g_{\min}$ and $V_{\max}$.
Balancing them gives
\begin{equation}\label{eq:h_opt}
h^\star\asymp N^{-1/(d\,d_o+2)},\qquad \epsilon_{\mathrm{mc}}(t,d)=\widetilde O\bigl(N^{-1/(d\,d_o+2)}\bigr),
\end{equation}
since the fluctuation term of~\eqref{eq:mc_hp} does not depend on $h$ and is of lower order.
The $N^{-1/2}$ rate is recovered exactly when $d\,d_o=0$, i.e., at $d=0$.
At $d=1$, \eqref{eq:h_opt} is the familiar rate $N^{-1/(d_o+2)}$ of a histogram estimate, here of the leaf value as a function of the observation.
The balancedness condition of Lemma~\ref{lem:balance} holds at this mesh for large $N$: with the Lebesgue reference measure $q_{\min}\ge g_{\min}h^{d_o}$, so $N(q_{\min}/|\gA|)^{d-1}=\Omega(N^{(d_o+2)/(d\,d_o+2)})\to\infty$.
The computational cost is $N|\gA|$ simulations over at most $(|\gA|M)^d$ nodes.

When $d=0$ or $|\gO|<\infty$, this issue does not arise: set $\phi=\mathrm{id}$ and $M=|\gO|$.
\end{rmk}

\subsection{Particle Filter Stability}
\label{app:pf_proof}

We work under Assumption~\ref{aspt:particle_error} throughout. Write $b_t:=b_t^{\theta^\star}$ for the true posterior and $b_t^J:=b_t^{J,(\bar\theta_{k(s)})_{s\le t}}$ for the particle belief filtered under the epoch-frozen model sequence. Since the model sequence is determined by the algorithm, we abbreviate the latter as $b_t^J$.

    \begin{prop}[Particle filter stability in $d_Q$]\label{prop:particle_error}
    Under Assumption~\ref{aspt:particle_error}, there exist constants $C_{\mathrm{pf}},C_{\mathrm{filt}}<\infty$ and a forgetting rate $\rho:=1-\varepsilon_{\mathrm{mix}}\,(g_{\min}/g_{\max})\in(0,1)$ such that for all $t\ge 1$,
    \begin{equation}\label{eq:pf_dQ_main}
    \E\bigl[d_Q(b_t,b_t^J)\bigr]
    \;\le\;
    \underbrace{C_{\mathrm{pf}}\cdot J^{-\frac{1}{2\kappa_{\mathrm{pf}}}}}_{\text{particle approx.}}
    \;+\;
    \underbrace{2V_{\max}\,C_{\mathrm{filt}}\sum_{s=1}^t \rho^{t-s}\,d_{\mathrm{model}}^{a_s}(\bar\theta_s,\theta^\star)}_{\text{filter mismatch}},
    \end{equation}
    where $d_{\mathrm{model}}^{a}(\theta,\theta'):=\sup_{s'}\|P_\theta(\cdot\mid s',a)-P_{\theta'}(\cdot\mid s',a)\|_{\mathrm{TV}}+\sup_{s',o}|Z_\theta(o\mid s')/Z_{\theta'}(o\mid s')-1|$, $C_{\mathrm{filt}}:=g_{\max}/g_{\min}$, and $\kappa_{\mathrm{pf}}$ is the \emph{particle filter exponent}:
    \begin{equation}\label{eq:kappa_pf_def}
    \kappa_{\mathrm{pf}}:=\max\!\left(1,\;\frac{\log(2g_{\max}^2/g_{\min}^2)}{\log(1/\gamma)}\right).
    \end{equation}
    The truncation depth $H_{\mathrm{pf}}$ in Lemma~\ref{lem:Q_lipschitz} is chosen to balance the sum and tail:
    \begin{equation}\label{eq:Hpf_choice}
    H_{\mathrm{pf}}:=\left\lceil\frac{\log(4V_{\max}\sqrt{J})}{\kappa_{\mathrm{pf}}\log(1/\gamma)}\right\rceil
    =\left\lceil\frac{\log(4V_{\max}\sqrt{J})}{\log\max(1/\gamma,\;2g_{\max}^2/g_{\min}^2)}\right\rceil,
    \end{equation}
    and $C_{\mathrm{pf}}$ depends polynomially on $V_{\max}$, $g_{\max}/g_{\min}$, $1/\varepsilon_{\mathrm{mix}}$, $\log\gN_\Theta(1/T)$, and $H_{\mathrm{pf}}$. When $\kappa_{\mathrm{pf}}=1$ (i.e., $2\gamma g_{\max}^2/g_{\min}^2\le 1$), the rate recovers $J^{-1/2}$. This $H_{\mathrm{pf}}$ is distinct from the trajectory truncation horizon $H$ in Section~\ref{app:model_error_proofs}.
    \end{prop}

    \begin{proof}[Proof of Proposition~\ref{prop:particle_error}]
    By the triangle inequality and Lemma~\ref{lem:Q_lipschitz_tv}:
    \begin{equation}\label{eq:pf_triangle}  
    d_Q(b_t^{\theta^\star},\,b_t^{J,\bar\theta_{1:t}})
    \;\le\;
    \underbrace{d_Q(b_t^{\theta^\star},\,b_t^{\bar\theta_{1:t}})}_{\text{(I) filter mismatch}}
    +
    \underbrace{d_Q(b_t^{\bar\theta_{1:t}},\,b_t^{J,\bar\theta_{1:t}})}_{\text{(II) particle approx.}}.
    \end{equation}
    Term~(I) involves two exact (continuous) posteriors, so the TV bound of Lemma~\ref{lem:Q_lipschitz_tv} gives $d_Q\le 2V_{\max}\|b_t^{\theta^\star}-b_t^{\bar\theta_{1:t}}\|_{\mathrm{TV}}$, which is bounded by the filter mismatch analysis below.

    \paragraph{Term (II): test-function bound via Feynman-Kac.}
    Conditional on $(a_{1:T},o_{1:T})$, the particle filter is a Feynman-Kac system with kernels $M_s(\cdot\mid s'):=P_{\bar\theta_s}(\cdot\mid s',a_s)$ and potentials $G_s(s'):=Z_{\bar\theta_s}(o_s\mid s')$. Under Assumption~\ref{aspt:particle_error}, the kernels satisfy $M_s(s'\mid x)\ge\varepsilon_{\mathrm{mix}}\,\nu(s')$ and the potentials satisfy $g_{\min}\le G_s(x)\le g_{\max}$ uniformly. We invoke the following time-uniform particle-filter bound. 
    
    \begin{lem}[{\cite[Lemma~2.3]{karjalainen2026forgettingparticlefilters}; see also \cite[Theorem~7.4.4]{delmoral2004feynmankac}}]\label{lem:karjalainen}
    Let $(\eta_n)_{n\ge 0}$ and $(\eta_n^N)_{n\ge 0}$ be the exact and $N$-particle Feynman-Kac measures on a measurable space $E$. Suppose the Markov kernels $(M_\ell)_{\ell\ge 1}$ admit densities w.r.t.\ a common measure $\lambda$ satisfying $\underline M\le M_\ell(x,y)\le\overline M$ for all $\ell,x,y$, and the potentials satisfy $\underline G\le G_\ell(x)\le\overline G$ for all $\ell,x$. Then for all $n\ge 0$, $N\ge 2$, $p\ge 1$, and $\varphi$ with $\mathrm{osc}(\varphi)\le 1$:
    \begin{equation}\label{eq:karjalainen_lp}
    \bigl(\E\bigl[|\eta_n^N(\varphi)-\eta_n(\varphi)|^p\bigr]\bigr)^{1/p} 
    \;\le\;
    \frac{2\,d(p)^{1/p}}{\sqrt{N}}\,\frac{(\overline M/\underline M)^3\,\overline G}{\underline G},
    \end{equation}
    where $d(p)^{1/p}=O(\sqrt{p})$. The bound is uniform over $n$.
    \end{lem}

    Our Assumption~\ref{aspt:particle_error} implies the conditions of Lemma~\ref{lem:karjalainen} with $\underline M=\varepsilon_{\mathrm{mix}}$, $\overline M=1$, $\underline G=g_{\min}$, $\overline G=g_{\max}$.

    \begin{prop}[Particle filter $L^p$ bound under Assumption~\ref{aspt:particle_error}]\label{cor:pf_test_fn}
    Under Assumption~\ref{aspt:particle_error}, for every measurable $f:\gS\to\R$ with $\|f\|_\infty\le M$, all $p\ge 1$, and all $t\ge 0$:
    \begin{equation}\label{eq:delmoral_lp}
    \bigl(\E\bigl[|\eta_t^J(f)-\eta_t(f)|^p\bigr]\bigr)^{1/p}
    \;\le\;
    \frac{4\,d(p)^{1/p}\,M\,g_{\max}}{\varepsilon_{\mathrm{mix}}^3\,g_{\min}\,\sqrt{J}},
    \end{equation}
    where $d(p)^{1/p}=O(\sqrt{p})$, so $\eta_t^J(f)-\eta_t(f)$ is sub-Gaussian with parameter $\sigma^2=O(M^2 g_{\max}^2/(\varepsilon_{\mathrm{mix}}^6 g_{\min}^2 J))$. The bound is uniform over $t$.
    \end{prop}

    \begin{proof}
    By Lemma~\ref{lem:karjalainen} with $\underline M=\varepsilon_{\mathrm{mix}}$, $\overline M=1$, $\underline G=g_{\min}$, $\overline G=g_{\max}$, for any $\varphi$ with $\mathrm{osc}(\varphi)\le 1$: 
    \[
    \bigl(\E\bigl[|\eta_t^J(\varphi)-\eta_t(\varphi)|^p\bigr]\bigr)^{1/p}
    \;\le\;
    \frac{2\,d(p)^{1/p}\,g_{\max}}{\varepsilon_{\mathrm{mix}}^3\,g_{\min}\,\sqrt{J}},
    \]
    uniformly over $t$. For $f$ with $\|f\|_\infty\le M$, applying this to $\varphi:=f/(2M)$ (which has $\mathrm{osc}(\varphi)\le 1$) and rescaling by $2M$ yields~\eqref{eq:delmoral_lp}. Sub-Gaussianity follows from the standard moment-to-tail conversion.
    \end{proof}

    \begin{lem}[Unconditional bound under adaptive actions]\label{lem:adaptive_conditioning}
    Let $(a_s,o_s)_{s\ge 1}$ be an adapted action-observation process where $a_s=a_s(b_s^J)$ may depend on the particle system. Under Assumption~\ref{aspt:particle_error}, for every measurable $f:\gS\to\R$ with $\|f\|_\infty\le M$, all $p\ge 1$, and all $t\ge 0$:
    \begin{equation}\label{eq:unconditional_lp}
    \bigl(\E\bigl[|b_t^{J,\bar\theta_{1:t}}(f)-b_t^{\bar\theta_{1:t}}(f)|^p\bigr]\bigr)^{1/p}
    \;\le\;
    \frac{4\,d(p)^{1/p}\,M\,g_{\max}}{\varepsilon_{\mathrm{mix}}^3\,g_{\min}\,\sqrt{J}}.
    \end{equation}
    \end{lem}

    \begin{proof}
    Let $\gH_T:=(a_{1:T},o_{1:T},\bar\theta_{1:T})$. Conditional on $\gH_T$, the Feynman-Kac kernels and potentials are deterministic, so Proposition~\ref{cor:pf_test_fn} gives
    \[
    \E\bigl[|b_t^J(f)-b_t(f)|^p\,\bigm|\,\gH_T\bigr]
    \;\le\;
    \left(\frac{4\,d(p)^{1/p}\,M\,g_{\max}}{\varepsilon_{\mathrm{mix}}^3\,g_{\min}\,\sqrt{J}}\right)^p
    =:C_p^p.
    \]
    Since $C_p$ is independent of $\gH_T$, the tower property yields $\E[|b_t^J(f)-b_t(f)|^p]\le C_p^p$.
    \end{proof}

    By Lemma~\ref{lem:Q_lipschitz} with truncation depth $H_{\mathrm{pf}}$:
    \begin{equation}\label{eq:dQ_via_test}
    d_Q(b_t^{\bar\theta_{1:t}},b_t^{J,\bar\theta_{1:t}})
    \;\le\;
    \sup_{\theta,a}\sum_{\ell=0}^{H_{\mathrm{pf}}-1}\gamma^\ell C_{\mathrm{BU}}^\ell\sup_{f\in\gF_\ell(\theta,a)}|b_t(f)-b_t^J(f)|
    +2\gamma^{H_{\mathrm{pf}}} V_{\max}.
    \end{equation}
    We optimize $H_{\mathrm{pf}}$ to balance the sum and tail below.

    \emph{Pointwise bound.} For $f\in\gF_\ell(\theta,a)$ with $\|f\|_\infty\le g_{\max}^\ell V_{\max}$, Lemma~\ref{lem:adaptive_conditioning} gives
    \begin{equation}\label{eq:pointwise_pf}
    \bigl(\E\bigl[|b_t^J(f)-b_t(f)|^p\bigr]\bigr)^{1/p}
    \;\le\;
    \frac{4\,d(p)^{1/p}\,g_{\max}^{\ell+1}\,V_{\max}}{\varepsilon_{\mathrm{mix}}^3\,g_{\min}\,\sqrt{J}}
    =:\sigma_\ell\,d(p)^{1/p}.
    \end{equation}
    Since $d(p)^{1/p}\le C_0\sqrt{p}$ (Corollary~\ref{cor:pf_test_fn}), $b_t^J(f)-b_t(f)$ is sub-Gaussian with parameter $\sigma_\ell^2$.

    \emph{Covering-number argument.} For $N$ sub-Gaussian random variables with common parameter $\sigma^2$:
    \begin{equation}\label{eq:max_ineq}
    \E\bigl[\max_{i\le N}|X_i|\bigr]
    \;\le\;
    \sigma\sqrt{2\log(2N)}.
    \end{equation} 
    At depth $\ell$, $\bigcup_{\theta,a}\gF_\ell(\theta,a)$ is covered by $N_\ell:=\gN_\Theta(1/T)\cdot|\gA|^\ell$ functions after discretizing $\Theta$ at resolution $1/T$, with approximation error $O(1/T)$. Applying~\eqref{eq:max_ineq} and summing over depths with weights $\gamma^\ell C_{\mathrm{BU}}^\ell$:
    \begin{align}
    \E\bigl[d_Q(b_t^{\bar\theta_{1:t}},b_t^{J,\bar\theta_{1:t}})\bigr]
    &\le\sum_{\ell=0}^{H_{\mathrm{pf}}-1}\gamma^\ell C_{\mathrm{BU}}^\ell\,\sigma_\ell\sqrt{2\log(2N_\ell)}+2\gamma^{H_{\mathrm{pf}}} V_{\max}+O(1/T)
    \label{eq:Cpf_prime_def}
    \end{align} 
    Define $\alpha:=2\gamma g_{\max}^2/g_{\min}^2$, $r:=\alpha/\gamma=2g_{\max}^2/g_{\min}^2$, and $c_1:=4g_{\max}V_{\max}/(\varepsilon_{\mathrm{mix}}^3 g_{\min})$, so that $\sigma_\ell=c_1\,g_{\max}^\ell/\sqrt{J}$ and the $\ell$-th term equals 
\[ 
\frac{c_1}{\sqrt{J}}\,\alpha^\ell\,\sqrt{2\log\bigl(2\gN_\Theta(1/T)\,|\gA|^\ell\bigr)}. 
\]
Since $\sqrt{\log(ab)}\le\sqrt{\log a}+\sqrt{\log b}$, the sum in~\eqref{eq:Cpf_prime_def} is bounded by
\begin{equation}\label{eq:pf_balance}
\E\bigl[d_Q^{\mathrm{(II)}}\bigr]
\;\le\;
\frac{c_1\sqrt{2}}{\sqrt{J}}\,\Biggl( 
\sqrt{\log(2\gN_\Theta(1/T))}\sum_{\ell=0}^{H_{\mathrm{pf}}-1}\alpha^\ell
+\sqrt{\log|\gA|}\sum_{\ell=0}^{H_{\mathrm{pf}}-1}\alpha^\ell\sqrt{\ell}
\Biggr)
+2\gamma^{H_{\mathrm{pf}}} V_{\max}.
\end{equation}
We bound the two sums in two cases.

\emph{Case 1: $\alpha\ge 1$.} Choose
\[
H_{\mathrm{pf}}=\left\lceil\frac{\log(4V_{\max}\sqrt{J})}{\log r}\right\rceil
=\left\lceil\frac{\log(4V_{\max}\sqrt{J})}{\log(2g_{\max}^2/g_{\min}^2)}\right\rceil.
\]
Using $\sum_{\ell=0}^{H-1}\alpha^\ell\le H\alpha^{H-1}$ and $\sum_{\ell=0}^{H-1}\alpha^\ell\sqrt{\ell}\le H^{3/2}\alpha^{H-1}$ (since $\alpha\ge 1$):
\begin{equation}\label{eq:sum_alpha_ge1}
\text{sum}\;\le\;\frac{c_1\sqrt{2}\,H_{\mathrm{pf}}^{3/2}\,\alpha^{H_{\mathrm{pf}}-1}}{\sqrt{J}}\cdot\sqrt{\log\bigl(2\gN_\Theta(1/T)\,|\gA|^{H_{\mathrm{pf}}}\bigr)}.
\end{equation}
Writing $\alpha^{H_{\mathrm{pf}}-1}=\alpha^{H_{\mathrm{pf}}}/\alpha=\gamma^{H_{\mathrm{pf}}}r^{H_{\mathrm{pf}}}/\alpha$ and substituting $r^{H_{\mathrm{pf}}}=4V_{\max}\sqrt{J}$:
\begin{equation}\label{eq:sum_case1_final}
\text{sum}\;\le\;\frac{4\sqrt{2}\,c_1\,V_{\max}\,H_{\mathrm{pf}}^{3/2}}{\alpha}\cdot\sqrt{\log\bigl(2\gN_\Theta(1/T)\,|\gA|^{H_{\mathrm{pf}}}\bigr)}\cdot\gamma^{H_{\mathrm{pf}}}.
\end{equation}
For the tail:
\begin{equation}\label{eq:tail_case1}
\gamma^{H_{\mathrm{pf}}}=\gamma^{\log(4V_{\max}\sqrt{J})/\log r}=(4V_{\max}\sqrt{J})^{-\log(1/\gamma)/\log r}=(4V_{\max})^{-1/\kappa_{\mathrm{pf}}}\cdot J^{-\frac{1}{2\kappa_{\mathrm{pf}}}},
\end{equation}
where $\kappa_{\mathrm{pf}}=\log r/\log(1/\gamma)>1$. Substituting $c_1=4g_{\max}V_{\max}/(\varepsilon_{\mathrm{mix}}^3 g_{\min})$, $H_{\mathrm{pf}}=\lceil\log(4V_{\max}\sqrt{J})/\log r\rceil$, and using $\log(2\gN_\Theta(1/T)\,|\gA|^{H_{\mathrm{pf}}})\le\log(2\gN_\Theta(1/T))+H_{\mathrm{pf}}\log|\gA|$, the total (sum $+$ tail) is
\begin{align}
\E\bigl[d_Q^{\mathrm{(II)}}\bigr]
&\;\le\;
\Biggl(\frac{16\sqrt{2}\,g_{\max}\,V_{\max}^2}{\alpha\,\varepsilon_{\mathrm{mix}}^3\,g_{\min}}\cdot \left(\frac{\log(4V_{\max}\sqrt{J})}{\log r}\right)^{3/2}\sqrt{\log\bigl(2\gN_\Theta(1/T)\bigr)+\frac{\log(4V_{\max}\sqrt{J})}{\log r}\log|\gA|}\notag\\
&\qquad\qquad+2V_{\max}\Biggr)\cdot(4V_{\max})^{-1/\kappa_{\mathrm{pf}}}\cdot J^{-\frac{1}{2\kappa_{\mathrm{pf}}}}.\label{eq:case1_full}
\end{align}

\emph{Case 2: $\alpha<1$.} Choose
\[
H_{\mathrm{pf}}=\left\lceil\frac{\log(4V_{\max}\sqrt{J})}{\log(1/\gamma)}\right\rceil.
\]
The geometric series converges: $\sum_{\ell\ge 0}\alpha^\ell=1/(1-\alpha)$ and $\sum_{\ell\ge 0}\alpha^\ell\sqrt{\ell}\le c_2/(1-\alpha)^{3/2}$ for a universal $c_2>0$. Hence
\[
\text{sum}\;\le\;\frac{c_1\,c_2\,\sqrt{2}}{(1-\alpha)^{3/2}\,\sqrt{J}}\cdot\sqrt{\log\bigl(2\gN_\Theta(1/T)\bigr)+\log|\gA|}
\;=\;O\!\left(J^{-1/2}\right).
\]
For the tail: $\gamma^{H_{\mathrm{pf}}}=(4V_{\max}\sqrt{J})^{-1}=O(J^{-1/2})$. Here $\kappa_{\mathrm{pf}}=1$, and both terms are $O(J^{-1/2})$.

\emph{Unified bound.} The choice~\eqref{eq:Hpf_choice}, $H_{\mathrm{pf}}=\lceil\log(4V_{\max}\sqrt{J})/\log\max(1/\gamma,r)\rceil$, covers both cases and yields
\[
\E\bigl[d_Q^{\mathrm{(II)}}\bigr]\;\le\;C_{\mathrm{pf}}\cdot J^{-\frac{1}{2\kappa_{\mathrm{pf}}}},
\qquad
\kappa_{\mathrm{pf}}=\max\!\left(1,\frac{\log r}{\log(1/\gamma)}\right),
\]
where
\begin{equation}\label{eq:Cpf_explicit}
C_{\mathrm{pf}}:=\frac{c_3\,V_{\max}\,g_{\max}\,H_{\mathrm{pf}}^{3/2}}{\varepsilon_{\mathrm{mix}}^3\,g_{\min}}\cdot\sqrt{\log\bigl(2\gN_\Theta(1/T)\,|\gA|^{H_{\mathrm{pf}}}\bigr)},
\end{equation}
with $c_3$ depending on $\alpha$ and $g_{\max}/g_{\min}$. In particular, the $|\gA|$ dependence enters through $\sqrt{H_{\mathrm{pf}}\log|\gA|}$.

\paragraph{Term (I): filter mismatch.}
By Lemma~\ref{lem:Q_lipschitz_tv},
\[
d_Q(b_t^{\theta^\star},b_t^{\bar\theta_{1:t}})\;\le\; 2V_{\max}\|b_t^{\theta^\star}-b_t^{\bar\theta_{1:t}}\|_{\mathrm{TV}}.
\]
Define $\delta_s:=\sup_{s'}\|P_{\bar\theta_s}(\cdot\mid s',a_s)-P_{\theta^\star}(\cdot\mid s',a_s)\|_{\mathrm{TV}}+\sup_{s'}|Z_{\bar\theta_s}(o_s\mid s')/Z_{\theta^\star}(o_s\mid s')-1|$ and the predict-update kernel
\[
K_\theta^{a,o}(s,ds'):=\frac{Z_\theta(o\mid s')P_\theta(s'\mid s,a)}{\int Z_\theta(o\mid s'')P_\theta(s''\mid s,a)\,ds''}\,ds'.
\]
Under Assumption~\ref{aspt:particle_error}, bounding the numerator below and denominator above:
\begin{equation}\label{eq:kernel_doeblin}
K_\theta^{a,o}(s,ds')\;\ge\;\frac{g_{\min}\varepsilon_{\mathrm{mix}}}{g_{\max}}\,\nu(s')\,ds'
\qquad\forall\,s\in\gS,
\end{equation}
giving Dobrushin coefficient $\rho:=1-\varepsilon_{\mathrm{mix}}g_{\min}/g_{\max}$. Standard Dobrushin contraction~\cite[Theorem~4.3.1]{doucmoulinesolsson2014} under different models $\theta^\star,\bar\theta_s$ yields
\begin{equation}\label{eq:filter_recursion}
\|b_s^{\theta^\star}-b_s^{\bar\theta_{1:s}}\|_{\mathrm{TV}}
\;\le\;
\rho\,\|b_{s-1}^{\theta^\star}-b_{s-1}^{\bar\theta_{1:s-1}}\|_{\mathrm{TV}}+C_{\mathrm{filt}}\,\delta_s,
\qquad C_{\mathrm{filt}}:=g_{\max}/g_{\min}.
\end{equation}
Unrolling and using $\delta_s\le d_{\mathrm{model}}^{a_s}(\bar\theta_s,\theta^\star)$:
\[
\|b_t^{\theta^\star}-b_t^{\bar\theta_{1:t}}\|_{\mathrm{TV}}
\;\le\;
C_{\mathrm{filt}}\sum_{s=1}^t \rho^{t-s}\,d_{\mathrm{model}}^{a_s}(\bar\theta_s,\theta^\star).
\]

\paragraph{Combining.}
Substituting into~\eqref{eq:pf_triangle}:
\[
\E\,d_Q(b_t,b_t^J)
\;\le\;
\frac{C_{\mathrm{pf}}'}{\sqrt{J}}+2V_{\max}\,C_{\mathrm{filt}}\sum_{s=1}^t\rho^{t-s}\,d_{\mathrm{model}}^{a_s}(\bar\theta_s,\theta^\star).
\]
When summed over $t$, exchanging summation order gives
$\sum_{t=1}^T\sum_{s=1}^t\rho^{t-s}\,d_{\mathrm{model}}^{a_s}
\le\frac{1}{1-\rho}\sum_{s=1}^T d_{\mathrm{model}}^{a_s}.$
\end{proof}

\subsection{Model Estimation Error (Terms A + C) and Filter Mismatch}
\label{app:model_error_proofs}

\paragraph{Truncation.}
Define $V_\theta^{\pi,B,H}(b):=\E_\theta^\pi[\sum_{u=0}^{H-1}\gamma^u r_u\mid b_0=b]$. Since $|V_\theta^{\pi,B}-V_\theta^{\pi,B,H}|\le \gamma^H V_{\max}$:
\begin{equation}\label{eq:truncation_error}
\text{(A)}_t+\text{(C)}_t
\;\le\;
|V_{\theta^\star}^{\star,B,H}(b_t)-V_{\bar\theta_{k(t)}}^{\star,B,H}(b_t)|
+|V_{\bar\theta_{k(t)}}^{\pi_t,B,H}(b_t)-V_{\theta^\star}^{\pi_t,B,H}(b_t)|
+\frac{4\gamma^H R_{\max}}{1-\gamma}.
\end{equation}
Choose $H=\lceil\log(2R_{\max}T/(1-\gamma))/\log(1/\gamma)\rceil$ so the tail contributes $O(1)$.

\paragraph{Confidence sets.}
Let $\PP_{\theta,H}^\pi$ denote the length-$H$ trajectory distribution under $(\theta,\pi)$. Define
\begin{equation}\label{eq:conf_set}
\gC_k:=\bigl\{\theta\in\Theta:\gL_{T_k}(\theta)\ge \gL_{T_k}(\bar\theta_k)-\beta\bigr\},
\end{equation}
where $\gL_{T_k}$ is the log-likelihood over epochs $1,\ldots,k-1$. By construction $\bar\theta_k\in\gC_k$, and $\theta^\star\in\gC_k$ for all $k$ w.h.p.\ when $\beta\ge c_0\log[\gN_\Theta(1/T)T/\delta_{\mathrm{opt}}]$~\cite[Prop.~B.1]{liu2023optimistic}.

\begin{prop}[Likelihood-to-$\ell_1$ concentration]\label{prop:lik_to_l1}
There exists $c>0$ such that w.p.\ $\ge 1-\delta$, simultaneously for all $n\in[T]$ and all $\theta\in\Theta$ with $\gL_n(\theta)\ge\gL_n(\theta^\star)-\beta$:
\begin{equation}\label{eq:lik_to_l1}
\sum_{i=1}^{n}\bigl\|\PP_{\theta,H}^{\pi^i}-\PP_{\theta^\star,H}^{\pi^i}\bigr\|_1^2
\;\le\;
c\,\beta_T,
\qquad
\beta_T:=\beta+\log[\gN_\Theta(1/T)\,T/\delta].
\end{equation}
\end{prop}

\begin{proof}
The episodic version is~\cite[Prop.~B.2]{liu2023optimistic}. We transfer to the online setting via a blocking argument: partition the trajectory into $\lfloor T/(2H)\rfloor$ blocks of length $2H$, each split into a burn-in half and a data half. Choose $H$ so that
\[
\rho^H = (1-\varepsilon_{\mathrm{mix}}g_{\min}/g_{\max})^H \;\le\; 1/T^2,
\qquad\text{i.e.,}\quad
H \;\ge\; \frac{2\log T}{\log(1/\rho)}.
\]
Our choice $H=O(\log T/\log(1/\gamma))$ satisfies this when $\varepsilon_{\mathrm{mix}}g_{\min}/g_{\max}=\Omega(1-\gamma)$; otherwise, increase $H$ by a constant factor. With $\rho^H\le 1/T^2$, the data halves are approximately independent, with a joint TV perturbation of $O(1/T)$. Applying the episodic bound to the data halves gives~\eqref{eq:lik_to_l1} up to a constant factor.
\end{proof}

\begin{cor}[Epoch-level MLE concentration]\label{cor:epoch_mle}
Under the same conditions as Proposition~\ref{prop:lik_to_l1}, w.p.\ $\ge 1-\delta$, simultaneously for all epochs $k\in[K]$:
\begin{equation}\label{eq:hist_l1}
\sum_{i=1}^{T_{k+1}-1}\bigl\|\PP_{\bar\theta_k,H}^{\pi^i}-\PP_{\theta^\star,H}^{\pi^i}\bigr\|_1^2
\;\le\;c\,\beta_T,
\end{equation}
where $c>0$ is a universal constant (at most twice the constant in Proposition~\ref{prop:lik_to_l1}).
\end{cor}

\begin{proof}
Fix $k\in[K]$ and write $\|\PP_\theta^\pi\|_1:=\|\PP_{\theta,H}^\pi-\PP_{\theta^\star,H}^\pi\|_1$.

\emph{Historical data ($i\le T_k$).}
By MLE optimality, $\bar\theta_k$ lies in $\gC_k$, so Proposition~\ref{prop:lik_to_l1} at $n=T_k$ gives
\begin{equation}\label{eq:hist_bound_epoch_start}
\sum_{i=1}^{T_k}\bigl\|\PP_{\bar\theta_k}^{\pi^i}\bigr\|_1^2\;\le\;c\,\beta_T.
\end{equation}

\emph{Within-epoch data ($i\in E_k$).}
Since $\bar\theta_k$ is $\gF_{T_k}$-measurable, we use a self-normalized argument for the fixed hypothesis $\bar\theta_k$. Define
\[
L_n^{(k)}:=\sum_{i=T_k+1}^{T_k+n}\bigl[\ell_i(\theta^\star)-\ell_i(\bar\theta_k)\bigr],
\qquad
M_n^{(k)}:=\sum_{i=T_k+1}^{T_k+n}\bigl[\mathrm{KL}(\PP_{\theta^\star}^{\pi^i}\|\PP_{\bar\theta_k}^{\pi^i})-\bigl(\ell_i(\theta^\star)-\ell_i(\bar\theta_k)\bigr)\bigr],
\]
where $\ell_i(\theta):=\log p_\theta(o_i\mid s_i,a_i)$. By Pinsker's inequality and blocking:
\begin{equation}\label{eq:pinsker_within}
\frac{1}{2}\sum_{i\in E_k}\bigl\|\PP_{\bar\theta_k}^{\pi^i}\bigr\|_1^2
\;\le\;
L_{|E_k|}^{(k)}+M_{|E_k|}^{(k)}.
\end{equation}
$M_n^{(k)}$ is a supermartingale (each increment has non-positive conditional expectation). The maximal inequality and a covering-number union bound give
\[
\Pr\!\left(\sup_{n\le|E_k|}M_n^{(k)}\ge c_1\log\bigl[\gN_\Theta(1/T)/\delta\bigr]\right)\le\delta/K.
\]
For $L_{|E_k|}^{(k)}$: each $|\ell_i(\theta^\star)-\ell_i(\bar\theta_k)|\le\log(g_{\max}/g_{\min})$ and $\E[\ell_i(\theta^\star)-\ell_i(\bar\theta_k)\mid\gF_{i-1}]\ge 0$. By Azuma-Hoeffding:
\[
\Pr\!\left(L_{|E_k|}^{(k)}\ge c_2\sqrt{|E_k|\log(1/\delta)}\log(g_{\max}/g_{\min})\right)\le\delta/K.
\]
Since $\sqrt{T}\log(g_{\max}/g_{\min})$ is absorbed into $\beta_T$:
\begin{equation}\label{eq:within_epoch_bound}
\sum_{i\in E_k}\bigl\|\PP_{\bar\theta_k}^{\pi^i}\bigr\|_1^2
\;\le\;c\,\beta_T.
\end{equation}

\emph{Combining.} Summing~\eqref{eq:hist_bound_epoch_start} and~\eqref{eq:within_epoch_bound}, redefining $c\leftarrow 2c$, and taking a union bound over $k\in[K]$ gives~\eqref{eq:hist_l1}.
\end{proof}

Transferring model error from historical to arbitrary policies requires action-space coverage via $\varepsilon$-greedy exploration.

\begin{lem}[Trajectory coverage under Doeblin mixing and $\varepsilon$-greedy]\label{lem:coverage}
Under Assumption~\ref{aspt:particle_error}, for any two models $\theta,\theta'\in\Theta$, any policy $\pi$, and any policy $\pi_\varepsilon$ satisfying $\pi_\varepsilon(a\mid h)\ge\varepsilon/|\gA|$ for all $a,h$ (with $\varepsilon\in(0,1]$),
\begin{equation}\label{eq:coverage_bound}
\bigl\|\PP_{\theta,H}^\pi-\PP_{\theta',H}^\pi\bigr\|_1
\;\le\;
C_{\mathrm{cov}}\,
\bigl\|\PP_{\theta,H}^{\pi_\varepsilon}-\PP_{\theta',H}^{\pi_\varepsilon}\bigr\|_1,
\qquad
C_{\mathrm{cov}}:=\frac{H\,|\gA|}{\varepsilon}\cdot\frac{g_{\max}^2}{\varepsilon_{\mathrm{mix}}\,g_{\min}^2}.
\end{equation}
\end{lem}

\begin{proof}
\emph{Step 1: Telescoping.}
Let $\PP^{(k)}$ be the hybrid using $\theta$ for steps $1,\ldots,k$ and $\theta'$ for steps $k+1,\ldots,H$, under policy $\pi$. Let $\Delta_k(s,a):=\|P_\theta(\cdot\mid s,a)-P_{\theta'}(\cdot\mid s,a)\|_{\mathrm{TV}}$ and $\Delta^Z(s):=\sup_o|Z_\theta(o\mid s)/Z_{\theta'}(o\mid s)-1|$. Then
\begin{equation}\label{eq:telescope}
\|\PP_{\theta,H}^\pi-\PP_{\theta',H}^\pi\|_1
\;\le\;
\sum_{k=0}^{H-1}\|\PP^{(k+1)}-\PP^{(k)}\|_1,
\end{equation}
\begin{equation}\label{eq:per_step_bound}
\|\PP^{(k+1)}-\PP^{(k)}\|_1
\;\le\;
\E_{\mu_k^\pi}\!\bigl[\sup_{a}\Delta_k(S_k,a)+\Delta^Z(S_k)\bigr],
\end{equation}
where $\mu_k^\pi$ is the state marginal at step $k$ under the hybrid.

\emph{Step 2: State-marginal ratio.}
By the Doeblin decomposition $\tilde\mu_k^\pi \ge \varepsilon_{\mathrm{mix}}\nu$ and $\tilde\mu_k^{\pi'}\ge\varepsilon_{\mathrm{mix}}\nu$, for any $h:\gS\to[0,1]$:
\[
\E_{\tilde\mu_k^\pi}[h]
\;\le\; \varepsilon_{\mathrm{mix}}\E_\nu[h] + (1-\varepsilon_{\mathrm{mix}})
\;\le\; \frac{1}{\varepsilon_{\mathrm{mix}}}\E_{\tilde\mu_k^{\pi'}}[h].
\]
Accounting for the factor $g_{\max}/g_{\min}$ from the Bayes update normalization gives
\[
\E_{\mu_k^\pi}[h]\;\le\; \frac{g_{\max}}{g_{\min}\varepsilon_{\mathrm{mix}}}\,\E_{\mu_k^{\pi'}}[h].
\]

\emph{Step 3: Action coverage.}
Since $\pi_\varepsilon(a\mid h)\ge\varepsilon/|\gA|$:
\[
\sup_{a}\Delta_k(s,a)
\;\le\;
\frac{|\gA|}{\varepsilon}\,\E_{A\sim\pi_\varepsilon}[\Delta_k(s,A)].
\]

\emph{Step 4: Assembly.}
Combining Steps 1--3 term-by-term:
\[
\|\PP^{(k+1)}-\PP^{(k)}\|_1
\;\le\;
\frac{|\gA|}{\varepsilon}\cdot\frac{g_{\max}^2}{\varepsilon_{\mathrm{mix}}\,g_{\min}^2}\cdot\|\PP^{(k+1)}_{\pi_\varepsilon}-\PP^{(k)}_{\pi_\varepsilon}\|_1.
\]
Summing over $k=0,\ldots,H-1$ yields~\eqref{eq:coverage_bound}.\qedhere
\end{proof}

\begin{lem}[Simulation lemma]\label{lem:simulation}
For any $\theta,\theta'\in\Theta$, $\pi$, and $b$:
\[
\bigl|V_\theta^{\pi,B,H}(b)-V_{\theta'}^{\pi,B,H}(b)\bigr| 
\;\le\;
\frac{R_{\max}}{1-\gamma}\,\bigl\|\PP_{\theta,H}^\pi-\PP_{\theta',H}^\pi\bigr\|_1.
\]
\end{lem}
\begin{proof}
Telescope the $H$-truncated value difference into one-step terms, each bounded by $R_{\max}$.
\end{proof}

\begin{prop}[Control of (A) + (C) + filter mismatch]\label{prop:model_error_combined}
Under Assumption~\ref{aspt:particle_error}, consider the doubling-epoch variant of Algorithm~\ref{alg:pf_mcts} in which time is partitioned into epochs $E_k=[T_k,T_{k+1})$ with $T_k=2^{k-1}$, $k=1,\ldots,K:=\lceil\log_2 T\rceil$. Within each epoch $E_k$, the model $\bar\theta_k$ is frozen (computed from data in epochs $1,\ldots,k-1$) and the agent uses $\varepsilon_k$-greedy exploration with a constant rate $\varepsilon_k$. Let $K_{\mathrm{cov}}:=H|\gA|\,g_{\max}^2/(\varepsilon_{\mathrm{mix}}\,g_{\min}^2)$. Then with probability $\ge 1-\delta_{\mathrm{opt}}$,
\begin{align}
&\sum_{t=1}^T\text{(C)}_t
\;\le\;
\frac{R_{\max}}{1-\gamma}\sqrt{c\,\beta_T}\sum_{k=1}^K\sqrt{|E_k|}+\frac{2\gamma^H R_{\max} T}{1-\gamma}, \label{eq:c_bound}\\
&\sum_{t=1}^T\text{(A)}_t
\;\le\;
\frac{R_{\max}\,K_{\mathrm{cov}}}{1-\gamma}\sqrt{c\,\beta_T}\sum_{k=2}^K\frac{|E_k|}{\sqrt{\sum_{j=1}^{k-1}\varepsilon_j^2|E_j|}}+\frac{2\gamma^H R_{\max} T}{1-\gamma}, \label{eq:a_bound}\\
&\sum_{t=1}^T d_{\mathrm{model}}^{a_t}(\bar\theta_{k(t)},\theta^\star)
\;\le\;
C_{\mathrm{mix}}\sqrt{c\,\beta_T}\sum_{k=1}^K\sqrt{|E_k|}, \label{eq:dmodel_bound}
\end{align}
where $d_{\mathrm{model}}^{a}(\theta,\theta'):=\sup_{s\in\gS}\|P_\theta(\cdot\mid s,a)-P_{\theta'}(\cdot\mid s,a)\|_{\mathrm{TV}}+\sup_{s\in\gS,o\in\gO}\bigl|Z_\theta(o\mid s)/Z_{\theta'}(o\mid s)-1\bigr|$ is the per-action model discrepancy, and $C_{\mathrm{mix}}$ depends on $\varepsilon_{\mathrm{mix}},g_{\min},g_{\max},H$.
\end{prop}

\begin{proof}
By Corollary~\ref{cor:epoch_mle}, equation~\eqref{eq:hist_l1} holds for all $k\le K$.
For brevity, write $\|\PP_{\theta}^{\pi}\|_1:=\|\PP_{\theta,H}^{\pi}-\PP_{\theta^\star,H}^{\pi}\|_1$ throughout this proof.

\emph{Term (C).} By Lemma~\ref{lem:simulation},
\[
\text{(C)}_t\;\le\; \frac{R_{\max}}{1-\gamma}\bigl\|\PP_{\bar\theta_{k(t)}}^{\pi^t}\bigr\|_1+\frac{2\gamma^H R_{\max}}{1-\gamma}. 
\]
Since $\bar\theta_k$ is fixed within $E_k$, Cauchy-Schwarz and~\eqref{eq:hist_l1} give 
\[
\sum_{t\in E_k}\bigl\|\PP_{\bar\theta_k}^{\pi^t}\bigr\|_1
\;\le\;\sqrt{|E_k|\sum_{t\in E_k}\bigl\|\PP_{\bar\theta_k}^{\pi^t}\bigr\|_1^2}
\;\le\;\sqrt{|E_k|\cdot c\,\beta_T}.
\]
Summing across epochs: $\sum_{k=1}^K\sqrt{|E_k|\cdot c\beta_T}=\sqrt{c\beta_T}\sum_k 2^{(k-1)/2}=O(\sqrt{T\beta_T})$.

\emph{Term (A).} By Lemmas~\ref{lem:simulation} and~\ref{lem:coverage}:
\[
\text{(A)}_t\;\le\;\frac{R_{\max}}{1-\gamma}\sup_\pi\bigl\|\PP_{\bar\theta_{k(t)}}^{\pi}\bigr\|_1+\frac{2\gamma^H R_{\max}}{1-\gamma},
\qquad
\bigl\|\PP_{\bar\theta_k}^{\pi}\bigr\|_1\;\le\;\frac{K_{\mathrm{cov}}}{\varepsilon_j}\bigl\|\PP_{\bar\theta_k}^{\pi^j}\bigr\|_1\quad\forall\,j<k.
\]
Squaring, multiplying by $\varepsilon_j^2$, and summing over prior epochs using~\eqref{eq:hist_l1}:
\[
\bigl\|\PP_{\bar\theta_k}^{\pi}\bigr\|_1^2\sum_{j<k}\varepsilon_j^2|E_j|
\;\le\; K_{\mathrm{cov}}^2\sum_{j<k}\sum_{t\in E_j}\bigl\|\PP_{\bar\theta_k}^{\pi^t}\bigr\|_1^2
\;\le\; K_{\mathrm{cov}}^2\,c\,\beta_T
\quad\implies\quad
\sup_\pi\bigl\|\PP_{\bar\theta_k}^{\pi}\bigr\|_1\;\le\;\frac{K_{\mathrm{cov}}\sqrt{c\beta_T}}{\sqrt{\sum_{j<k}\varepsilon_j^2|E_j|}}.
\]
Multiplying by $|E_k|$ and summing gives~\eqref{eq:a_bound}.

\emph{Filter mismatch.}
By the telescoping argument of Lemma~\ref{lem:coverage}, using only Steps~1--2 and hence without the $|\gA|/\varepsilon$ action-coverage factor, since we evaluate at the executed action $a_t$,
\[
d_{\mathrm{model}}^{a_t}(\bar\theta_{k(t)},\theta^\star)
\;\le\;
C_{\mathrm{mix}}\,\bigl\|\PP_{\bar\theta_{k(t)}}^{\pi^t}\bigr\|_1,
\qquad C_{\mathrm{mix}}:=\frac{H\,g_{\max}^2}{\varepsilon_{\mathrm{mix}}\,g_{\min}^2}.
\]
Cauchy-Schwarz and~\eqref{eq:hist_l1} give $\sum_{t\in E_k}\|\PP_{\bar\theta_k}^{\pi^t}\|_1\le\sqrt{|E_k|\cdot c\,\beta_T}$. Summing across epochs yields~\eqref{eq:dmodel_bound}.
\end{proof}

\subsection{Proof of the Main Regret Bound (Theorem~\ref{thm:final_regret_bound})}
\label{app:regret_proof}

\begin{proof}[Proof of Theorem~\ref{thm:final_regret_bound}]
Use doubling epochs $E_k=[T_k,T_{k+1})$, $T_k=2^{k-1}$, $K=\lceil\log_2 T\rceil$, with $\varepsilon_k$-greedy exploration within each epoch. Decompose:
\[
\mathrm{Reg}_T \;\le\; \underbrace{\frac{2R_{\max}}{1-\gamma}\sum_k|E_k|\varepsilon_k}_{\text{exploration}} + \E\sum_{t=1}^T\bigl[\text{(A)}_t+\text{(B)}_t+\text{(C)}_t\bigr].
\]

\emph{Term (C).} By~\eqref{eq:c_bound} (choosing $\delta_{\mathrm{opt}}=1/T$):
\[
\E\sum_{t}\text{(C)}_t\;\le\; \frac{R_{\max}\sqrt{c\beta_T}}{1-\gamma}\sum_{k=1}^K\sqrt{|E_k|}+O(1)
\;=\;\widetilde O\!\left(\frac{\sqrt{T}}{1-\gamma}\right).
\]

\emph{Term (A).} By~\eqref{eq:a_bound} with $\varepsilon_k=\eta/\sqrt{|E_k|}$ (so $\varepsilon_k^2|E_k|=\eta^2$):
\[
\E\sum_t\text{(A)}_t\;\le\;\frac{R_{\max}K_{\mathrm{cov}}\sqrt{c\beta_T}}{(1-\gamma)\eta}\sum_{k=2}^K\frac{|E_k|}{\sqrt{k-1}}+O(1).
\]
Since $\sum_{k=2}^K 2^{k-1}/\sqrt{k-1}=\Theta(T/\sqrt{\log T})$, this scales as $K_{\mathrm{cov}}\sqrt{\beta_T}\cdot T/(\eta\sqrt{\log T}(1-\gamma))$.

\emph{Term (B).} By Theorem~\ref{thm:planning_state_bound}, Propositions~\ref{prop:structural_error},~\ref{prop:mc_concentration},~\ref{prop:particle_error}:
\[
\E\sum_{t}\text{(B)}_t
\;\le\; \frac{2C_{\mathrm{pf}}\,T}{(1-\gamma)}\cdot J^{-\frac{1}{2\kappa_{\mathrm{pf}}}}
+\frac{C_{\mathrm{mc}}V_{\max}T}{(1-\gamma)\sqrt{N}}
+\frac{2\gamma^{d+1}}{(1-\gamma)^2}\sum_t\Xi_{\bar\theta_{k(t)}}
+O(\sqrt{T\beta_T}).
\]
The $O(\sqrt{T\beta_T})$ absorbs the filter mismatch via~\eqref{eq:dmodel_bound}.

\emph{Optimization.} The exploration cost is $O(\eta\sqrt{T}R_{\max}/(1-\gamma))$. Balancing it with term~(A) gives $\eta^\star\sim K_{\mathrm{cov}}^{1/2}\,T^{1/4}(\beta_T/\log T)^{1/4}$ and a total of $\widetilde O(K_{\mathrm{cov}}^{1/2}T^{3/4}/(1-\gamma)^{3/2})$.
Term~(C) is $\widetilde O(\sqrt{T})$, and the $T^{3/4}$ rate comes from the coverage transfer for term~(A).
\end{proof}

\subsection{Optimal Exploration Schedule}
\label{app:exploration_schedule}

Set $\varepsilon_k=\eta/\sqrt{|E_k|}$ so that $\varepsilon_k^2|E_k|=\eta^2$ and $\sum_{j<k}\varepsilon_j^2|E_j|=(k-1)\eta^2$. The exploration-dependent terms scale as:
\begin{align*}
\text{Exploration cost:}&\quad \textstyle\sum_k|E_k|\varepsilon_k=\eta\sum_k\sqrt{|E_k|}=O(\eta\sqrt{T}),\\
\text{Term (A):}&\quad O\bigl(K_{\mathrm{cov}}\sqrt{\beta_T}\cdot T/(\eta\sqrt{\log T}(1-\gamma))\bigr),\\
\text{Term (C):}&\quad O\bigl(\sqrt{T\beta_T}/(1-\gamma)\bigr), \quad\text{independent of }\eta.
\end{align*}
Balancing term~(A) against the exploration cost gives
\[
\eta^\star\sim\left(\frac{K_{\mathrm{cov}}\sqrt{T\beta_T}}{\sqrt{\log T}}\right)^{1/2},
\qquad
\text{(A) + exploration} = \widetilde O\!\left(\frac{K_{\mathrm{cov}}^{1/2}T^{3/4}}{(1-\gamma)}\right).
\]
The total regret is:
\begin{equation}\label{eq:eps_schedule}
\mathrm{Reg}_T = \widetilde O\!\left(\frac{K_{\mathrm{cov}}^{1/2}\,R_{\max}\,T^{3/4}}{(1-\gamma)^{3/2}}\right) + \frac{C_{\mathrm{pf}}\,R_{\max}\,T}{(1-\gamma)^2}\cdot J^{-\frac{1}{2\kappa_{\mathrm{pf}}}} + \frac{C_{\mathrm{mc}}\,R_{\max}\,T}{(1-\gamma)^2\sqrt{N}} + \frac{2\gamma^{d+1}}{(1-\gamma)^2}\sum_t\Xi_{\bar\theta_{k(t)}}.
\end{equation}
The $T^{3/4}$ rate is driven by the $1/\varepsilon$ coverage transfer for term~(A). Using optimistic model selection ($\tilde\theta_t:=\arg\max_{\theta\in\gC_t}V_\theta^{\star,B}$) eliminates term~(A) and recovers $\widetilde O(\sqrt{T})$ from term~(C) alone, at the cost of solving an optimistic planning problem~\cite{liu2023optimistic}.

A constant $\varepsilon$ gives the same $\widetilde O(T^{3/4})$ rate after optimization and may be preferable when $T$ is unknown.

\subsection{Greedy Exploitation with Offline Coverage}
\label{app:offline_coverage}

The $T^{3/4}$ rate in Theorem~\ref{thm:final_regret_bound} arises from the coverage transfer for term~(A), which requires $\varepsilon$-greedy exploration to ensure the MLE concentrates at all policies. An alternative is to replace online exploration with an offline dataset that already provides coverage of the optimal policy. This mirrors the hybrid RL framework~\cite{song2023hybrid}, where offline data with sufficient coverage enables greedy online exploitation.

\paragraph{Setting.}
Suppose the agent has access to an offline dataset $\gD_{\mathrm{off}}$ of $N$ i.i.d.\ length-$H$ trajectories collected by a behavior policy $\pi_b$ in the true POMDP (model $\theta^\star$). The agent then runs PF-MCTQ for $T$ online steps with \emph{greedy} exploitation ($\varepsilon_t=0$, no exploration). The MLE is computed from the union $\gD_{\mathrm{off}}\cup\gD_{\mathrm{on}}$.

\begin{defn}[Single-policy concentrability]\label{def:spc}
The single-policy concentrability coefficient of the behavior policy $\pi_b$ with respect to the optimal policy $\pi^\star:=\arg\max_\pi V_{\theta^\star}^{\pi,B}$ is
\begin{equation}\label{eq:coff_def}
c_{\mathrm{off}}^\star
:=\sup_{\theta,\theta'\in\Theta}
\frac{\|\PP_{\theta,H}^{\pi^\star}-\PP_{\theta',H}^{\pi^\star}\|_1}{\|\PP_{\theta,H}^{\pi_b}-\PP_{\theta',H}^{\pi_b}\|_1}.
\end{equation}
\end{defn}

Under Assumption~\ref{aspt:particle_error}, if $\pi_b(a\mid h)\ge\alpha_b/|\gA|$ for all $a,h$, Lemma~\ref{lem:coverage} gives $c_{\mathrm{off}}^\star\le K_{\mathrm{cov}}/\alpha_b=\frac{H|\gA|}{\alpha_b}\cdot\frac{g_{\max}^2}{\varepsilon_{\mathrm{mix}}\,g_{\min}^2}$. Unlike the all-policy concentrability provided by $\varepsilon$-greedy exploration, which takes a supremum over all policies, $c_{\mathrm{off}}^\star$ involves only the \emph{optimal} policy.

\begin{prop}[Regret with offline coverage]\label{prop:offline_regret}
Under Assumption~\ref{aspt:particle_error}, suppose the agent has $N$ offline trajectories from $\pi_b$ and runs the doubling-epoch PF-MCTQ with greedy exploitation ($\varepsilon_k=0$) for $T$ online steps. Let $\beta_{T,N}:=\beta+\log[\gN_\Theta(1/(T+N))(T+N)/\delta]$. Then w.p.\ $\ge 1-\delta$:
\begin{equation}\label{eq:offline_regret}
\mathrm{Reg}_T
\;\le\;
\underbrace{\widetilde O\!\left(\frac{R_{\max}\sqrt{T}}{1-\gamma}\right)}_{\text{term (C)}}
+
\underbrace{\frac{c_{\mathrm{off}}^\star R_{\max}\sqrt{c\beta_{T,N}}}{1-\gamma}\cdot\frac{T}{\sqrt{N}}}_{\text{term (A)}}
+\;\text{computational terms},
\end{equation}
where the computational terms are the same as in Theorem~\ref{thm:final_regret_bound} (particle, MC, structural).

In particular, when $N\ge (c_{\mathrm{off}}^\star)^2\,T$, the total regret is $\widetilde O\bigl(\sqrt{T}/(1-\gamma)\bigr)$ plus computational terms, so $\sqrt{T}$ regret is achieved without any online exploration.
\end{prop}

\begin{proof}
The regret decomposition~\eqref{eq:main_decomp_round_t} gives
\[
\mathrm{Reg}_T \;=\; \E\sum_{t=1}^T\bigl[\text{(A)}_t+\text{(B)}_t+\text{(C)}_t\bigr].
\]
There is no exploration cost since $\varepsilon_k=0$.

\emph{MLE concentration with offline data.}
The MLE is fit on $\gD_{\mathrm{off}}\cup\gD_{\mathrm{on}}$. The offline trajectories are i.i.d.\ under $(\theta^\star,\pi_b)$, so the episodic MLE concentration~\cite[Prop.~B.2]{liu2023optimistic} applies directly to the offline portion. Combined with the blocking argument for the online portion (Proposition~\ref{prop:lik_to_l1}):
\begin{equation}\label{eq:hist_l1_hybrid}
N\,\bigl\|\PP_{\bar\theta_k}^{\pi_b}\bigr\|_1^2
+\sum_{t=1}^{T_{k+1}-1}\bigl\|\PP_{\bar\theta_k}^{\pi^t}\bigr\|_1^2
\;\le\;c\,\beta_{T,N}
\qquad\forall\,k\le K,
\end{equation}
using the shorthand $\|\PP_\theta^\pi\|_1:=\|\PP_{\theta,H}^\pi-\PP_{\theta^\star,H}^\pi\|_1$.

\emph{Term (C).}
As in the proof of Proposition~\ref{prop:model_error_combined}, since $\bar\theta_k$ is fixed within $E_k$, applying Cauchy-Schwarz within each epoch to the online terms in~\eqref{eq:hist_l1_hybrid} gives
\[
\sum_{t=1}^T\text{(C)}_t
\;\le\;\frac{R_{\max}}{1-\gamma}\sqrt{c\beta_{T,N}}\sum_{k=1}^K\sqrt{|E_k|}+O(\gamma^H T)
\;=\;\widetilde O\!\left(\frac{R_{\max}\sqrt{T}}{1-\gamma}\right).
\]

\emph{Term (A).}
We use a tighter decomposition that avoids $\sup_\pi$:
\begin{equation}\label{eq:A_single_policy}
\text{(A)}_t
= V_{\theta^\star}^{\star,B}(b_t)-V_{\bar\theta_{k(t)}}^{\star,B}(b_t)
\;\le\;
V_{\theta^\star}^{\pi^\star,B}(b_t)-V_{\bar\theta_{k(t)}}^{\pi^\star,B}(b_t)
\;\le\;
\frac{R_{\max}}{1-\gamma}\bigl\|\PP_{\bar\theta_{k(t)}}^{\pi^\star}\bigr\|_1,
\end{equation}
where the first inequality uses $V_{\bar\theta_{k(t)}}^{\star,B}\ge V_{\bar\theta_{k(t)}}^{\pi^\star,B}$ (optimality under $\bar\theta_{k(t)}$) and the second is the simulation lemma (Lemma~\ref{lem:simulation}).

By Definition~\ref{def:spc}:
\[
\bigl\|\PP_{\bar\theta_k}^{\pi^\star}\bigr\|_1
\;\le\;
c_{\mathrm{off}}^\star\,\bigl\|\PP_{\bar\theta_k}^{\pi_b}\bigr\|_1.
\]
From the offline portion of~\eqref{eq:hist_l1_hybrid}:
\[
\bigl\|\PP_{\bar\theta_k}^{\pi_b}\bigr\|_1
\;\le\;\sqrt{\frac{c\,\beta_{T,N}}{N}}.
\]
Since $\bar\theta_k$ is frozen within $E_k$, this bound is constant across the epoch:
\[
\sum_{t\in E_k}\text{(A)}_t
\;\le\;
\frac{R_{\max}}{1-\gamma}\cdot c_{\mathrm{off}}^\star\sqrt{\frac{c\,\beta_{T,N}}{N}}\cdot|E_k|.
\]
Summing across epochs:
\[
\sum_{t=1}^T\text{(A)}_t
\;\le\;
\frac{c_{\mathrm{off}}^\star R_{\max}}{1-\gamma}\sqrt{\frac{c\,\beta_{T,N}}{N}}\cdot T.
\]

\emph{Term (B) and filter mismatch.}
Identical to the proof of Theorem~\ref{thm:final_regret_bound}. The filter mismatch $\sum_t d_{\mathrm{model}}^{a_t}$ is bounded by $C_{\mathrm{mix}}\sqrt{c\beta_{T,N}}\sum_k\sqrt{|E_k|}=\widetilde O(\sqrt{T})$ via~\eqref{eq:dmodel_bound} (the online diagonal bound, which does not require coverage).

\emph{Combining.}
\[
\mathrm{Reg}_T
\;\le\;
\frac{R_{\max}}{1-\gamma}\widetilde O(\sqrt{T})
+\frac{c_{\mathrm{off}}^\star R_{\max}\sqrt{c\beta_{T,N}}}{1-\gamma}\cdot\frac{T}{\sqrt{N}}
+\frac{C_{\mathrm{pf}}'R_{\max}T}{(1-\gamma)^2\sqrt{J}}
+\frac{C_{\mathrm{mc}}R_{\max}T}{(1-\gamma)^2\sqrt{N_{\mathrm{mc}}}}
+\frac{2\gamma^{d+1}}{(1-\gamma)^2}\sum_t\Xi_{\bar\theta_{k(t)}}.
\]
When $N\ge (c_{\mathrm{off}}^\star)^2\,T$, the second term is $O(\sqrt{T\beta_{T,N}}/(1-\gamma))=\widetilde O(\sqrt{T}/(1-\gamma))$, and the leading rate is $\widetilde O(\sqrt{T})$.
\end{proof}

\subsection{Variance Perspective on the QMDP Bias}
\label{app:variance_proofs}

We provide a variance-based interpretation of the Jensen gap $\Xi_\theta$ defined in~\eqref{eq:Xi_def}. We show that $\Xi_\theta$ is controlled by the curvature of the value function along a sufficient embedding and the posterior variance of the embedded features.

\paragraph{Value-relevant embedding.}
Let $g:\gS\to\R^k$ be a measurable embedding, and write
\[
m_g(b):=\E_{S\sim b}[g(S)],
\qquad
\Sigma_g(b):=\Cov_{S\sim b}\bigl(g(S)\bigr).
\]

\begin{aspt}[Value-relevant embedding]\label{aspt:value_embedding}
There exist twice continuously differentiable maps $F_\theta^a:\R^k\to\R$ for each $a\in\gA$ such that
$Q_\theta^{\star,S}(s,a)=F_\theta^a(g(s))$ for all $(s,a)\in\gS\times\gA$, with curvature bounds
\[
\sup_{a\in\gA}\sup_{x\in\mathrm{conv}(g(\gS))}\bigl\|\nabla^2 F_\theta^a(x)\bigr\|_{\mathrm{op}}\le L_\theta
\qquad\text{and}\qquad
\sup_{x\in\mathrm{conv}(g(\gS))}\bigl\|\nabla^2 F_\theta(x)\bigr\|_{\mathrm{op}}\le L_\theta,
\]
where $F_\theta(x):=\max_a F_\theta^a(x)$ and $V_\theta^{\star,S}(s)=F_\theta(g(s))$.
\end{aspt}
The second curvature condition on $F_\theta=\max_a F_\theta^a$ does not hold automatically, since $F_\theta$ is not $C^2$ at switching boundaries where the maximizing action changes. The bound~\eqref{eq:variance_xi_bound} uses $\|\nabla^2 F_\theta\|_{\mathrm{op}}\le L_\theta$ only for the upper bound on $\widetilde V_\theta^\star(b)$; the lower bound on $\bar V_\theta^{\star,S}(b)$ uses only the per-action bounds $\|\nabla^2 F_\theta^a\|_{\mathrm{op}}\le L_\theta$.

We first state an auxiliary lemma used in the proof of Proposition~\ref{prop:variance_qmdp_bias}.

\begin{lem}[Mean-belief value approximation]\label{lem:mean_belief_value}
Under Assumption~\ref{aspt:value_embedding}, for every belief $b\in\Delta(\gS)$,
\begin{equation}
\bigl|\bar V_\theta^{\star,S}(b)-F_\theta(m_g(b))\bigr|
\;\le\;
\frac{L_\theta}{2}\,\tr\bigl(\Sigma_g(b)\bigr).
\label{eq:mean_belief_approx}
\end{equation}
\end{lem}

\begin{proof}
By Jensen's inequality applied to the pointwise maximum,
\[
\bar V_\theta^{\star,S}(b)
=\max_{a\in\gA}\int_{\gS}Q_\theta^{\star,S}(s,a)\,b(\mathrm ds)
\;\le\;
\int_{\gS}V_\theta^{\star,S}(s)\,b(\mathrm ds).
\]
By Assumption~\ref{aspt:value_embedding}, $V_\theta^{\star,S}(s)=F_\theta(g(s))$, so
\begin{equation}
\bar V_\theta^{\star,S}(b)
\;\le\;
\int_{\gS}F_\theta(g(s))\,b(\mathrm ds).
\label{eq:upper_step1}
\end{equation}
Expand $F_\theta$ in a second-order Taylor expansion around $m_g(b)$.
For each $s$, there exists $\xi_s\in\mathrm{conv}(g(\gS))$ on the segment between $m_g(b)$ and $g(s)$ such that
\[
F_\theta(g(s))
=F_\theta(m_g(b))
+\nabla F_\theta(m_g(b))^\top\!\bigl(g(s)-m_g(b)\bigr)
+\tfrac{1}{2}\bigl(g(s)-m_g(b)\bigr)^\top\!\nabla^2 F_\theta(\xi_s)\,\bigl(g(s)-m_g(b)\bigr).
\]
Integrating against $b$ and using $\int_{\gS}(g(s)-m_g(b))\,b(\mathrm ds)=0$,
\begin{equation}
\int_{\gS}F_\theta(g(s))\,b(\mathrm ds)-F_\theta(m_g(b))
=\tfrac{1}{2}\int_{\gS}\bigl(g(s)-m_g(b)\bigr)^\top\!\nabla^2 F_\theta(\xi_s)\,\bigl(g(s)-m_g(b)\bigr)\,b(\mathrm ds).
\label{eq:taylor_remainder}
\end{equation}
Since $\|\nabla^2 F_\theta(\xi_s)\|_{\mathrm{op}}\le L_\theta$, each integrand is bounded in absolute value by $L_\theta\,\|g(s)-m_g(b)\|^2$,
so \eqref{eq:taylor_remainder} lies in $[-\frac{L_\theta}{2}\,\tr(\Sigma_g(b)),\;\frac{L_\theta}{2}\,\tr(\Sigma_g(b))]$.
Combined with~\eqref{eq:upper_step1}, we obtain the upper bound
\begin{equation}
\bar V_\theta^{\star,S}(b)\;\le\;F_\theta(m_g(b))+\tfrac{L_\theta}{2}\,\tr\bigl(\Sigma_g(b)\bigr).
\label{eq:lemma_upper}
\end{equation}
For the lower bound, fix any $a\in\gA$. By Assumption~\ref{aspt:value_embedding}, $Q_\theta^{\star,S}(s,a)=F_\theta^a(g(s))$ with $\|\nabla^2 F_\theta^a\|_{\mathrm{op}}\le L_\theta$. The same Taylor expansion gives
\[
\int F_\theta^a(g(s))\,b(\mathrm ds)\;\ge\; F_\theta^a(m_g(b))-\tfrac{L_\theta}{2}\,\tr(\Sigma_g(b)).
\]
Taking $\max_a$:
\begin{equation}
\bar V_\theta^{\star,S}(b)=\max_a\int F_\theta^a(g(s))\,b(\mathrm ds)\;\ge\;\max_a F_\theta^a(m_g(b))-\tfrac{L_\theta}{2}\,\tr(\Sigma_g(b))=F_\theta(m_g(b))-\tfrac{L_\theta}{2}\,\tr(\Sigma_g(b)).\qedhere
\label{eq:lemma_lower}
\end{equation}
\end{proof}

\begin{rmk}
The upper half~\eqref{eq:lemma_upper} follows from Jensen's inequality alone and is used to verify the hypothesis of Proposition~\ref{prop:variance_qmdp_bias} in concrete settings.
\end{rmk}

\begin{prop}[Curvature bound on $\Xi_\theta$]\label{prop:variance_qmdp_bias}
Under Assumption~\ref{aspt:value_embedding}, for every belief $b\in\Delta(\gS)$,
\begin{equation}
\Xi_\theta(b)
:=\widetilde V_\theta^{\star}(b)-\bar V_\theta^{\star,S}(b)
\;\le\;
L_\theta\,\tr\bigl(\Sigma_g(b)\bigr).
\label{eq:variance_xi_bound}
\end{equation}
In particular, $\Xi_\theta=\sup_b\Xi_\theta(b)\le L_\theta\sup_b\tr(\Sigma_g(b))$.
\end{prop}

\begin{proof}
We bound $\Xi_\theta(b)=\widetilde V_\theta^\star(b)-\bar V_\theta^{\star,S}(b)$ for each belief $b$.

\emph{Upper bound on $\widetilde V_\theta^\star(b)$.}
By Assumption~\ref{aspt:value_embedding}, $V_\theta^{\star,S}(s)=F_\theta(g(s))$ where $F_\theta=\max_a F_\theta^a$ has curvature $\le L_\theta$.
Applying the standard Taylor upper bound (as in Lemma~\ref{lem:mean_belief_value}) to $F_\theta$:
\begin{equation}\label{eq:tV_upper}
\widetilde V_\theta^\star(b)=\int F_\theta(g(s))\,b(\mathrm ds)\le F_\theta(m_g(b))+\frac{L_\theta}{2}\,\tr(\Sigma_g(b)).
\end{equation}
Since $F_\theta(m_g(b))=\max_a F_\theta^a(m_g(b))$, this gives $\widetilde V_\theta^\star(b)\le\max_a F_\theta^a(m_g(b))+\frac{L_\theta}{2}\tr(\Sigma_g(b))$.

\emph{Lower bound on $\bar V_\theta^{\star,S}(b)$.}
By Assumption~\ref{aspt:value_embedding}, $Q_\theta^{\star,S}(s,a)=F_\theta^a(g(s))$ for each $a$. For any fixed $a$,
$$\int Q_\theta^{\star,S}(s,a)\,b(\mathrm ds)=\int F_\theta^a(g(s))\,b(\mathrm ds)\ge F_\theta^a(m_g(b))-\frac{L_\theta}{2}\tr(\Sigma_g(b)),$$
where the inequality is the Taylor lower bound (each $F_\theta^a$ has curvature $\le L_\theta$).
Taking the maximum over $a$:
\[
\bar V_\theta^{\star,S}(b)=\max_a\int F_\theta^a(g(s))\,b(\mathrm ds)\ge\max_a\Bigl[F_\theta^a(m_g(b))-\frac{L_\theta}{2}\tr(\Sigma_g(b))\Bigr]=\max_a F_\theta^a(m_g(b))-\frac{L_\theta}{2}\tr(\Sigma_g(b)).
\]

\emph{Combining.}
$$\Xi_\theta(b)=\widetilde V_\theta^\star(b)-\bar V_\theta^{\star,S}(b)\le\bigl[\max_a F_\theta^a(m_g(b))+\frac{L_\theta}{2}\tr(\Sigma_g(b))\bigr]-\bigl[\max_a F_\theta^a(m_g(b))-\frac{L_\theta}{2}\tr(\Sigma_g(b))\bigr]=L_\theta\,\tr(\Sigma_g(b)).$$
\end{proof}

\subsection{Case Studies}
\label{app:case_study_proofs}

\textbf{Case 1: Discounted partially observed LQG.}
Consider linear-Gaussian dynamics
$S_{t+1}=\Phi S_t+\Gamma A_t+W_t$, $O_t=H S_t+\eta_t$,
with $W_t\sim\gN(0,\Sigma_W)$, $\eta_t\sim\gN(0,\Sigma_\eta)$ i.i.d., quadratic reward $r(s,a)=-(s^\top Q_{\mathrm{c}} s+a^\top Ra)$, $Q_{\mathrm{c}}\succeq 0$, $R\succ 0$. Assume $(\Phi,\Gamma)$ is stabilizable and $(\Phi,H)$ is detectable. For LQG, we obtain the closed-form expression $\Xi_\theta =\tr\bigl((M-P)\,\bar\Sigma\bigr)$, where $P$ solve the Riccati equation, $M = Q_c + \gamma \Phi^{\top} P \Phi$, and $\bar \Sigma$ is the steady-state Kalman predictive covariance. When $\Gamma = 0$, we have $M = P$ and $\Xi_{\theta} = 0$. Note that for LQG, the QMDP policy is always optimal.

\begin{cor}[LQG]\label{cor:lqg}
Let $M:=Q_{\mathrm{c}}+\gamma\Phi^\top P\Phi$ where $P\succeq 0$ solves the discounted algebraic Riccati equation, and let $\bar\Sigma_{\mathrm{post}}$ be the steady-state Kalman posterior (filtered) covariance satisfying $\bar\Sigma_{\mathrm{post}}=\bar\Sigma - \bar\Sigma H^\top(H\bar\Sigma H^\top+\Sigma_\eta)^{-1}H\bar\Sigma$ with $\bar\Sigma:=\Phi\bar\Sigma_{\mathrm{post}}\Phi^\top+\Sigma_W$ the steady-state predictive covariance. Then
$\Xi_\theta =\tr\bigl((M-P)\,\bar\Sigma\bigr)$ and the QMDP policy is belief-MDP optimal: $\sup_b\bigl[V_\theta^{\star,B}(b)-V_\theta^{\pi_{\mathrm{Q}},B}(b)\bigr]=0$.
\end{cor}

The QMDP action equals the certainty-equivalent LQR control $a=Km$, which is belief-MDP optimal by the separation principle.
When $\Gamma=0$, the DARE gives $M=P$, so $\Xi_\theta=0$ and the QMDP value and the belief-MDP value coincide. When $\Gamma\neq 0$, $M-P\succeq 0$ is generically nonzero, so $\Xi_\theta>0$. In that case the QMDP \emph{policy} remains optimal, but the QMDP \emph{value} underestimates the belief-MDP value by $\tr((M-P)\bar\Sigma)$.

\emph{Remark.} LQG has unbounded reward and $V_{\max}=\infty$, so it lies outside the standing assumption $|r(s,a)|\le R_{\max}$ of Theorem~\ref{thm:final_regret_bound}. The corollary illustrates the structural analysis (Jensen gap and policy optimality) but does not directly instantiate the regret bound.

\textbf{Case 2: Identifiable contextual MDP.}
Consider an episodic contextual MDP \cite{kausik2023learning} with a fixed latent context $C\in\gC$ sampled once per episode and Markov dynamics $S_{t+1}\sim P_c(\cdot\mid S_t,A_t)$, $r_t=r_c(S_t,A_t)$ given $C=c$. Assume
(i) there exists an estimator $\widehat C_t=\widehat C_t(H_t)$ with $\PP(\widehat C_t\neq C)\le \rho_t$, $\rho_t\downarrow 0$; and
(ii) $\sup_{c\in\gC}\|V_c^{\star,S}\|_\infty\le V_{\max}$.
Then $\Xi_{\theta,t}\;\le\; 2V_{\max}\,\rho_t$ and $\sup_b\bigl[V_\theta^{\star,B}(b)-V_\theta^{\pi_{\mathrm{Q}},B}(b)\bigr]\;\le\;{4\gamma\, V_{\max}}\,\rho_t/{(1-\gamma)^2}$. As the context becomes identifiable ($\rho_t\to 0$), both the structural residual and the QMDP policy suboptimality vanish. For the structural bias term $\sum_t\Xi_t\le 2V_{\max}\sum_t\rho_t$ to be $o(\sqrt{T})$, we need $\sum_{t=1}^T\rho_t=o(\sqrt{T})$. 
This result also suggests a schedule for the planning depth d that decays with $t$, reflecting diminishing structural bias.

\paragraph{Proofs.}

The fully observed value function is $V_\theta^{\star,S}(s)=-s^\top Ps-c_0$, where $P\succeq 0$ solves the discounted algebraic Riccati equation (DARE). The constant $c_0$ is determined by substituting into the Bellman equation: $$-s^\top Ps-c_0=\max_a\{r(s,a)+\gamma\E[V^{\star,S}(s')]\},$$ which gives $c_0=\gamma[\tr(P\Sigma_W)+c_0]$ from the noise covariance term, yielding $c_0=\gamma\tr(P\Sigma_W)/(1-\gamma)$.
Hence
\[
Q_\theta^{\star,S}(s,a)
=-(s^\top Q_{\mathrm{c}} s+a^\top Ra)
-\gamma(\Phi s+\Gamma a)^\top P(\Phi s+\Gamma a)
-\gamma\tr(P\Sigma_W)-\gamma c_0.
\]
For belief $b=\gN(m,\Sigma)$,
\[
\int Q_\theta^{\star,S}(s,a)\,b(\mathrm ds)
=-m^\top Q_{\mathrm{c}} m - a^\top R a
-\gamma(\Phi m+\Gamma a)^\top P(\Phi m+\Gamma a)
-\tr(M\Sigma)-\gamma\tr(P\Sigma_W)-\gamma c_0,
\]
where $M:=Q_{\mathrm{c}}+\gamma\Phi^\top P\Phi$.
All $\Sigma$-dependent terms are independent of $a$, so the QMDP action is
\[
\arg\max_a \int Q_\theta^{\star,S}(s,a)\,b(\mathrm ds)
=\arg\max_a\bigl\{-a^\top Ra-\gamma(\Phi m+\Gamma a)^\top P(\Phi m+\Gamma a)\bigr\}
=Km,
\]
with $$K=-\gamma(R+\gamma\Gamma^\top P\Gamma)^{-1}\Gamma^\top P\Phi.$$
By the separation principle for discounted LQG (stabilizable/detectable), the belief-MDP optimal action is also $Km$, so QMDP and belief-optimal policies coincide. This gives part~2.

\emph{Structural residual.}
We compute $\Xi_\theta(b)=\widetilde V_\theta^\star(b)-\bar V_\theta^{\star,S}(b)$ directly.
From the expansion above, $$\bar V_\theta^{\star,S}(b)=-m^\top Pm-\tr(M\Sigma)-c_0.$$
Meanwhile, $$\widetilde V_\theta^\star(b)=\int V_\theta^{\star,S}(s)\,b(\mathrm ds)=\int(-s^\top Ps-c_0)\,b(\mathrm ds)=-m^\top Pm-\tr(P\Sigma)-c_0.$$
Therefore $\Xi_\theta(b)=\tr((M-P)\Sigma)$.
Since $$M-P=\gamma^2\Phi^\top P\Gamma(R+\gamma\Gamma^\top P\Gamma)^{-1}\Gamma^\top P\Phi\succeq 0,$$ this is nonnegative.
Taking the supremum over beliefs generated by the stationary Kalman filter yields $$\Xi_\theta\le\tr((M-P)\bar\Sigma),$$ where $\bar\Sigma$ is the steady-state Kalman predictive covariance.

When $\Gamma=0$, the DARE gives $M=P$, so $\Xi_\theta=0$.

\emph{Structural residual (part~1).}
Let $\mu_t(\cdot):=\PP(C=\cdot\mid H_t)$ be the posterior on context and $c^\star$ the true context.
By the definition of the Bayes classifier $$\hat C_t^{\mathrm{MAP}}=\arg\max_c\mu_t(c),$$
we have that
$$\E[1-\max_c\mu_t(c)]\le \PP(\widehat C_t\neq C)\le \rho_t.$$
For any bounded $f:\gC\to[-V_{\max},V_{\max}]$, let $c_t^{\mathrm{MAP}}=\arg\max_c\mu_t(c)$. Then
$$\bigl|\sum_c \mu_t(c)f(c)-f(c_t^{\mathrm{MAP}})\bigr|\le 2V_{\max}(1-\max_c\mu_t(c)),$$
since the total posterior mass on contexts other than $c_t^{\mathrm{MAP}}$ is $1-\max_c\mu_t(c)$.
The structural residual is bounded directly (without the triangle inequality):
\begin{align*}
\Xi_{\theta,t}(b)
&= \sum_c\mu_t(c)\max_a Q_c^{\star,S}(s,a) - \max_a\sum_c\mu_t(c)Q_c^{\star,S}(s,a) \\
&\le \sum_c\mu_t(c)\max_a Q_c^{\star,S}(s,a) - \sum_c\mu_t(c)Q_c^{\star,S}(s,a_{c_t^{\mathrm{MAP}}}^\star) \\
&= \sum_{c\neq c_t^{\mathrm{MAP}}}\mu_t(c)\bigl[\max_a Q_c^{\star,S}(s,a)-Q_c^{\star,S}(s,a_{c_t^{\mathrm{MAP}}}^\star)\bigr] \\
&\le \sum_{c\neq c_t^{\mathrm{MAP}}}\mu_t(c)\cdot 2V_{\max}
\;=\; 2V_{\max}(1-\max_c\mu_t(c)),
\end{align*}
where $a_{c_t^{\mathrm{MAP}}}^\star:=\arg\max_a Q_{c_t^{\mathrm{MAP}}}^{\star,S}(s,a)$ and the first inequality uses $$\max_a\sum_c\mu_t(c)Q_c(s,a)\ge\sum_c\mu_t(c)Q_c(s,a_{c_t^{\mathrm{MAP}}}^\star),$$
Taking expectations yields $$\E[\Xi_{\theta,t}]\le 2V_{\max}\E[1-\max_c\mu_t(c)]\le 2V_{\max}\rho_t.$$

\emph{Policy suboptimality (part~2).}
From Proposition~\ref{prop:structural_error}, $$\|Q_\theta^{\star,B}-Q_{\theta,0}\|_\infty\le\frac{\gamma}{1-\gamma}\Xi_{\theta,t}.$$
For any belief $b$, let $$a^\star=\arg\max_a Q_\theta^{\star,B}(b,a), \qquad a_{\mathrm{Q}}=\pi_{\mathrm{Q}}(b).$$ Then
$$Q_\theta^{\star,B}(b,a^\star)-Q_\theta^{\star,B}(b,a_{\mathrm{Q}})
\le 2\|Q_\theta^{\star,B}-Q_{\theta,0}\|_\infty
\le \frac{2\gamma}{1-\gamma}\Xi_{\theta,t},$$
where the first inequality uses $$a_{\mathrm{Q}}=\arg\max_a Q_{\theta,0}(b,a).$$
By standard policy-loss arguments (telescoping the per-step suboptimality),
$$\sup_b[V_\theta^{\star,B}(b)-V_\theta^{\pi_{\mathrm{Q}},B}(b)]
\le \frac{1}{1-\gamma}\cdot\frac{2\gamma}{1-\gamma}\Xi_{\theta,t}
\le \frac{4\gamma V_{\max}}{(1-\gamma)^2}\rho_t.$$

\section{Empirical Validation of the Regret Decomposition}
\label{app:validation}

\paragraph{Instance.}
The mixing Tiger of Section~\ref{sec:validation} is a variant of Tiger~\citep{kaelbling1998planning} that we constructed for diagnostic purposes, and it is not a standard benchmark.
It has two latent states, three actions (listen, open-left, open-right), and two observations that report the state correctly with probability $0.85$ under every action.
Listening mixes the state, $P(\cdot\mid s,\text{listen})=0.95\,\delta_s+0.05\,\mathrm{Unif}$, opening a door resets it uniformly, and the standard Tiger rewards $(-1,+10,-100)$ are mapped affinely to $(0.9,1,0)$.
Relative to the uniform measure every transition density lies in $[0.05,1.95]$, so Assumption~\ref{aspt:particle_error} holds with $\varepsilon_{\mathrm{mix}}=0.05$, $g_{\min}=0.15$, and $g_{\max}=0.85$. The rewards lie in $[0,1]$, the discount factor is $\gamma=0.8$, and $\Xi_\theta=0.10$.

\paragraph{Isolating each term.}
Except in the estimation run, the planner and filter use the true model, so each experiment contains a single source of error.
Belief-MDP quantities are computed by value iteration on a grid of $4001$ belief points with linear interpolation.
\emph{Particle filtering:} a bootstrap SIR filter with $J$ particles tracks the belief for $40$ listening steps, and the error is the absolute error of its estimate of $\PP(S=1)$ against the exact filter, averaged over steps ($2000$ replications).
\emph{Monte Carlo:} the root value $Q_{\theta,3}(b,a)$ is estimated from $N$ sampled first transitions, $r(s,a)+\gamma\bar V^{\star,S}_{\theta,2}(\tau(b,a,o))$, with the depth-$2$ continuation computed exactly and fixed before sampling, so the $N$ returns are conditionally i.i.d. The error is the root-mean-squared error over five beliefs and all three actions, computed from $2000$ replications.
This design isolates the $N^{-1/2}$ fluctuation of Proposition~\ref{prop:mc_concentration}. It does not include the pooling of leaf particles across simulations or the DP-backup bias of the full estimator, for which that proposition does not require independent returns.
\emph{Structural:} $\|\gT_\theta^d\bar V_\theta^{\star,S}-V_\theta^{\star,B}\|_\infty$ is computed exactly for $d=0,\dots,12$ and compared with the envelope $\gamma^{d+1}\Xi_\theta/(1-\gamma)$ of Proposition~\ref{prop:structural_error}.
Rates are least-squares slopes of the log mean error against the log budget, and the structural contraction is the fitted per-level ratio.
Figure~\ref{fig:theory_audit} shows all three terms.
The budgeting rule of Section~\ref{sec:validation} reads off, for each threshold, the smallest tested $J$, $N$, and $d$ whose isolated error falls below it.

\begin{figure}[h]
    \centering
    \includegraphics[width=\linewidth]{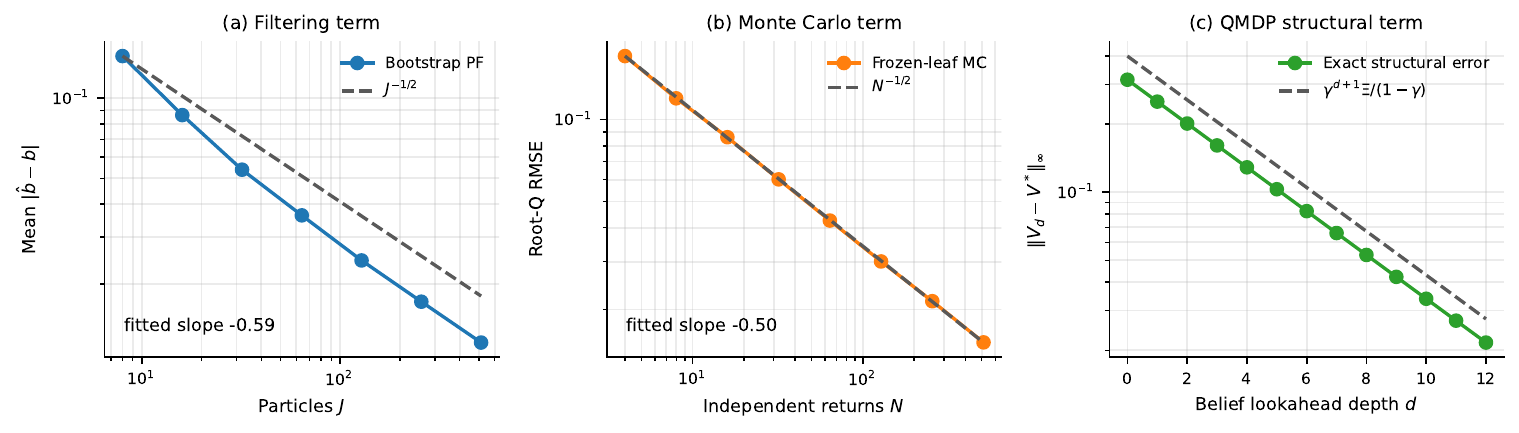}
    \caption{The three computational terms on the mixing Tiger. (a) Filtering error versus particles $J$ (fitted slope $-0.59$). (b) Root-$Q$ RMSE of the frozen-leaf estimator versus returns $N$ (fitted slope $-0.50$). (c) Exact structural error versus belief lookahead depth $d$. The error contracts by a factor $0.80=\gamma$ per level and stays below $\gamma^{d+1}\Xi_\theta/(1-\gamma)$. Means over $2000$ replications in (a, b).}
    \label{fig:theory_audit}
\end{figure}

\paragraph{Estimation term.}
Here the model is learned online by maximum likelihood.
In Figure~\ref{fig:bound_scaling}(a), the optimization error $\varepsilon_{\mathrm{opt}}$ decays as $n^{-0.59}$ and the model's total-variation error as $n^{-0.55}$ in the number of model-estimation samples $n$.
In panel~(b), the cumulative estimation term $\sum_{t\le T}\varepsilon_{\mathrm{opt}}(t)$ grows as $T^{0.57}$, close to $\sqrt T$. With the known model, the same pipeline accumulates about $1.3\cdot10^{-4}$ per round.

\paragraph{Structural term across instances.}
In Figure~\ref{fig:bound_scaling}(c), the structural error of $32$ instances stays below the envelope $\gamma^{d+1}\Xi/(1-\gamma)$ and contracts by a factor of about $\gamma$ per level, with fitted ratios of $0.798$--$0.800$.

\begin{figure}[h]
    \centering
    \includegraphics[width=\linewidth]{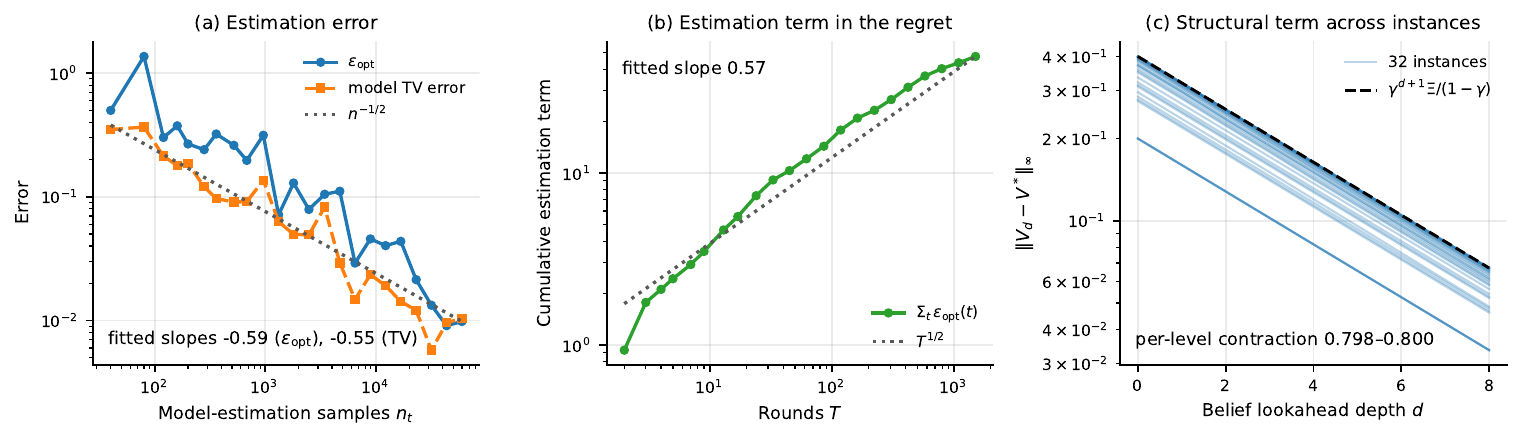}
    \caption{Estimation and structural terms. (a) $\varepsilon_{\mathrm{opt}}$ and the model's total-variation error versus model-estimation samples $n_t$, with an $n^{-1/2}$ reference (fitted slopes $-0.59$ and $-0.55$). (b) Cumulative estimation term $\sum_t\varepsilon_{\mathrm{opt}}(t)$ versus rounds $T$, with a $T^{1/2}$ reference (fitted slope $0.57$). (c) Structural error $\|\bar V_d-V^\star\|_\infty$ of $32$ instances versus belief lookahead depth $d$, below the envelope $\gamma^{d+1}\Xi/(1-\gamma)$ (per-level contraction $0.798$--$0.800$).}
    \label{fig:bound_scaling}
\end{figure}

\section{Assumption Stress Tests}
\label{app:stress}

The two diagnostics below are intended to illustrate the mechanisms behind Assumption~\ref{aspt:particle_error} and the structural term, and are not meant as benchmarks.
Figure~\ref{fig:stress} plots both.

\paragraph{QMDP information bias (active sensing).}
In this Light--Dark-type problem, information can be obtained only by giving up immediate reward. Two latent contexts persist with mixing rate $0.02$. Two exploit actions give reward $1$ when they match the context and reveal nothing, while a sense action gives no reward and reveals the context with accuracy $q$. The discount factor is $\gamma=0.9$.
Table~\ref{tab:active_sensing} reports the normalized return $(1-\gamma)V^{\pi_d,B}(b_0)$ at the uniform belief of the greedy depth-$d$ policy, computed exactly on the belief grid.
Because sensing has no immediate reward and the QMDP leaf values the state as if it will be revealed, depths $0$ and $1$ never sense and attain a normalized return of $0.5$. Depth $3$ comes within $0.006$ of the optimum for $q\ge0.75$, and depth $2$ does so for $q\ge0.85$. At $q=0.65$, even depth $3$ does not sense.

\begin{table}[h]
\centering
\caption{Active sensing: normalized return of the depth-$d$ policy (exact computation).}
\label{tab:active_sensing}
\small
\begin{tabular}{lccccc}
\toprule
Sensing accuracy $q$ & QMDP ($d=0$) & $d=1$ & $d=2$ & $d=3$ & Optimum \\
\midrule
0.55 & 0.500 & 0.500 & 0.500 & 0.500 & 0.500 \\
0.65 & 0.500 & 0.500 & 0.500 & 0.500 & 0.565 \\
0.75 & 0.500 & 0.500 & 0.500 & 0.641 & 0.643 \\
0.85 & 0.500 & 0.500 & 0.717 & 0.717 & 0.722 \\
0.95 & 0.500 & 0.500 & 0.793 & 0.794 & 0.798 \\
\bottomrule
\end{tabular}
\end{table}

\paragraph{Bounded-density violation (item search).}
An item lies in one of $K=32$ boxes under a uniform prior, and the observation reveals the box exactly, so $g_{\min}=0$.
A particle set of size $J$ drawn from the prior has all weights zero whenever no particle lies in the revealed box, which happens with probability $(1-1/K)^J$. In that case the particle policy chooses the wrong box, so this is also the decision-failure probability.
Table~\ref{tab:box_search} compares the empirical collapse rate over $10{,}000$ replications with this formula.
The effective sample size does not detect this failure. At $g_{\min}=0$ the mean ESS$/J$ is about $1/K$ for every $J$. With a small floor $g_{\min}=10^{-6}$, a particle set that misses the revealed box has equal weights and ESS$/J=1$, so the mean ESS$/J$ is $0.81$ at $J=8$, even though $77\%$ of the runs miss.

\begin{table}[h]
\centering
\caption{Item search with $K=32$ and $g_{\min}=0$: weight collapse versus particles $J$.}
\label{tab:box_search}
\small
\begin{tabular}{lccccc}
\toprule
$J$ & 8 & 16 & 32 & 64 & 128 \\
\midrule
Empirical collapse rate & .773 & .602 & .360 & .128 & .017 \\
Analytic $(1-1/K)^J$ & .776 & .602 & .362 & .131 & .017 \\
Mean ESS$/J$ & .032 & .031 & .031 & .032 & .031 \\
\bottomrule
\end{tabular}
\end{table}

\begin{figure}[h]
    \centering
    \includegraphics[width=\linewidth]{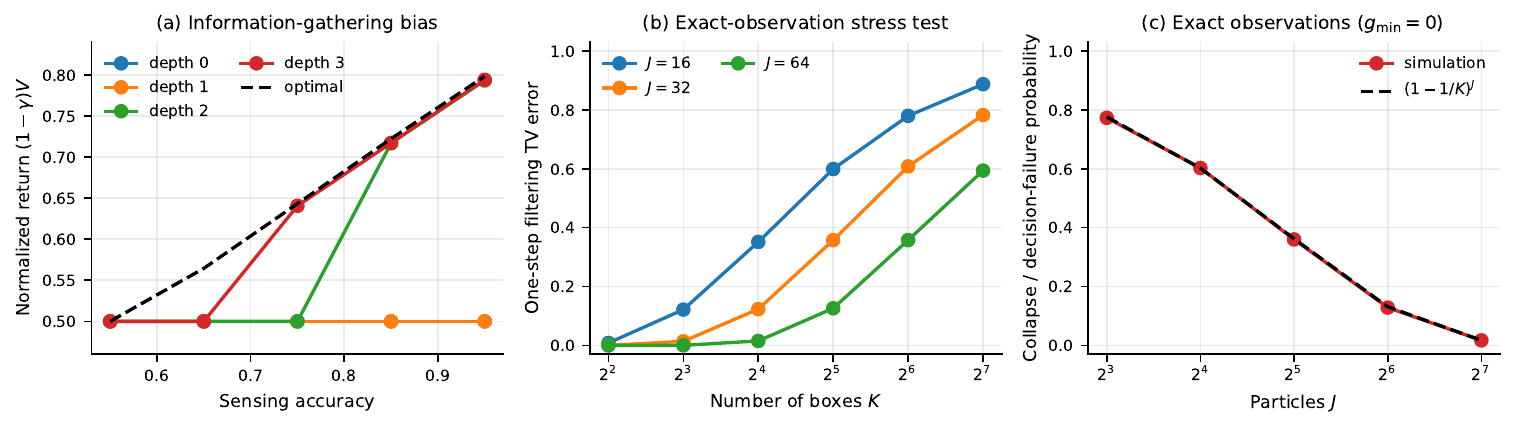}
    \caption{Stress tests illustrating the two mechanisms. (a) In active sensing, depths $0$--$1$ never sense, while depths $2$--$3$ approach the belief-optimal return when sensing is informative. (b) With exact observations over $K$ boxes, the one-step filtering error grows when the particle set is too small to cover the observed box. (c) At $g_{\min}=0$, weight collapse matches $(1-1/K)^J$ ($K=32$), which is also the decision-failure probability. Means over $10{,}000$ replications in (b, c).}
    \label{fig:stress}
\end{figure}

\section{Planner Comparison at Depth \texorpdfstring{$d>0$}{d>0}}
\label{app:planners}
Here we run the analyzed MCTQ at $d>0$ on two finite POMDPs and compare it with POMCP~\citep{silver2010monte} and a sparse particle-belief tree (Sparse-PFT)~\citep{lim2021sparse}.

\paragraph{Setup.}
The environments are the active-sensing problem of Appendix~\ref{app:stress} with $\gamma=0.9$ and the mixing Tiger of Appendix~\ref{app:validation}.
The agent tracks the exact belief, and all planners share the same $256$-particle root sample, with common random numbers.
We report the mean $\pm$ s.e.\ over $3$ seeds of $40$ episodes each.
Calls per decision are counted in generative-model steps, of which MCTQ uses $|\gA|Nd$.

\begin{table}[h]
\centering
\caption{Active sensing ($\gamma=0.9$): return, mean $\pm$ s.e.\ over three seeds of $40$ episodes. The belief optimum scores $7.45\pm0.20$ and the exact QMDP policy $5.54\pm0.32$.}
\label{tab:planners}
\small
\begin{tabular}{lccc}
\toprule
Planner & $d$ & Calls per decision & Return \\
\midrule
MCTQ, uniform & 0 & 0 & $5.04\pm0.11$ \\
MCTQ, uniform ($=$ UCB1) & 1 & 768 & $5.06\pm0.13$ \\
POMCP & -- & 768 / 3072 & $5.47$ / $5.52$ \\
MCTQ, uniform & 2 & 1536 & $6.37\pm0.27$ \\
MCTQ, UCB1 & 2 & 1536 & $6.96\pm0.11$ \\
Sparse-PFT, $J=64$ & 2 & 2688 & $7.18\pm0.08$ \\
MCTQ, uniform & 2 & 6144 & $7.30\pm0.12$ \\
MCTQ, uniform / UCB1 & 3 & 2304 & $5.69$ / $6.74$ \\
\bottomrule
\end{tabular}
\end{table}

\paragraph{Results.}
Depth roughly determines the value, as predicted by Proposition~\ref{prop:structural_error}.
At a fixed budget, performance is U-shaped in $d$.
On the mixing Tiger, QMDP achieves the belief optimum (both $4.586$), with every planner landing within $0.06$ of it.

\section{Ant: Correction Screen, Full-Horizon Runs, and Runtime}
\label{app:ant}

\paragraph{Correction screen.}
Unlike in the submitted version, Ant training is no longer stopped early, and Figure~\ref{fig:mujoco} shows the full Ant horizon.
Training the released configuration over the full horizon revealed a systematic late collapse in both observation ablations.
We report this result here and examined three choices in the released Ant configuration: it freezes the model after its random-data warm start, although the main text describes online updates; it uses AdamW weight decay $0.1$; and it performs one policy update per environment step.
At learning rate $3\cdot10^{-4}$, weight decay alone multiplies the parameters by about $0.22$ over $50{,}000$ optimizer steps.
The controls in Table~\ref{tab:ant_screen} show that the first two corrections are not sufficient and that the collapse is prevented by reducing the update-to-data ratio (UTD).
We screened corrections on position-hidden Ant (seed $0$) using a fixed $15$k--$20$k selection window (Figure~\ref{fig:ant_screen}).
Neither online updates nor corrected decay alone prevents collapse.
We selected UTD $1/10$ before launching the final multi-seed comparison. Table~\ref{tab:hyper_ant} lists the corrected configuration.

\begin{table}[h]
\centering
\caption{Correction screen on position-hidden Ant (seed $0$): mean $10$-episode rolling return over $15$k--$20$k steps and at $20.1$k steps.}
\label{tab:ant_screen}
\small
\begin{tabular}{lccccc}
\toprule
Variant & Online model & Decay & UTD & $15$k--$20$k mean & At $20.1$k \\
\midrule
Submitted & No & $0.1$ & $1$ & $-279$ & $-563$ \\
Online model only & Yes & $0.1$ & $1$ & $-378$ & $-834$ \\
Online model + corrected decay & Yes & $10^{-4}$ & $1$ & $-259$ & $-505$ \\
Corrected + UTD $1/4$ & Yes & $10^{-4}$ & $1/4$ & $902$ & $881$ \\
Corrected + UTD $1/10$ & Yes & $10^{-4}$ & $1/10$ & $980$ & $975$ \\
\bottomrule
\end{tabular}
\end{table}

\begin{figure}[h]
    \centering
    \includegraphics[width=0.6\linewidth]{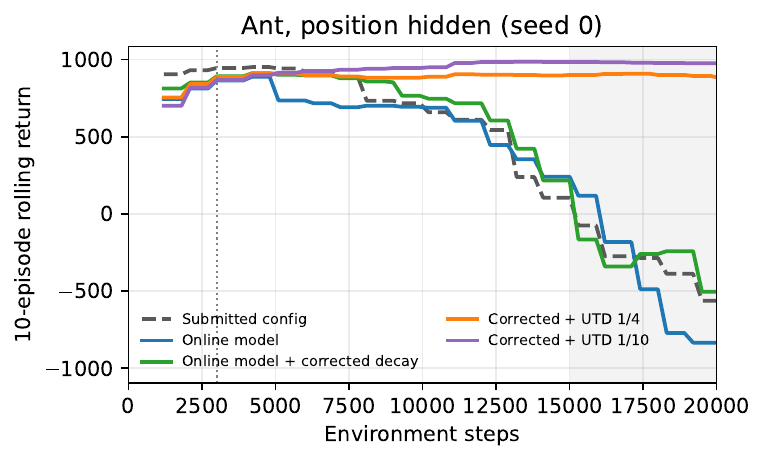}
    \caption{Correction screen on position-hidden Ant (seed $0$). ``Corrected'' denotes online model updates with weight decay $10^{-4}$; UTD $1/k$ means one policy update per $k$ environment steps. The shaded region is the $15$k--$20$k selection window, and the dotted line marks the original $3$k cutoff.}
    \label{fig:ant_screen}
\end{figure}

\paragraph{Full horizon.}
All six corrected runs reach $49{,}200$ steps.
In the prespecified $40$k--$49.2$k late window, corrected PF-MCTQ scores $893\pm177$ with position hidden and $822\pm268$ with velocity hidden (mean and 95\% Student-$t$ interval over three seeds). The submitted configuration, run with two seeds, scores $-396$ and $-518$, respectively.
Independent final $10$-episode evaluations are $862\pm178$ and $767\pm323$ for the corrected method, compared with means of $-751$ and $-625$ for the submitted configuration.
Table~\ref{tab:ant_late} also reports the five-seed recurrent-SAC and GRU-PPO comparisons, and Table~\ref{tab:ant_checkpoints} gives PF-MCTQ at fixed checkpoints.
All reported values are measured, and no curve or table entry is extrapolated or forward-filled.
The submitted configuration is not asymptotically stable, while the targeted update-ratio correction prevents catastrophic collapse across the tested seeds.

\begin{table}[h]
\centering
\caption{Ant return over the prespecified $40$k--$49.2$k window. Values with $\pm$ are means and 95\% Student-$t$ intervals. For the submitted configuration's two complete seeds, parentheses show both seed values. Recurrent baselines have five seeds.}
\label{tab:ant_late}
\small
\begin{tabular}{llccc}
\toprule
Hidden variable & Method & Seeds & Late return & Final evaluation \\
\midrule
Position & PF-MCTQ, corrected & 3 & $893\pm177$ & $862\pm178$ \\
Position & PF-MCTQ, submitted & 2 & $-396$ ($-506$, $-287$) & $-751$ ($-1055$, $-448$) \\
Position & Recurrent SAC & 5 & $-123\pm47$ & not recorded \\
Position & GRU-PPO & 5 & $599\pm147$ & not recorded \\
Velocity & PF-MCTQ, corrected & 3 & $822\pm268$ & $767\pm323$ \\
Velocity & PF-MCTQ, submitted & 2 & $-518$ ($-513$, $-522$) & $-625$ ($-511$, $-739$) \\
Velocity & Recurrent SAC & 5 & $593\pm154$ & not recorded \\
Velocity & GRU-PPO & 5 & $669\pm37$ & not recorded \\
\bottomrule
\end{tabular}
\end{table}

The checkpoint values below are the logged $10$-episode rolling training return at epoch ends, so every value is measured.

\begin{table}[h]
\centering
\caption{Ant checkpoints. Corrected: mean and 95\% Student-$t$ half-width over three seeds ($t_{2,0.975}=4.30$). Submitted: mean of the two complete seeds, with both seed values in parentheses. The last two rows are the mean over the prespecified $40$k--$49.2$k window and a final $10$-episode evaluation.}
\label{tab:ant_checkpoints}
\small
\begin{tabular}{lcccc}
\toprule
 & \multicolumn{2}{c}{Corrected (3 seeds)} & \multicolumn{2}{c}{Submitted (2 seeds)} \\
\cmidrule(lr){2-3}\cmidrule(lr){4-5}
Env.\ steps & Pos.\ hidden & Vel.\ hidden & Pos.\ hidden & Vel.\ hidden \\
\midrule
3k & $825\pm146$ & $882\pm60$ & $908$ ($945$, $871$) & $897$ ($925$, $870$) \\
10.2k & $931\pm61$ & $934\pm11$ & $600$ ($659$, $541$) & $435$ ($606$, $264$) \\
20.1k & $959\pm36$ & $962\pm15$ & $-911$ ($-563$, $-1260$) & $-620$ ($-697$, $-542$) \\
49.2k & $868\pm249$ & $816\pm241$ & $-448$ ($-678$, $-217$) & $-648$ ($-708$, $-589$) \\
\midrule
40k--49.2k mean & $893\pm177$ & $822\pm268$ & $-396$ ($-506$, $-287$) & $-518$ ($-513$, $-522$) \\
Final evaluation & $862\pm178$ & $767\pm323$ & $-751$ ($-1055$, $-448$) & $-625$ ($-511$, $-739$) \\
\bottomrule
\end{tabular}
\end{table}

\paragraph{Runtime.}
Table~\ref{tab:runtime} reports wall-clock cost on one RTX 3090 with position hidden.
Training time is end-to-end (data collection and all updates, excluding the final evaluation).
PF-MCTQ rates come from full $49.2$k-step runs and the baseline rates from $3$k-step runs with the same $300$-step epochs. Initialization is the one-time cost before the first environment step, and for PF-MCTQ it is dominated by the $2000$ model warm-start updates.

\begin{table}[h]
\centering
\caption{Wall-clock cost on Ant with the released code (position hidden, one RTX 3090).}
\label{tab:runtime}
\small
\begin{tabular}{lccc}
\toprule
Method & Initialization (s) & Training s / 1k steps & Steps / s \\
\midrule
PF-MCTQ (corrected) & 86.3 & 15.8 & 63.4 \\
PF-MCTQ (submitted cfg.) & 92.9 & 41.7 & 24.0 \\
Recurrent SAC & 21.0 & 53.6 & 18.6 \\
GRU-PPO & 2.8 & 5.5 & 181.8 \\
\bottomrule
\end{tabular}
\end{table}

\end{document}